%% file: scaffold.tex
\documentclass{article}

\usepackage{booktabs}
\usepackage{multirow}

\usepackage{amssymb}
\usepackage{amsmath}
\usepackage{amsthm}
\usepackage{bbold}
\usepackage{siunitx}

\usepackage{natbib}
\usepackage{tikz}
\usepackage{chngcntr}
\usepackage{verbatim}
\usepackage{mathtools}
\usepackage{esvect}
\usepackage{xcolor}
\usepackage{xspace}
\usepackage{xparse}
\usepackage{algorithm}
\usepackage{algorithmic}
\usepackage{thmtools,thm-restate}
\usepackage{graphicx}
\usepackage{sidecap}
\usepackage{hyperref}
\hypersetup{colorlinks,allcolors=blue}      
\usepackage{enumitem}
\usepackage[flushleft]{threeparttable}
\usepackage{tikzit}
\usepackage{subcaption}
\input{figures/drift.tikzstyles}
\setlist{leftmargin=5.5mm}
\DeclareCaptionFormat{myformat}{\fontsize{8.5}{10}\selectfont#1#2#3}
\usepackage{bm}

\newcommand{\covfefe}{{\sc SCAFFOLD}}
\newcommand{\fullformcovfefe}{Stochastic Controlled Averaging}
\newcommand{\fedavg}{{\sc FedAvg}}
\newcommand{\easgd}{{\sc FedProx}}
\newcommand{\emnist}{{\sc EMNIST}}

\input{macros.tex}

\usepackage[accepted]{icml2020}

\icmltitlerunning{\covfefe: \fullformcovfefe\ for Federated Learning}

\begin{document}

\twocolumn[
\icmltitle{\covfefe: \fullformcovfefe\ for Federated Learning}



\icmlsetsymbol{equal}{*}

\begin{icmlauthorlist}

\icmlauthor{Sai Praneeth Karimireddy}{epfl,intern}
\icmlauthor{Satyen Kale}{goo}
\icmlauthor{Mehryar Mohri}{goo,ci}
\icmlauthor{Sashank J. Reddi}{goo}
\icmlauthor{Sebastian U. Stich}{epfl}
\icmlauthor{Ananda Theertha Suresh}{goo}
\icmlaffiliation{epfl}{EPFL, Lausanne}
\icmlaffiliation{goo}{Google Research, New York}
\icmlaffiliation{ci}{Courant Institute, New York}
\icmlaffiliation{intern}{Based on work performed at Google Research, New York.}
\end{icmlauthorlist}

\icmlcorrespondingauthor{Sai Praneeth Karimireddy}{sai.karimireddy@epfl.ch}

\icmlkeywords{Federated learning, optimization, distributed optimization, decentralized optimization}

\vskip 0.3in
]



\printAffiliationsAndNotice{}  

\begin{abstract}%
    Federated Averaging (\fedavg) has emerged as the algorithm of choice for federated learning due to its simplicity and low communication cost. However, in spite of recent research efforts, its performance is not fully understood. We obtain tight convergence rates for \fedavg\ and prove that it suffers from `client-drift' when the data is heterogeneous (non-iid), resulting in unstable and slow convergence.
    
    As a solution, we propose a new algorithm (\covfefe) which uses control variates (variance reduction) to correct for the `client-drift' in its local updates. We prove that \covfefe\ requires significantly fewer communication rounds and is not affected by data heterogeneity or client sampling. Further, we show that (for quadratics) \covfefe\ can take advantage of similarity in the client's data yielding even faster convergence. The latter is the first result to quantify the usefulness of local-steps in distributed optimization.

  \end{abstract}

\input{main.tex}

{
  \bibliography{scaffold}
  \bibliographystyle{icml2020}
}

\newpage
\onecolumn
\part*{Appendix}
\appendix
\input{appendix.tex}

\end{document}

%% file: figures/drift.tikzstyles
\tikzstyle{server}=[fill={rgb,255: red,64; green,64; blue,64}, draw=black, shape=circle]
\tikzstyle{client}=[fill={rgb,255: red,55; green,126; blue,184}, draw=black, shape=circle]
\tikzstyle{optima}=[fill={rgb,255: red,55; green,126; blue,184}, draw=black, shape=rectangle]
\tikzstyle{globalopt}=[fill={rgb,255: red,255; green,128; blue,0}, draw=black, shape=rectangle]
\tikzstyle{pseudoserver}=[fill={rgb,255: red,191; green,191; blue,191}, draw={rgb,255: red,128; green,128; blue,128}, shape=circle]
\tikzstyle{sgdnode}=[fill=white, draw=black, shape=circle]
\tikzstyle{client2}=[fill={rgb,255: red,77; green,175; blue,74}, draw=black, shape=circle]
\tikzstyle{clientopt}=[fill={rgb,255: red,77; green,175; blue,74}, draw=black, shape=rectangle]
\tikzstyle{averageopt}=[fill={rgb,255: red,191; green,191; blue,191}, draw={rgb,255: red,128; green,128; blue,128}, shape=rectangle]
\tikzstyle{gradient}=[dashed]

\tikzstyle{server update}=[->, draw={rgb,255: red,77; green,175; blue,74}]
\tikzstyle{client update}=[draw={rgb,255: red,55; green,126; blue,184}, ->]
\tikzstyle{update}=[draw=black, dashed, -]
\tikzstyle{sgd}=[draw={rgb,255: red,65; green,61; blue,58}, ->]
\tikzstyle{correction}=[draw={rgb,255: red,228; green,26; blue,28}, ->]
\tikzstyle{midpoint}=[-, draw={rgb,255: red,191; green,191; blue,191}]

%% file: macros.tex
\DeclarePairedDelimiterX{\inp}[2]{\langle}{\rangle}{#1, #2}
\DeclarePairedDelimiterX{\abs}[1]{\lvert}{\rvert}{#1}
\DeclarePairedDelimiterX{\norm}[1]{\lVert}{\rVert}{#1}
\DeclarePairedDelimiterX{\cbr}[1]{\{}{\}}{#1} 
\DeclarePairedDelimiterX{\rbr}[1]{(}{)}{#1} 
\DeclarePairedDelimiterX{\sbr}[1]{[}{]}{#1} 

\providecommand{\refLE}[1]{\ensuremath{\stackrel{(\ref{#1})}{\leq}}}

\providecommand{\mgeq}{\succeq}
\providecommand{\mleq}{\preceq}

\providecommand{\tsum}{\textstyle\sum} 
  \providecommand{\R}{\mathbb{R}} 
  \providecommand{\real}{\mathbb{R}} 

  \DeclareMathOperator{\expect}{\mathbb{E}}
  \DeclareMathOperator{\E}{\expect}

  \DeclareMathOperator{\sgn}{sign}
  \makeatletter
  \def\sign{\@ifnextchar*{\@sgnargscaled}{\@ifnextchar[{\sgnargscaleas}{\@ifnextchar{\bgroup}{\@sgnarg}{\sgn} }}}
  \def\@sgnarg#1{\sgn\rbr{#1}}
  \def\@sgnargscaled#1{\sgn\rbr*{#1}}
  \def\@sgnargscaleas[#1]#2{\sgn\rbr[#1]{#2}}
  \makeatother

  \providecommand{\alphav}{\bm{\alpha}}

  \providecommand{\cc}{\bm{c}}

  \let\ggg\gg
  \renewcommand{\gg}{\bm{g}}

  \renewcommand{\vv}{\bm{v}}
  
  \providecommand{\xx}{\bm{x}}
  \providecommand{\yy}{\bm{y}}
  \providecommand{\zz}{\bm{z}}

  \providecommand{\cA}{\mathcal{A}}

  \providecommand{\cO}{\mathcal{O}}

  \providecommand{\cS}{\mathcal{S}}
  \providecommand{\cT}{\mathcal{T}}

  \providecommand{\error}{\mathcal{E}}
  \providecommand{\ce}{\mathcal{C}}
  \providecommand{\tce}{\Xi}

\newtheorem{theorem}{Theorem}

\newtheorem{lemma}{Lemma}
\newtheorem{remark}[lemma]{Remark}

\newcommand{\ignore}[1]{}

\newcommand{\e}{\epsilon}

\definecolor{color1}{RGB}{228,26,28}
\definecolor{color2}{RGB}{55,126,184}
\definecolor{color3}{RGB}{77,175,74}
\definecolor{color4}{RGB}{152,78,163}
\definecolor{color5}{RGB}{255,127,0}

\usepackage{pifont}
\newcommand{\yes}{$\checkmark$}%
\newcommand{\no}{$\times$}%

\makeatletter
\newcommand{\myitem}[1]{%
\item[\textbf{(#1)}]\protected@edef\@currentlabel{#1}%
}
\makeatother

\usetikzlibrary{fit,calc}
\newcommand*{\tikzmk}[1]{\tikz[remember picture,overlay,] \node (#1) {};}
\newcommand{\boxit}[1]{\tikz[remember picture,overlay]{\node[yshift=3pt,fill=#1,opacity=.25,fit={($(A)+(-0.2*\linewidth - 3pt,3pt)$)($(B)+(0.75*\linewidth - 5pt,-2pt)$)}] {};}}

\newcommand{\blockfed}[1]{\tikz[remember picture,overlay]{\node[yshift=3pt,fill=#1,opacity=.25,fit={($(A)+(-0.1*\linewidth - 3pt,3pt)$)($(B)+(0.88*\linewidth - 5pt,-2pt)$)}] {};}}

\newcommand{\highlight}[1]{\tikz[remember picture,overlay]{\node[yshift=3pt,fill=#1,opacity=.25,fit={($(A)+(2pt,6pt)$)($(B)+(-3pt,-7pt)$)}] {};}}
\colorlet{client}{red!40}
\newcommand{\speedup}[1]{{\color{gray}(\ifdim #1 pt > 0.3pt #1\else $< #1$\fi{}$\times$)}}

%% file: main.tex
\section{Introduction}
Federated learning has emerged as an important paradigm in modern large-scale machine learning. Unlike in traditional centralized learning where models are trained using large datasets stored in a central server
\citep{dean2012large,iandola2016firecaffe,goyal2017accurate}, in
federated learning, the training data remains distributed over a large
number of clients, which may be phones, network
sensors, hospitals, or alternative local information sources
\citep{konevcny2016federated,konecny2016federated2,
  mcmahan2017communication,mohri2019agnostic,kairouz2019advances}. A centralized model (referred to as server model) is then trained without ever transmitting client data over the network,
thereby ensuring a basic level of privacy. In this work, we investigate stochastic optimization algorithms for federated learning.

The key challenges for federated optimization are 1) dealing with unreliable and relatively slow network connections between the server and the clients, 2) only a small subset of clients being available for training at a given time, and 3) large heterogeneity (non-iid-ness) in the data present on the different clients \citep{konecny2016federated2}. The most popular algorithm for this setting is \fedavg\ \citep{mcmahan2017communication}. \fedavg\ tackles the communication bottleneck by performing multiple local updates on the available clients before communicating to the server. While it has shown success in certain applications, its performance on heterogeneous data is still an active area of research \citep{li2018fedprox,yu2019parallel,li2019fedconvergence,haddadpour2019convergence,khaled2020tighter}. We prove that indeed such heterogeneity has a large effect on \fedavg---it introduces a \emph{drift} in the updates of each client resulting in slow and unstable convergence. Further, we show that this client-drift persists even if full batch gradients are used and all clients participate throughout the training.

As a solution, we propose a new \fullformcovfefe\ algorithm (\covfefe) which tries to correct for this client-drift. Intuitively, \covfefe\ estimates the update direction for the server model ($\cc$) and the update direction for each client $\cc_i$.\footnote{We refer to these estimates as \emph{control variates} and the resulting correction technique as stochastic controlled averaging.} The difference $(\cc - \cc_i)$ is then an estimate of the client-drift which is used to correct the local update. This strategy successfully overcomes heterogeneity and converges in significantly fewer rounds of communication. Alternatively, one can see heterogeneity as introducing `client-variance' in the updates across the different clients and \covfefe\ then performs `client-variance reduction' \citep{schmidt2017minimizing,johnson2013accelerating,defazio2014saga}. We use this viewpoint to show that \covfefe\ is relatively unaffected by client sampling.

Finally, while accommodating heterogeneity is important, it is equally important that a method can take advantage of similarities in the client data. We prove that \covfefe\ indeed has such a property, requiring fewer rounds of communication when the clients are more similar.

\paragraph{Contributions.} We summarize our main results below.
\begin{itemize}[topsep=0pt,itemsep=0pt,partopsep=0.5ex,parsep=0.5ex]
  \item We derive tighter convergence rates for \fedavg\ than previously known for convex and non-convex functions with client sampling and heterogeneous data.
  \item We give matching lower bounds to prove that even with no client sampling and full batch gradients, \fedavg\ can be slower than SGD due to client-drift.
  \item We propose a new \fullformcovfefe\ algorithm (\covfefe) which corrects for this client-drift. We prove that \covfefe\ is at least as fast as SGD and converges for arbitrarily heterogeneous data.
  \item We show \covfefe\ can additionally take advantage of similarity between the clients to further reduce the communication required, proving the advantage of taking local steps over large-batch SGD for the first time.
  \item We prove that \covfefe\ is relatively unaffected by the client sampling obtaining variance reduced rates, making it especially suitable for federated learning.
\end{itemize}
Finally, we confirm our theoretical results on simulated and real datasets (extended MNIST by \citet{cohen2017emnist}).

\textbf{Related work.} For identical clients, \fedavg\ coincides with
parallel SGD analyzed by \cite{zinkevich2010parallelized} who proved asymptotic convergence. \citet{stich2018local} and, more
recently \citet{stich2019error,patel2019communication,khaled2020tighter},
gave a sharper analysis of the same method, under the name of local
SGD, also for identical functions. However, there still remains a gap
between their upper bounds and the lower bound of
\citet{woodworth2018graph}. The analysis of \fedavg\ for heterogeneous clients is more delicate
due to the afore-mentioned client-drift, first empirically observed by \citet{zhao2018federated}. Several
analyses bound this drift by assuming bounded gradients
\citep{wang2019adaptive,yu2019parallel}, or view it as additional noise
\citep{khaled2020tighter}, or assume that the client optima are $\epsilon$-close \citep{li2018fedprox,haddadpour2019convergence}. In a concurrent work, \citep{liang2019variance} propose to use variance reduction to deal with client heterogeneity but still show rates slower than SGD and do not support client sampling.
Our method \covfefe\ can also be seen as an improved version of the distributed optimization algorithm DANE by \cite{shamir2014communication}, where a fixed number of
(stochastic) gradient steps are used in place of a proximal point update. A more in-depth discussion of related work is given in Appendix~\ref{appsec:related}. We summarize the complexities of different methods for heterogeneous clients in Table~\ref{tab:results}.

\begin{table}[!t]
  \caption{Summary of notation used in the paper}
  \resizebox{\linewidth}{!}{
  \begin{tabular}{@{}cl@{}} 
  \toprule
   $N$, $S$, and $i$ & total num., sampled num., and index of clients\\ 
   $R$, $r$ & number, index of communication rounds\\ 
   $K$, $k$ & number, index of local update steps\\ 
   $\xx^r$ & aggregated server model after round $r$\\
   $\yy^r_{i,k}$ & $i$th client's model in round $r$ and step $k$\\ 
   $\cc^r$, $\cc^r_i$ & control variate of server, $i$th client after round $r$\\ 
  \bottomrule
  \end{tabular}
  }
  \end{table}
  
\begin{table*}[!t]
  \caption{Number of communication rounds required to reach $\e$ accuracy for $\mu$ strongly convex  and non-convex functions (log factors are ignored). Set $\mu = \e$ for general convex rates. $(G,B)$ bounds gradient dissimilarity \eqref{asm:heterogeneity}, and $\delta$  bounds Hessian dissimilarity \eqref{asm:hessian-similarity}. Our rates for \fedavg\ are more general and tighter than others, even matching the lower bound. However, SGD is still faster ($B \geq 1$). \covfefe\ does not require any assumptions, is faster than SGD, and is robust to client sampling. Further, when clients become more similar (small $\delta$), \covfefe\ converges even faster.}
  \centering
  \begin{threeparttable}
  \begin{tabular}{@{}l>{\large}c>{\large}c>{\small}c>{\small}l@{}}
  \toprule
    Method  & \normalsize{Strongly convex} & \normalsize{Non-convex} & \normalsize{Sampling} & \normalsize{Assumptions}\\
    \midrule
    SGD (large batch) & $\frac{\sigma^2}{\mu N K \e} + \frac{1}{\mu}$ & $\frac{\sigma^2}{NK \e^2} + \frac{1}{\e}$ & \no & --\\
    \cmidrule{2-5}
    FedAvg\\ 
      \qquad {\citep{li2019fedconvergence}}  & $\frac{\sigma^2}{\mu^2 N K \e} + \frac{G^2K}{\mu^2 \e}$ & -- & \no & $(G,0)$-BGD\\
      \qquad {\citep{yu2019parallel}}  & -- & $\frac{\sigma^2}{N K \e^2} + \frac{G^2 NK}{\e}$ & \no & $(G,0)$-BGD\\
      \qquad {\citep{khaled2020tighter}}  & $\frac{\sigma^2 + G^2}{\mu N K \e} + \frac{\sigma + G}{\mu \sqrt{\e}} + \frac{N B^2}{\mu}$ & -- & \no & $(G,B)$-BGD\\
      \tikzmk{A}
      \qquad {Ours (Thm. \ref{thm:fedavg})}\tnote{1}  & $\frac{M^2}{\mu S K \e} + \frac{G}{\mu \sqrt{\e}} + \frac{B^2}{\mu}$ & $\frac{M^2}{SK \e^2} + \frac{G}{\e^{3/2}} + \frac{B^2}{\e}$ & \yes & $(G,B)$-BGD\\
      \qquad\ {Lower-bound} (Thm.~\ref{thm:lowerbound})  & $\Omega\rbr{\frac{\sigma^2}{\mu N K \e} + \frac{G}{\sqrt{\mu\e}}}$ & \normalsize{?} & \no &  $(G,1)$-BGD
      \hfill \tikzmk{B}\highlight{client}\\
    \cmidrule{2-5}
    FedProx~\citep{li2018fedprox}\tnote{2} & $\frac{B^2}{\mu}$ & $\frac{B^2}{\e}$ \small{(weakly convex)}  & \yes & $\sigma = 0$, $(0,B)$-BGD\\
    DANE~\citep{shamir2014communication}\tnote{2,3} & $\frac{\delta^2}{\mu^2}$ & --  & \no & $\sigma = 0$, $\delta$-BHD\\
    VRL-SGD \citep{liang2019variance} & -- & $\frac{N \sigma^2}{K \e^2} + \frac{N}{\e}$  & \no & --\\
    \cmidrule{2-5}
    \tikzmk{A}
    SCAFFOLD\\
    \qquad Theorem \ref{thm:convergence} & $\frac{\sigma^2}{\mu SK\e} + \frac{1}{\mu} + \frac{N}{S}$ & $\frac{\sigma^2}{SK\e^2} + \frac{1}{\e}\rbr{\frac{N}{S}}^{\frac{2}{3}}$  & \yes & --\\
    \qquad Theorem \ref{thm:interpolation}\tnote{3} & $\frac{\sigma^2}{\mu NK \e} + \frac{1}{\mu K} + \frac{\delta}{\mu}$ & $\frac{\sigma^2}{NK \e^2} + \frac{1}{K \e} + \frac{\delta}{\e}$  & \no &  $\delta$-BHD
    \hfill \tikzmk{B}\highlight{client}\\
    \bottomrule
  \end{tabular}
  \begin{tablenotes}
    \item[1]$M^2 := \sigma^2 + K(1-\frac{S}{N})G^2$. Note that $\frac{M^2}{S} = \frac{\sigma^2}{N}$ when no sampling ($S = N$).
    \item[2] proximal point method i.e. $K \ggg 1$.
    \item[3] proved only for quadratic functions.
  \end{tablenotes}
  \end{threeparttable}
  \label{tab:results}
  \end{table*}


\vspace{-2mm}
\section{Setup}
We formalize the problem as minimizing a sum of stochastic
functions, with only access to stochastic samples:\vspace{-2mm}
\[
    \min_{\xx \in \real^d} \cbr[\Big]{ f(\xx)
:= \frac{1}{N}\sum_{i=1}^N \rbr*{f_i(\xx)
:= \expect_{\zeta_i}[f_i(\xx;\zeta_i)]}}\,.
\]
The functions $f_i$ represents the loss function on client $i$. All our results can be easily extended to the weighted case. 

We assume that $f$ is bounded from below by $f^\star$ and $f_i$ is $\beta$-smooth. Further, we assume $g_i(\xx) := \nabla f_i(\xx;\zeta_i)$ is an unbiased stochastic gradient of $f_i$ with variance bounded by $\sigma^2$. For some results, we assume $\mu\geq 0$ (strong) convexity. Note that $\sigma$ only bounds the variance \emph{within} clients. We also define two non-standard terminology below.
\begin{enumerate}[topsep=0pt,itemsep=-1ex,partopsep=1ex,parsep=1ex,%
itemindent=8pt]
    \myitem{A1} \label{asm:heterogeneity} \textbf{$(G,B)$-BGD} or bounded gradient dissimilarity: there exist constants $G \geq 0$ and $B \geq 1$ such that\vspace{-2mm}
      \begin{equation*}
          \frac{1}{N}\sum_{i=1}^N \norm{\nabla f_i(\xx)}^2 \leq G^2 + B^2 \norm{\nabla f(\xx)}^2 \,,\ \forall \xx\,.\vspace{-1mm}
      \end{equation*}
      If $\{f_i\}$ are convex, we can relax the assumption to\vspace{-2mm}
      \begin{equation*}
          \frac{1}{N}\sum_{i=1}^N \norm{\nabla f_i(\xx)}^2 \leq G^2 + 2\beta  B^2(f(\xx) - f^\star)\,,\ \forall \xx\,.\vspace{-1mm}
      \end{equation*}  
      \myitem{A2}\textbf{$\delta$-BHD} or bounded Hessian dissimilarity:\vspace{-2mm}
      \label{asm:hessian-similarity} 
      \begin{equation*}\label{eqn:hessian-similarity}
      \norm{\nabla^2 f_i(\xx) - \nabla^2 f(\xx)} \leq  \delta\,,\ \forall \xx\,.\vspace{-2mm}
      \end{equation*}
      Further, $f_i$ is $\delta$-weakly convex i.e. $\nabla^2 f_i(\xx) \mgeq -\delta I$.

\end{enumerate}
The assumptions \ref{asm:heterogeneity} and \ref{asm:hessian-similarity} are orthogonal---it is possible to have $G=0$ and $\delta = 2\beta$, or $\delta = 0$ but $G \ggg 1$. 


  \section{Convergence of FedAvg}\label{sec:fedavg}
  In this section we review \fedavg\ and improve its convergence analysis by deriving tighter rates than known before. 
  The scheme consists of two main parts: local updates to the model \eqref{eqn:fedavg-model}, and aggregating the client updates to update the server model \eqref{eqn:fedavg-agg}. In each round, a subset of clients $\cS \subseteq [N]$ are sampled uniformly. Each of these clients $i \in \cS$ copies the current sever model $\yy_i = \xx$ and performs $K$ local updates of the form:
  \begin{equation}\label{eqn:fedavg-model}
    \yy_i \leftarrow \yy_i - \eta_l g_i(\yy_i)\,.
  \end{equation}
  Here $\eta_l$ is the local step-size. Then the clients' updates $(\yy_i - \xx)$ are aggregated to form the new server model using a global step-size $\eta_g$ as:
  \begin{equation}\label{eqn:fedavg-agg}
    \xx \leftarrow \xx + \frac{\eta_g}{\abs{\cS}}\sum_{i \in \cS}(\yy_i - \xx)\,.
  \end{equation}  
  \subsection{Rate of convergence}
  We now state our novel convergence results for functions with bounded dissimilarity (proofs in Appendix~\ref{appsec:fedavg}).
  \begin{theorem}\label{thm:fedavg}
   For $\beta$-smooth functions $\{f_i\}$ which satisfy \eqref{asm:heterogeneity}, the output of \fedavg\ has expected error smaller than $\e$ for some values of $\eta_l, \eta_g, R$ satisfying
        \begin{itemize}[topsep=0pt,itemsep=-1ex,partopsep=0.5ex,parsep=0.5ex]
          \item \textbf{Strongly convex:} $\eta_g \geq \sqrt{S}$, $\eta_l \leq \frac{1}{(1 + B^2)6\beta K \eta_g}$, and 
          \[
            R = \tilde\cO\rbr[\bigg]{\frac{\sigma^2}{\mu K S \e}+ \rbr[\big]{1  - \tfrac{S}{N}}\frac{G^2}{\mu S \e} + \frac{\sqrt{\beta}G}{\mu \sqrt{\e}} + \frac{B^2 \beta}{\mu}}\,,
          \]
         \item \textbf{General convex:} $\eta_g \geq \sqrt{S}$, $\eta_l \leq \frac{1}{(1 + B^2)6\beta K \eta_g}$, and
          \[
              R = \cO\rbr[\bigg]{\frac{\sigma^2D^2}{K S \e^2}+ \rbr[\big]{1  - \tfrac{S}{N}}\frac{G^2D^2}{S \e^2} + \frac{\sqrt{\beta}G}{\e^{\frac{3}{2}}} + \frac{B^2 \beta D^2}{\e}}\,,
          \] 
        \item \textbf{Non-convex:} $\eta_g \geq \sqrt{S}$, $\eta_l \leq \frac{1}{(1 + B^2)6\beta K \eta_g}$, and
        \[
              R = \cO\rbr[\bigg]{\frac{\beta \sigma^2 F}{K S \e^2}+ \rbr[\big]{1  - \tfrac{S}{N}}\frac{G^2 F}{S \e^2} + \frac{\sqrt{\beta}G}{\e^{\frac{3}{2}}} + \frac{B^2 \beta F}{\e}}\,,
        \]
        \end{itemize}
        where $D := \norm{\xx^0 - \xx^\star}^2$ and $F := f(\xx^0) - f^\star$.
    \end{theorem}
      It is illuminating to compare our rates with those of the simpler iid.\ case i.e.\ with $G =0$ and $B=1$. Our strongly-convex rates become $\frac{\sigma^2}{\mu S K \e}+ \frac{1}{\mu}$. In comparison, the best previously known rate for this case was by \citet{stich2019error} who show a rate of $\frac{\sigma^2}{\mu S K \e}+ \frac{S}{\mu}$. The main source of improvement in the rates came from the use of \emph{two separate step-sizes} ($\eta_l$ and $\eta_g$). By having a larger global step-size $\eta_g$, we can use a smaller local step-size $\eta_l$ thereby reducing the client-drift while still ensuring progress. However, even our improved rates do not match the lower-bound for the identical case of $\frac{\sigma^2}{\mu S K \e}+ \frac{1}{K \mu}$ \citep{woodworth2018graph}. We bridge this gap for quadratic functions in Section \ref{sec:interpolation}.

      We now compare \fedavg\ to two other algorithms FedProx by \cite{li2018fedprox} (aka EASGD by \cite{zhang2015deep}) and to SGD. Suppose that $G=0$ and $\sigma =0$ i.e. we use full batch gradients and all clients have very similar optima. In such a case, \fedavg\ has a complexity of $\frac{B^2}{\mu}$ which is identical to that of FedProx \citep{li2018fedprox}. Thus, FedProx does not have any theoretical advantage.
    
      Next, suppose that all clients participate (no sampling) with $S = N$ and there is no variance $\sigma = 0$. Then, the above for strongly-convex case simplifies to ${\frac{G}{\mu \sqrt{\e}} + \frac{B^2}{\mu}}$. In comparison, extending the proof of \citep{khaled2020tighter} using our techniques gives a worse dependence on $G$ of ${\frac{G^2}{\mu K N \e} + \frac{G}{\mu \sqrt{\e}}}$. Similarly, for the non-convex case, our rates are tighter and have better dependence on $G$ than \citep{yu2019parallel}. However, simply running SGD in this setting would give a communication complexity of $\frac{\beta}{\mu}$ which is faster, and independent of similarity assumptions. In the next section we examine the necessity of such similarity assumptions.

    \subsection{Lower bounding the effect of heterogeneity}
    We now show that when the functions $\{f_i\}$ are distinct, the local updates of \fedavg\ on each client experiences \emph{drift} thereby slowing down convergence. We show that the amount of this client drift, and hence the slowdown in the rate of convergence, is exactly determined by the gradient dissimilarity parameter $G$ in \eqref{asm:heterogeneity}.
    
    \begin{figure}[t]
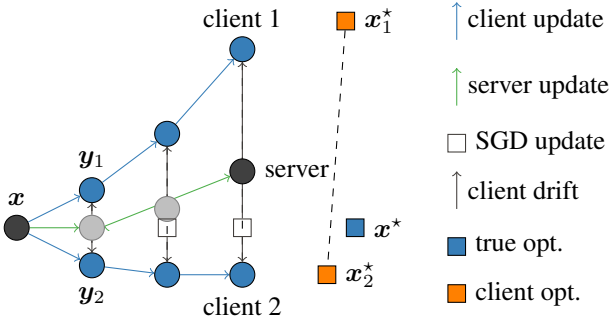

      \centering
      \ctikzfig{drift}  
      \caption{Client-drift in \fedavg\ is illustrated for 2 clients with 3 local steps ($N=2$, $K=3$). The local updates $\yy_i$ (in blue) move towards the individual client optima $\xx_i^\star$ (orange square). The server updates (in red) move towards $\frac{1}{N}\sum_{i}\xx_i^\star$ instead of to the true optimum $\xx^\star$ (black square).}\label{fig:client drift}
    \end{figure}

    We now examine the mechanism by which the client-drift arises (see Fig.~\ref{fig:client drift}). Let $\xx^\star$ be the global optimum of $f(\xx)$ and $\xx_i^\star$ be the optimum of each client's loss function $f_i(\xx)$. In the case of heterogeneous data, it is quite likely that each of these $\xx_i^\star$ is far away from the other, and from the global optimum $\xx^\star$. Even if all the clients start from the same point $\xx$, each of the $\yy_i$ will move towards their client optimum $\xx_i^\star$. This means that the average of the client updates (which is the server update) moves towards $\frac{1}{N}\sum_{i=1}^N \xx_i^\star$. This difference between $\frac{1}{N}\sum_{i=1}^N \xx_i^\star$ and the true optimum $\xx^\star$ is exactly the cause of client-drift. To counter this drift, \fedavg\ is forced to use much smaller step-sizes which in turn hurts convergence. We can formalize this argument to prove a lower-bound (see Appendix~\ref{appsec:lowerbound} for proof).

    \begin{theorem}\label{thm:lowerbound}
      For any positive constants $G$ and $\mu$, there exist $\mu$-strongly convex functions satisfying \ref{asm:heterogeneity} for which \fedavg\ with $K \geq 2$, $\sigma = 0$ and $N = S$ has an error\vspace{-2mm}
      \[
          f(\xx^r) - f(\xx^\star) \geq \Omega\rbr[\bigg]{ \frac{G^2}{\mu R^2}}\,.\vspace{-3mm}
      \]
     \end{theorem}
     This implies that the $\frac{G}{\sqrt{\e}}$ term is unavoidable even if there is no stochasticity. Further, because \fedavg\ uses $RKN$ stochastic gradients, we also have the statistical lower-bound of $\frac{\sigma^2}{\mu KN \e}$. Together, these lower bounds prove that the rate derived in Theorem~\ref{thm:fedavg} is nearly optimal (up to dependence on $\mu$). In the next section, we introduce a new method \covfefe\ to mitigate this client-drift.


\section{\covfefe\ algorithm}\label{sec:method}
In this section we first describe \covfefe\ and then discuss how it solves the problem of client-drift. 

\begin{algorithm}[t]
  \caption{\covfefe: \fullformcovfefe\ for federated
    learning}\label{alg:sampling}
  \begin{algorithmic}[1]
    \STATE \textbf{server input:} initial $\xx$ and  $\cc$, and global step-size $\eta_g$
        \STATE \textbf{\color{client} client $i$'s input:} $\cc_i$, and local step-size $\eta_l$
          \FOR{each round $r=1,\dots, R$}
          \STATE sample clients $\cS \subseteq \{1,\dots,N\}$
          \STATE \textbf{communicate} $(\xx, \cc)$ to all clients $i \in \cS$
          \ONCLIENT{\tikzmk{A} $i \in \cS$}
            \STATE initialize local model $\yy_{i} \leftarrow \xx$
            \FOR{$k = 1,\dots,K $}
      \STATE compute mini-batch gradient $g_i(\yy_{i})$
      \STATE $\yy_{i} \leftarrow \yy_{i} - \eta_l\,(g_i(\yy_{i})- \cc_i + \cc)$ 
      \ENDFOR
      \STATE $\cc_i^+ \leftarrow$ (i) $g_i(\xx)$, or (ii) $\cc_i - \cc + \frac{1}{K\eta_l}(\xx- \yy_i)$\label{step:control-update}
      \STATE \textbf{communicate} $(\Delta\yy_i, \Delta\cc_i) \leftarrow (\yy_{i} - \xx,\cc_i^+ - \cc_i)$
      \STATE $\cc_i \leftarrow \cc_i^+$
          \ENDON
          \tikzmk{B}\boxit{client}
          \STATE $(\Delta\xx, \Delta\cc) \leftarrow \tfrac{1}{\abs{\cS}}\sum_{i \in \cS}(\Delta\yy_i, \Delta \cc_i)$ 
          \STATE $\xx \leftarrow \xx + \eta_g\Delta\xx$ and $\cc \leftarrow \cc + \tfrac{\abs{\cS}}{N}\Delta\cc$
          \ENDFOR
    \end{algorithmic}
  \end{algorithm}
  \vspace{-2mm}
\paragraph{Method.} SCAFFOLD has three main steps: local updates to the client model \eqref{eqn:covfefe-model}, local updates to the client control variate \eqref{eqn:covfefe-control}, and aggregating the updates \eqref{eqn:covfefe-aggregate}. We describe each in more detail.

Along with the server model $\xx$, \covfefe\ maintains a state for each client (client control variate $\cc_i$) and for the server (server control variate $\cc$). These are initialized to ensure that $\cc = \frac{1}{N}\sum \cc_i$ and can safely all be initialized to 0. In each round of communication, the server parameters $(\xx, \cc)$ are communicated to the participating clients $\cS \subset [N]$. Each participating client $i \in \cS$ initializes its local model with the server model $\yy_i \leftarrow \xx$. Then it makes a pass over its local data performing $K$ updates of the form:
\begin{equation}\label{eqn:covfefe-model}
  \yy_i \leftarrow \yy_i - \eta_l(g_i(\yy_i) + \cc - \cc_i)\,.
\end{equation}
Then, the local control variate $\cc_i$ is also updated. For this, we provide two options:\vspace{-2mm}
\begin{equation} \label{eqn:covfefe-control}
  \cc_i^+ \leftarrow \begin{cases}
  \text{Option I.}& g_i(\xx)\,, \text{or}\\
  \text{Option II.} & \cc_i - \cc + \tfrac{1}{K \eta_l}(\xx - \yy_{i})\,.
  \end{cases}\vspace{-2mm}
  \end{equation}
  Option I involves making an additional pass over the local data to compute the gradient at the server model $\xx$. Option II instead re-uses the previously computed gradients to update the control variate. Option I can be more stable than II depending on the application, but II is cheaper to compute and usually suffices (all our experiments use Option II). The client updates are then aggregated and used to update the server parameters: \vspace{-2mm}
  \begin{align}\label{eqn:covfefe-aggregate}
    \begin{split}
      \xx &\leftarrow \xx + \frac{\eta_g}{\abs{\cS}}\sum_{i \in \cS}(\yy_i - \xx)\,,\\
      \cc &\leftarrow \cc + \frac{1}{N}\sum_{i \in \cS}(\cc_i^+ - \cc_i)\,.\vspace{-4mm}
    \end{split}
  \end{align}
  This finishes one round of communication. Note that the clients in \covfefe\ are \emph{stateful} and retain the value of $\cc_i$ across multiple rounds. Further, if $\cc_i$ is always set to 0, then \covfefe\ becomes equivalent to \fedavg. The full details are summarized in Algorithm~\ref{alg:sampling}.

  \paragraph{Usefulness of control variates.} 
  \begin{figure}[t]
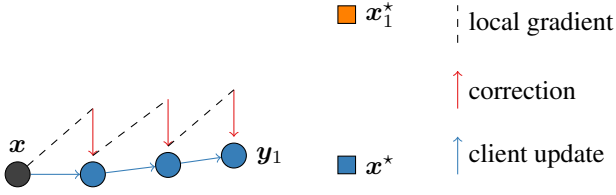

    \centering
    \ctikzfig{scaffold}  
    \caption{Update steps of \covfefe\ on a single client. The local gradient (dashed black) points to $\xx_1^\star$ (orange square), but the correction term $(\cc - \cc_i)$ (in red) ensures the update moves towards the true optimum $\xx^\star$ (black square).}\label{fig:scaffold}
  \end{figure}
  
  If communication cost was not a concern, the ideal update on client $i$ would be\vspace{-2mm}
  \begin{equation}\label{eqn:ideal-model}
    \yy_i \leftarrow \yy_i + \frac{1}{N}\sum_{j}g_j(\yy_i)\,.\vspace{-3mm}
  \end{equation}
  Such an update essentially computes an unbiased gradient of $f$ and hence becomes equivalent to running \fedavg\ in the iid case (which has excellent performance). Unfortunately such an update requires communicating with all clients for every update step. \covfefe\ instead uses control variates such that \vspace{-2mm}
  \[
      \cc_j \approx g_j(\yy_i) \text{ and } \cc \approx \frac{1}{N}\sum_{j}g_j(\yy_i)\,.\vspace{-3mm}
  \]
  Then, \covfefe\ \eqref{eqn:covfefe-model} mimics the ideal update \eqref{eqn:ideal-model} with\vspace{-2mm}
  \[ 
    (g_i(\yy_i) - \cc_i + \cc) \approx \frac{1}{N}\sum_{j}g_j(\yy_i)\,.\vspace{-3mm}
  \]
  Thus, the local updates of \covfefe\ remain synchronized and converge for arbitrarily heterogeneous clients.

\section{Convergence of \covfefe}\label{sec:convergence}
We state the rate of \covfefe\ without making any assumption on the similarity between the functions. See Appendix \ref{appsec:covfefe} for the full proof.
\begin{theorem}\label{thm:convergence}
 For any $\beta$-smooth functions $\{f_i\}$, the output of \covfefe\ has expected error smaller than $\e$ for $\eta_g = \sqrt{S}$ and some values of $\eta_l, R$ satisfying
      \begin{itemize}[topsep=0pt,itemsep=-1ex,partopsep=1ex,parsep=1ex]
        \item \textbf{Strongly convex:} $\eta_l \leq \min\rbr*{\tfrac{1}{81\beta K\eta_g}, \tfrac{S}{15 \mu N K\eta_g}}$ and 
        \[
          R = \tilde\cO\rbr[\bigg]{\frac{\sigma^2}{\mu K S \e}+ \frac{\beta}{\mu} + \frac{N}{S}}\,,
        \]
       \item \textbf{General convex:} $\eta_l \leq \tfrac{1}{81
       \beta K\eta_g}$ and
       \[
        R = \tilde\cO\rbr[\bigg]{\frac{\sigma^2D^2}{K S \e^2}+ \frac{\beta D^2}{\e} + \frac{N F}{S}}\,,
      \]
      \item \textbf{Non-convex:} $\eta_l \leq \frac{1}{24K\eta_g\beta}\rbr*{\frac{S}{N}}^{\frac{2}{3}}$ and
      \[
            R = \cO\rbr[\bigg]{\frac{\beta \sigma^2 F}{K S \e^2}+ \rbr[\bigg]{\frac{N}{S}}^{\frac{2}{3}}\frac{\beta F}{\e} }\,,
       \]
      \end{itemize}
      where $D := \norm{\xx^0 - \xx^\star}^2$ and $F := f(\xx^0) - f^\star$.
  \end{theorem}
Let us first examine the rates without client sampling ($S = N$). For the strongly convex case, the number of rounds becomes $\frac{\sigma^2}{\mu NK\e} + \frac{1}{\mu}$. This rate holds for arbitrarily heterogeneous clients unlike Theorem~\ref{thm:fedavg} and further matches that of SGD with $K$ times larger batch-size, proving that \covfefe\ is at least as fast as SGD. These rates also match known lower-bounds for distributed optimization \citep{arjevani2015communication} (up to acceleration) and are unimprovable in general. However in certain cases \covfefe\ is provably faster than SGD. We show this fact in Section~\ref{sec:interpolation}.

Now let $\sigma = 0$. Then our rates in the strongly-convex case are $\frac{1}{\mu} + \frac{N}{S}$ and $\rbr[\big]{\frac{N}{S}}^{\frac{2}{3}}\frac{1}{\e}$ in the non-convex case. These exactly match the rates of SAGA \citep{defazio2014saga,reddi2016fast}. In fact, when $\sigma = 0$, $K = 1$ and $S=1$, the update of \covfefe\ with option I reduces to SAGA where in each round consists of sampling one client $f_i$. Thus \covfefe\ can be seen as an extension of variance reduction techniques for federated learning, and one could similarly extend SARAH~\citep{nguyen2017sarah}, SPIDER~\citep{fang2018spider}, etc. Note that standard SGD with client sampling is provably slower and converges at a sub-linear rate even with $\sigma = 0$.

\paragraph{Proof sketch.} For simplicity, assume that $\sigma = 0$ and consider the ideal update of \eqref{eqn:ideal-model} which uses the full gradient $\nabla f(\yy)$ every step. Clearly, this would converge at a linear rate even with $S =1$. \fedavg\ would instead use an update $\nabla f_i(\yy)$. The difference between the ideal update \eqref{eqn:ideal-model} and the \fedavg\ update \eqref{eqn:fedavg-model} is $\norm{\nabla f_i(\yy) - \nabla f(\yy)}$. We need a bound on the gradient-dissimilarity as in \eqref{asm:heterogeneity} to bound this error. \covfefe\ instead uses the update $\nabla f_i(\yy) - \cc_i + \cc$, and the difference from ideal update becomes
\[
  \sum_{i}\norm{(\nabla f_i(\yy) - \cc_i + \cc) - \nabla f(\yy)}^2 \leq \sum_i\norm{\cc_i - \nabla f_i(\yy)}^2\,.
  \]
  Thus, the error is independent of how similar or dissimilar the functions $f_i$ are, and instead only depends on the quality of our approximation $\cc_i \approx \nabla f_i(\yy)$. Since $f_i$ is smooth, we can expect that the gradient $\nabla f_i(\yy)$ does not change too fast and hence is easy to approximate. Appendix~\ref{appsec:covfefe} translates this intuition into a formal proof.


  \section{Usefulness of local steps}\label{sec:interpolation}
  In this section we investigate when and why taking local steps might be useful over simply computing a large-batch gradient in distributed optimization. We will show that when the functions across the clients share some similarity, local steps can take advantage of this and converge faster. For this we consider quadratic functions and express their similarity with the $\delta$ parameter introduced in~(\ref{asm:hessian-similarity}). 

  \begin{theorem}\label{thm:interpolation}
    For any $\beta$-smooth quadratic functions $\{f_i\}$ with $\delta$ bounded Hessian dissimilarity \eqref{asm:hessian-similarity}, the output of \covfefe\ with $S=N$ (no sampling) has error smaller than $\e$ for $\eta_g = 1$ and some values of $\eta_l, R$ satisfying
         \begin{itemize}[topsep=0pt,itemsep=-1ex,partopsep=0.5ex,parsep=0.5ex]
           \item \textbf{Strongly convex:} $\eta_l \leq \tfrac{1}{15 K\delta + 8\beta}$ and 
           \[
             R = \tilde\cO\rbr[\bigg]{\frac{\beta \sigma^2}{\mu KN\e} + \frac{\beta + \delta K}{\mu K}}\,,
           \]
          \item \textbf{Weakly convex:} $\eta_l \leq \tfrac{1}{15 K\delta + 8\beta}$  and
          \[
           R = \cO\rbr[\bigg]{ \frac{\beta \sigma^2 F}{K N \e^2} + \frac{(\beta + \delta K )F}{K \e}}\,,
         \]\vspace{-3mm}
         \end{itemize}
         where we define $F: = (f(\xx^0) - f^\star)$.
     \end{theorem}
  When $\sigma = 0$ and $K$ is large, the complexity of \covfefe\ becomes $\frac{\delta}{\mu}$. In contrast DANE, which being a proximal point method also uses large $K$, requires $(\frac{\delta}{\mu})^2$ rounds \citep{shamir2014communication} which is significantly slower, or needs an additional backtracking-line search to match the rates of \covfefe\ \citep{yuan2019convergence}. Further, Theorem~\ref{thm:interpolation} is the first result to demonstrate improvement due to similairty for non-convex functions as far as we are aware.
  
  Suppose that $\{f_i\}$ are identical. Recall that $\delta$ in \eqref{asm:hessian-similarity} measures the Hessian dissimilarity between functions and so $\delta = 0$ for this case. Then Theorem~\ref{thm:interpolation} shows that the complexity of \covfefe\ is $\frac{\sigma^2}{\mu K N \e} + \frac{1}{\mu K}$ which (up to acceleration) matches the i.i.d. lower bound of \citep{woodworth2018graph}. In contrast, SGD with $K$ times larger batch-size would require $\frac{\sigma^2}{\mu K N \e} + \frac{1}{\mu}$ (note the absence of $K$ in the second term). Thus, for identical functions, \covfefe\ (and in fact even \fedavg) improves linearly with increasing number of local steps. In the other extreme, if the functions are arbitrarily different, we may have $\delta = 2 \beta$. In this case, the complexity of \covfefe\ and large-batch SGD match the lower bound of \citet{arjevani2015communication} for the heterogeneous case.

  The above insights can be generalized to when the functions are only somewhat similar. If the Hessians are $\delta$-close and $\sigma = 0$, then the complexity is $\frac{\beta + \delta K}{\mu K}$. This bound implies that the optimum number of local steps one should use is $K = \frac{\beta}{\delta}$. Picking a smaller $K$ increases the communication required whereas increasing it further would only waste computational resources. While this result is intuitive---if the functions are more `similar', local steps are more useful---Theorem~\ref{thm:interpolation} shows that it is the similarity of the \emph{Hessians} which matters. This is surprising since the Hessians of $\{f_i\}$ may be identical even if their individual optima $\xx_i^\star$ are arbitrarily far away from each other and the gradient-dissimilarity \eqref{asm:heterogeneity} is unbounded.

  \paragraph{Proof sketch.} Consider a simplified \covfefe\ update with $\sigma = 0$ and no sampling ($S = N$):
  \[
    \yy_i = \yy_i - \eta(\nabla f_i(\yy_i) + \nabla f(\xx) - \nabla f_i(\xx))\,.
  \]
  We would ideally want to perform the update $\yy_i = \yy_i - \eta\nabla f(\yy_i)$ using the full gradient $\nabla f(\yy_i)$. We reinterpret the correction term of \covfefe\ $(\cc - \cc_i)$ as performing the following first order correction to the local gradient $\nabla f_i(\yy_i)$ to make it closer to the full gradient $\nabla f(\yy_i)$:
  \begin{align*}
    \underbrace{\nabla f_i(\yy_i) - \nabla f_i(\xx)}_{\approx \nabla^2 f_i(\xx) (\yy_i - \xx)} &+ \underbrace{\nabla f(\xx)}_{\approx \nabla f(\yy_i) + \nabla^2 f(\xx)(\xx - \yy_i)}\\
    &\hspace{-10mm}\approx \nabla f(\yy_i) + (\nabla^2 f_i(\xx) - \nabla^2 f(\xx)) (\yy_i - \xx)\\
    &\hspace{-10mm}\approx \nabla f(\yy_i) + \delta (\yy_i - \xx)
  \end{align*}
  Thus the \covfefe\ update approximates the ideal update up to an error $\delta$. This intuition is proved formally for quadratic functions in Appendix~\ref{appsec:interpolation}. Generalizing these results to other functions is a challenging open problem.

  
  \begin{figure}[!t]
    \centering
    \captionsetup[subfigure]{position=b,format=myformat}
    \begin{subfigure}{0.34\columnwidth}
      \centering
      \includegraphics[width=\linewidth]{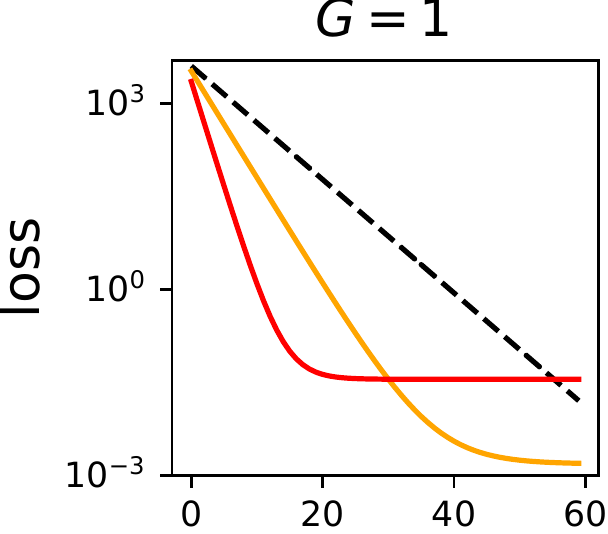}
    \end{subfigure}
    \begin{subfigure}{.31\columnwidth}
      \centering
      \includegraphics[width=\linewidth]{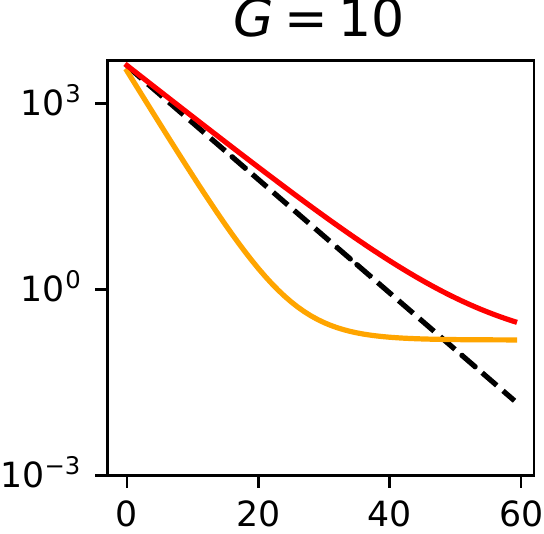}
    \end{subfigure}
    \begin{subfigure}{.31\columnwidth}
      \centering
      \includegraphics[width=\linewidth]{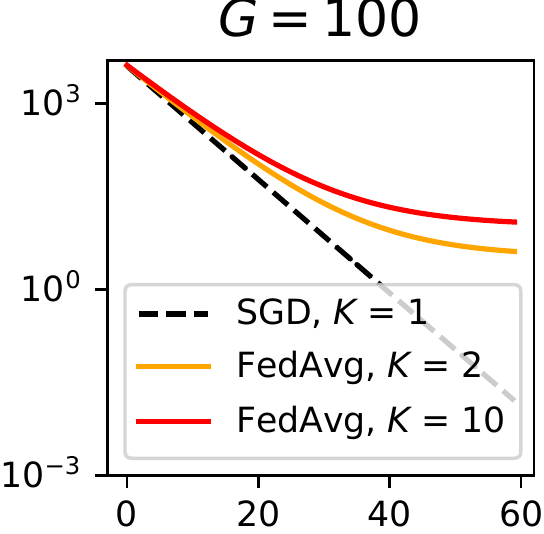}
    \end{subfigure}
    \begin{subfigure}{.34\columnwidth}
      \centering
      \includegraphics[width=\linewidth]{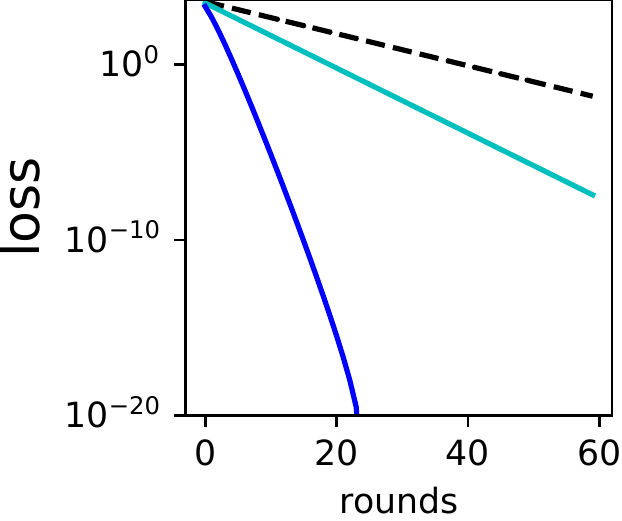}
    \end{subfigure}
    \begin{subfigure}{0.31\columnwidth}
      \centering
      \includegraphics[width=\linewidth]{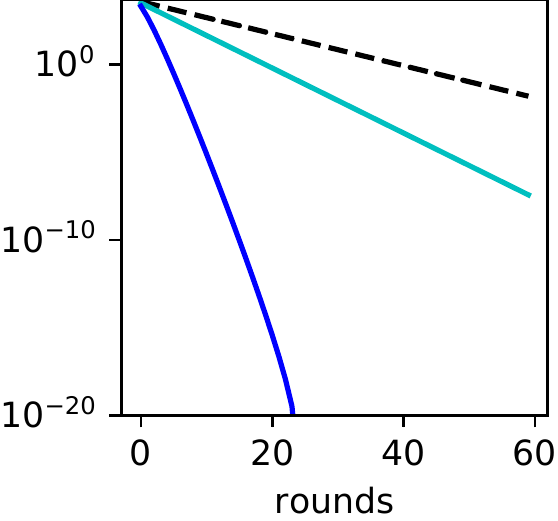}
    \end{subfigure}
    \begin{subfigure}{.31\columnwidth}
      \centering
      \includegraphics[width=\linewidth]{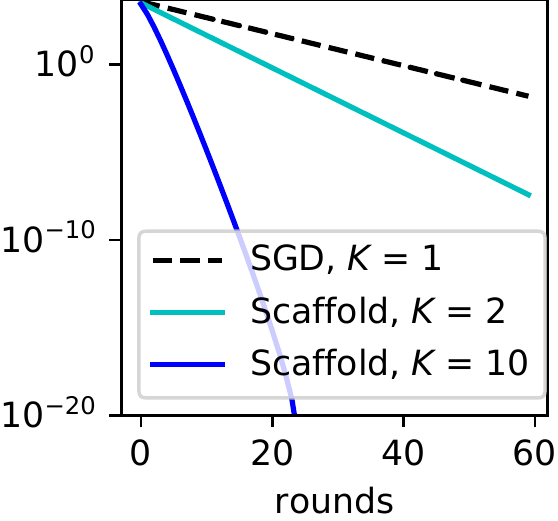}
    \end{subfigure}
    \caption{SGD (dashed black), FedAvg (above), and SCAFFOLD (below) on simulated data. FedAvg gets worse as local steps increases with $K=10$ (red) worse than $K=2$ (orange). It also gets slower as gradient-dissimilarity ($G$) increases (to the right). SCAFFOLD significantly improves with more local steps, with $K=10$ (blue) faster than $K=2$ (light blue) and SGD. Its performance is identical as we vary heterogeneity ($G$).}
    \label{fig:simulated}\vspace{-5mm}
  \end{figure}


\section{Experiments}\label{sec:experiments}
We run experiments on both simulated and real datasets to confirm our theory. Our main findings are i) \covfefe\ consistently outperforms SGD and \fedavg\ across all parameter regimes, and ii) the benefit (or harm) of local steps depends on both the algorithm and the similarity of the clients data.

\subsection{Setup}
Our simulated experiments uses $N=2$ quadratic functions based on our lower-bounds in Theorem~\ref{thm:lowerbound}. We use full-batch gradients ($\sigma = 0$) and no client sampling. Our real world experiments run logistic regression (convex) and 2 layer fully connected network (non-convex) on the \emnist\ \citep{cohen2017emnist}. We divide this dataset among $N=100$ clients as follows: for $s\%$ similar data we allocate to each client $s\%$ i.i.d. data and the remaining $(100 - s)\%$ by sorting according to label (cf. \citet{hsu2019measuring}).

We consider four algorithms: SGD, \fedavg\, \covfefe\, and \easgd\ with SGD as the local solver \citep{li2018fedprox}. On each client SGD uses the full local data to compute a single update, whereas the other algorithms take 5 steps per epoch (batch size is 0.2 of local data). We always use global step-size $\eta_g = 1$ and tune the local step-size $\eta_l$ individually for each algorithm. \covfefe\ uses option II (no extra gradient computations) and \easgd\ has fixed regularization~$= 1$ to keep comparison fair. Additional tuning of the regularization parameter may sometimes yield improved empirical performance.

\subsection{Simulated results}
The results are summarized in Fig.~\ref{fig:simulated}. Our simulated data has Hessian difference $\delta = 1$ \eqref{asm:hessian-similarity} and $\beta = 1$. We vary the gradient heterogeneity \eqref{asm:heterogeneity} as $G \in [1, 10, 100]$. For all valued of $G$, \fedavg\ gets slower as we increase the number of local steps. This is explained by the fact that client-drift increases as we increase the number of local steps, hindering progress. Further, as we increase $G$, \fedavg\ continues to slow down exactly as dictated by Thms.~\ref{thm:fedavg} and~\ref{thm:lowerbound}. Note that when heterogeneity is small ($G = \beta = 1$), \fedavg\ can be competitive with SGD. 

\covfefe\ is consistently faster than SGD, with $K=2$ being twice as fast and $K = 10$ about 5 times faster. Further, its convergence is completely unaffected by $G$, confirming our theory in Thm.~\ref{thm:convergence}. The former observation that we do not see linear improvement with $K$ is explained by Thm.~\ref{thm:interpolation} since we have $\delta > 0$. This sub linear improvement is still significantly faster than both SGD and \fedavg.

\subsection{\emnist\ results}
\begin{table*}[!t]
  \caption{Communication rounds to reach $0.5$ test accuracy for logistic regression on \emnist\ as we vary number of epochs. 1k+ indicates 0.5 accuracy was not reached even after 1k rounds, and similarly an arrowhead indicates that the barplot extends beyond the table. 1 epoch for local update methods corresponds to 5 local steps (0.2 batch size), and 20\% of clients are sampled each round. We fix $\mu=1$ for \easgd\ and use variant (ii) for \covfefe\ to ensure all methods are comparable. Across all parameters (epochs and similarity), \covfefe\ is the fastest method. When similarity is 0 (sorted data), \fedavg\ consistently gets worse as we increase the number of epochs, quickly becoming slower than SGD. \covfefe\ initially gets worse and later stabilizes, but is always at least as fast as SGD. As similarity increases (i.e. data is more shuffled), both \fedavg\ and \covfefe\ significantly outperform SGD though \covfefe\ is still better than \fedavg. Further, with higher similarity, both methods benefit from increasing number of epochs.}
  \centering
  \includegraphics[width=\linewidth]{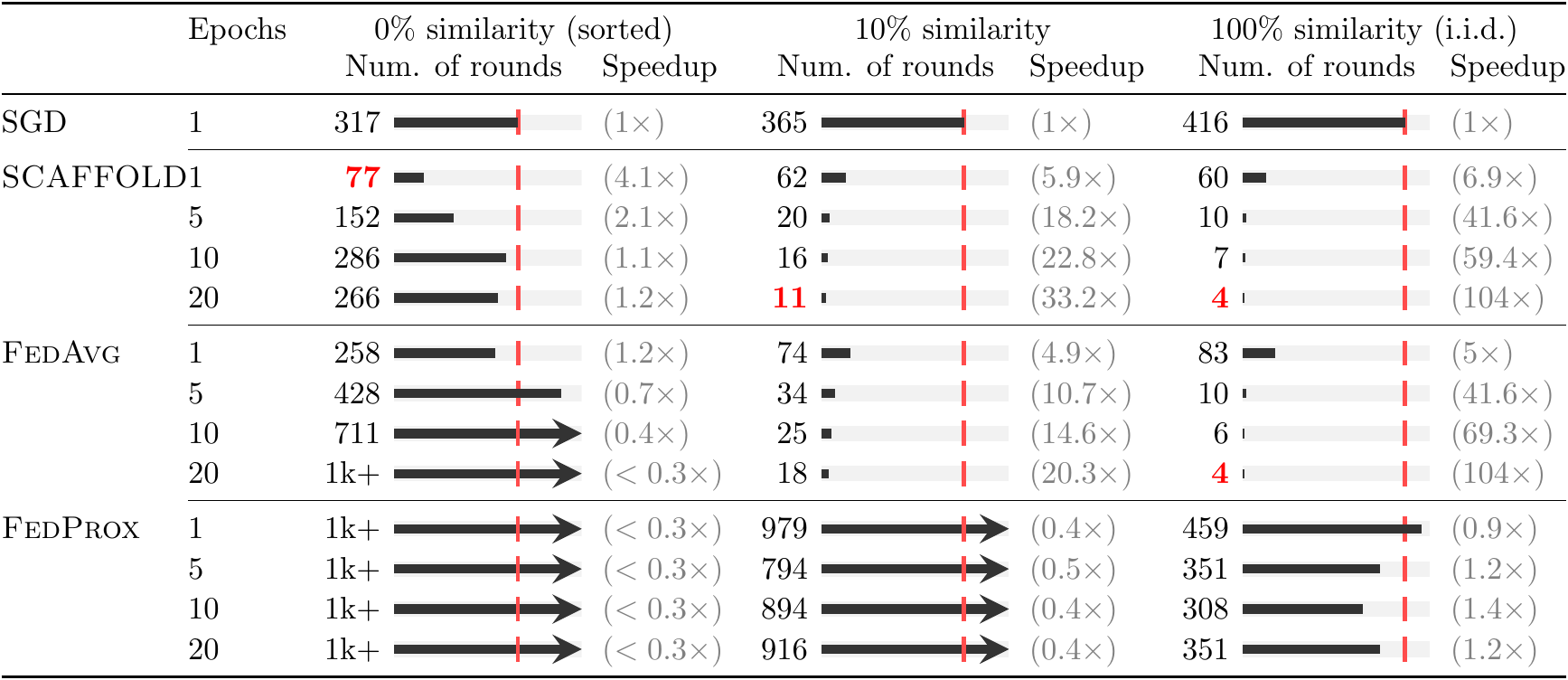}
  \label{tab:similarity}
  \end{table*}

We run extensive experiments on the \emnist\ dataset to measure the interplay between the algorithm, number of epochs (local updates), number of participating clients, and the client similarity. Table~\ref{tab:similarity} measures the benefit (or harm) of using more local steps, Table~\ref{tab:sampling} studies the resilience to client sampling, and Table~\ref{tab:nonconvex} reports preliminary results on neural networks. We are mainly concerned with minimizing the number of \emph{communication rounds}. We observe that

\textbf{SCAFFOLD is consistently the best.} Across all range of values tried, we observe that \covfefe\ outperforms SGD, \fedavg, and \easgd. The latter \easgd\ is always slower than the other local update methods, though in some cases it outperforms SGD. Note that it is possible to improve \easgd\ by carefully tuning the regularization parameter \citep{li2018fedprox}. \fedavg\ is always slower than \covfefe\ and faster than \easgd.

\textbf{SCAFFOLD $>$ SGD $>$ FedAvg\ for heterogeneous clients.} When similarity is $0\%$, \fedavg\ gets slower with increasing local steps. If we take more than 5 epochs, its performance is worse than SGD's. \covfefe\ initially worsens as we increase the number of epochs but then flattens. However, its performance is always better than that of SGD, confirming that it can handle heterogeneous data. 

\textbf{SCAFFOLD and FedAvg get faster with more similarity, but not SGD.} As similarity of the clients increases, the performance of SGD remains relatively constant. On the other hand, \covfefe\ and \fedavg\ get significantly faster as similarity increases. Further, local steps become much more useful, showing monotonic improvement with the increase in number of epochs. This is because with increasing the i.i.d.ness of the data, both the gradient and Hessian dissimilarity decrease.

\textbf{SCAFFOLD is resilient to client sampling.} As we decrease the fraction of clients sampled, \covfefe\ ,and \fedavg\ only show a sub-linear slow-down. They are more resilient to sampling in the case of higher similarity.

\begin{table}[!t]
  \caption{Communication rounds to reach 0.45 test accuracy for logistic regression on \emnist\ as we vary the number of sampled clients. Number of epochs is kept fixed to 5. \covfefe\ is consistently faster than \fedavg. As we decrease the number of clients sampled in each round, the increase in number of rounds is sub-linear. This slow-down is better for more similar clients.\vspace{1mm}}
  \centering\resizebox{\linewidth}{!}{
  \begin{threeparttable}
  \begin{tabular}{@{}llllll@{}}
  \toprule
  & Clients &\multicolumn{2}{c}{0\%~similarity} & \multicolumn{2}{c}{10\%~similarity}  \\
  \midrule
  \covfefe & 20\% & 143 &\speedup{1.0} &9 &\speedup{1.0}\\
   & 5\% & 290 &\speedup{2.0} &13 &\speedup{1.4}\\
   & 1\% & 790 &\speedup{5.5} &28 &\speedup{3.1}\\
  \cmidrule{2-6}
  \fedavg & 20\% & 179 &\speedup{1.0} &12 &\speedup{1.0}\\
  & 5\% & 334 &\speedup{1.9} &17 &\speedup{1.4}\\
  & 1\% & 1k+ & {\color{gray}(5.6+$\times$)} &35 &\speedup{2.9}\\
  \bottomrule
  \end{tabular}
  \end{threeparttable}}
  \label{tab:sampling}
  \end{table}

\textbf{SCAFFOLD outperforms FedAvg on non-convex experiments.} We see that \covfefe\ is better than \fedavg\ in terms of final test accuracy reached, though interestingly \fedavg\ seems better than SGD even when similarity is 0. However, much more extensive experiments (beyond current scope) are needed before drawing conclusions.
 
\begin{table}[!t]
  \caption{Best test accuracy after 1k rounds with 2-layer fully connected neural network (non-convex) on \emnist\ trained with 5 epochs per round (25 steps) for the local methods, and 20\% of clients sampled each round. \covfefe\ has the best accuracy and SGD has the least. \covfefe\ again outperforms other methods. SGD is unaffected by similarity, whereas the local methods improve with client similarity.\vspace{2mm}}
  \centering
  \begin{threeparttable}
  \begin{tabular}{@{}lll@{}}
  \toprule
  & 0\%~similarity & 10\%~similarity  \\
  \midrule
  SGD & 0.766 & 0.764\\
  \fedavg & 0.787 & 0.828\\
  \covfefe & \textbf{0.801} & \textbf{0.842}\\
  \bottomrule
  \end{tabular}
  \end{threeparttable}
  \label{tab:nonconvex}
  \end{table}

\section{Conclusion}
Our work studied the impact of heterogeneity on the performance of optimization methods for federated learning. Our careful theoretical analysis showed that \fedavg\ can be severely hampered by \emph{gradient dissimilarity}, and can be even slower than SGD. We then proposed a new stochastic algorithm (\covfefe) which overcomes gradient dissimilarity using control variates. We demonstrated the effectiveness of \covfefe\ via strong convergence guarantees and empirical evaluations. Further, we showed that while \covfefe\ is always at least as fast as SGD, it can be much faster depending on the \emph{Hessian dissimilarity} in our data. Thus, different algorithms can take advantage of (and are limited by) different notions of dissimilarity. We believe that characterizing and isolating various dissimilarities present in real world data can lead to further new algorithms and significant impact on distributed, federated, and decentralized learning.
\clearpage
\paragraph{Acknowledgments.}
We thank Filip Hanzely and Jakub Kone{\v{c}}n{\`y} for discussions regarding variance reduction techniques and Blake Woodworth, Virginia Smith and Kumar Kshitij Patel for suggestions which improved the writing.

%% file: appendix.tex


\section{Related work and significance}\label{appsec:related}
\paragraph{Federated learning.} As stated earlier, federated learning involves learning a centralized model from distributed client data. This centralized model benefits from all client data and can often
result in a beneficial performance e.g. in
including next word prediction \citep{hard2018federated,
  yang2018applied}, emoji prediction \citep{ramaswamy2019federated},
decoder models \citep{chen2019federatedb}, vocabulary estimation
\citep{chen2019federated}, low latency vehicle-to-vehicle
communication \citep{samarakoon2018federated}, and predictive models
in health \citep{brisimi2018federated}. Nevertheless, federated
learning raises several types of issues and has been the topic of
multiple research efforts studying the issues of generalization and fairness \citep{mohri2019agnostic, li2019fair}, the design of more efficient communication
strategies \citep{konevcny2016federated,konecny2016federated2,
  suresh2017distributed,stich2018sparsified,karimireddy2019error,
  basu2019qsparse}, the study of lower bounds
\citep{woodworth2018graph}, differential privacy
guarantees \citep{agarwal2018cpsgd},
security \citep{bonawitz2017practical}, etc. 
We refer to \citet{kairouz2019advances} for an in-depth survey of this area.

\paragraph{Convergence of \fedavg}  For identical clients, \fedavg\ coincides with
parallel SGD analyzed by \cite{zinkevich2010parallelized} who proved asymptotic convergence. \citet{stich2018local} and, more
recently \citet{stich2019error,patel2019communication,khaled2020tighter},
gave a sharper analysis of the same method, under the name of local
SGD, also for identical functions. However, there still remains a gap
between their upper bounds and the lower bound of
\citet{woodworth2018graph}. The analysis of \fedavg\ for heterogeneous clients is more delicate
due to the afore-mentioned client-drift, first empirically observed by \citet{zhao2018federated}. Several
analyses bound this drift by assuming bounded gradients
\citep{wang2019adaptive,yu2019parallel}, or view it as additional noise
\citep{khaled2020tighter}, or assume that the client optima are $\epsilon$-close \citep{li2018fedprox,haddadpour2019convergence}. In a concurrent work, \citep{liang2019variance} propose to use variance reduction to deal with client heterogeneity but still show rates slower than SGD. We summarize the communication complexities of different methods for heterogeneous clients in Table~\ref{tab:results}.

\paragraph{Variance reduction.} The use of \emph{control variates} is a classical technique to reduce
variance in Monte Carlo sampling methods
(cf.\ \cite{glasserman2013monte}). In optimization, they were used for
finite-sum minimization by SVRG
\citep{johnson2013accelerating,zhang2013linear} and then in SAGA
\citep{defazio2014saga} to simplify the linearly convergent method SAG
\citep{schmidt2017minimizing}. Numerous variations and extensions of
the technique are studied in \citep{hanzely2019one}. Starting from
\citep{reddi2016stochastic}, control variates have also frequently
been used to reduce variance in finite-sum non-convex settings
\citep{reddi2016fast,nguyen2018inexact,fang2018spider,tran2019hybrid}.
Further, they are used to obtain
linearly converging decentralized algorithms under the guise of `gradient-tracking'
in \citep{shi2015extra,nedich2016geometrically} and for gradient
compression as `compressed-differences' in
\citep{mishchenko2019distributed}. Our method can be viewed as seeking
to remove the `client-variance' in the gradients across the clients, though
there still remains additional stochasticity as in \citep{kulunchakov2019estimate}, which is important in deep learning \citep{defazio2019ineffectiveness}.

\paragraph{Distributed optimization.} The problem of client-drift we described is a common phenomenon in
distributed optimization. In fact, classic techniques such as ADMM
mitigate this drift, though they are not applicable in federated
learning. For well structured convex problems, CoCoA uses the dual
variable as the control variates, enabling flexible distributed
methods \citep{smith2018cocoa}. DANE by \cite{shamir2014communication}
obtain a closely related primal only algorithm, which was later
accelerated by \citet{reddi2016aide} and recently extended to federated learning \citep{li2020feddane}. \covfefe\ can be
viewed as an improved version of DANE where a fixed number of
(stochastic) gradient steps are used in place of a proximal point update. In a similar spirit, distributed variance reduction techniques have been proposed for the finite-sum case \citep{lee2015distributed,konecny2016federated2,cen2019convergence}. However, these methods are restricted to finite-sums and are not applicable to the stochastic setting studied here.

\newpage

\section{Technicalities}
We examine some additional definitions and introduce some technical lemmas.

\subsection{Additional definitions}
We make precise a few definitions and explain some of their implications.
\begin{enumerate}
    \myitem{A3}
      \label{asm:strong-convexity} $f_i$ is \textbf{$\mu$-convex} for $\mu \geq 0$ and satisfies:
      \begin{equation*}\label{eqn:strong-convexity}
  \inp{\nabla f_i(\xx)}{\yy - \xx} \leq - \rbr*{f_i(\xx) - f_i(\yy) + \frac{\mu}{2}\norm{\xx - \yy}^2}\,, \text{ for any } i, \xx, \yy\,.
      \end{equation*}
      Here, we allow that $\mu = 0$ (we refer to this case as the general convex case as opposed to strongly convex). It is also possible to generalize all proofs here to the weaker notion of PL-strong convexity \citep{karimi2016linear}.
    \myitem{A4}\label{asm:variance} $g_i(\xx) := \nabla f_i(\xx;\zeta_i)$ is unbiased stochastic gradient of $f_i$ with \textbf{bounded variance}
      \begin{equation*}\label{eqn:variance}
  \expect_{\zeta_i}[\norm{g_i(\xx) - \nabla f_i(\xx)}^2] \leq \sigma^2\,, \text{ for any } i, \xx\,.
      \end{equation*}
      Note that \eqref{asm:variance} only bounds the variance within the same client, but not the variance across the clients.
    \myitem{A5}\label{asm:smoothness} $\{f_i\}$ are \textbf{$\beta$-smooth} and satisfy:
      \begin{equation}
        \label{eqn:lip-grad}
        \norm{\nabla f_i(\xx) - \nabla f_i(\yy)} \leq \beta \norm{\xx - \yy}\,, \text{ for any } i, \xx, \yy\,.
      \end{equation}
    \end{enumerate}
    The assumption \eqref{asm:smoothness} also implies the following quadratic upper bound on $f_i$
    \begin{equation}\label{eqn:quad-upper}
        f_i(\yy) \leq f_i(\xx) + \inp{\nabla f_i(\xx)}{\yy - \xx} + \frac{\beta}{2}\norm{\yy - \xx}^2\,.
    \end{equation}
    If additionally the function $\{f_i\}$ are convex and $\xx^\star$ is an optimum of $f$, \eqref{asm:smoothness} implies (via \citet{nesterov2018lectures}, Theorem 2.1.5)
    \begin{align}\label{eqn:smoothness}
      \begin{split}
        \frac{1}{2\beta N}\sum_{i=1}^N\norm{\nabla f_i(\xx) - \nabla
        f_i(\xx^\star)}^2 \leq f(\xx) - f^\star\,.
      \end{split}
      \end{align}
      Further, if $f_i$ is twice-differentiable, \eqref{asm:smoothness} implies that $\norm{\nabla^2 f_i(\xx)} \leq \beta$ for any $\xx$.

\subsection{Some technical lemmas}

Now we cover some technical lemmas which are useful for
computations later on.
The two lemmas below are useful to unroll recursions and derive convergence
rates.
The first one is a slightly improved (and simplified) version of \citep[Theorem~2]{stich2019unified}. It is straightforward to remove the additional logarithmic terms if we use a varying step-size \citep[Lemma~13]{kulunchakov2019estimate}.

\begin{lemma}[linear convergence rate]
\label{lem:constant}
For every non-negative sequence $\{d_{r-1}\}_{r \geq 1}$ and any
parameters $\mu > 0$, $\eta_{\max} \in (0, 1/\mu]$, $c \geq 0$,
$R \geq \frac{1}{2\eta_{\max} \mu}$, there exists a constant step-size
$\eta \leq \eta_{\max}$ and weights $w_r:=(1-\mu\eta)^{1-r}$ such that
for $W_R := \sum_{r=1}^{R+1} w_r$,
\begin{align*}
    \Psi_R := \frac{1}{W_R}\sum_{r=1}^{R+1} \left( \frac{w_r}{\eta} \left(1-\mu \eta \right) d_{r-1} - \frac{w_r}{\eta} d_{r} + c \eta w_r \right) = \tilde \cO \left( \mu d_0 \exp\rbr*{- \mu\eta_{\max}  R } + \frac{c}{\mu R}  \right)\,.
\end{align*}
\end{lemma}
\begin{proof}
  By substituting the value of $w_r$, we observe that we end up with a
  telescoping sum and estimate
    \begin{align*}
     \Psi_R = \frac{1}{\eta W_R} \sum_{r=1}^{R+1} \left(w_{r-1}d_{r-1} - w_{r} d_{r} \right) + \frac{c\eta}{W_R}\sum_{r=1}^{R+1} w_r \leq \frac{d_0}{\eta W_R} +  c \eta \,.
    \end{align*}
    When $R \geq \frac{1}{2\mu\eta}$, $(1 - \mu \eta)^R \leq \exp(-\mu\eta R) \leq \frac{2}{3}$. For such an $R$, we can lower bound $\eta W_R$ using
    \[
\eta W_R = \eta (1 - \mu \eta)^{-R} \sum_{r=0}^{R} (1 - \mu \eta)^r = \eta (1 - \mu \eta)^{-R} \frac{1 - (1 - \mu \eta)^R}{\mu \eta} \geq (1 - \mu \eta)^{-R} \frac{1}{3\mu}\,.
    \]
    This proves that for all $R \geq \frac{1}{2\mu \eta}$,
    \[
 \Psi_R \leq 3\mu d_0 (1 - \mu \eta)^{R}  +  c \eta \leq 3 \mu d_o \exp(-\mu \eta R) + c\eta\,.
    \]
    The lemma now follows by carefully tuning $\eta$. Consider the following two cases depending on the magnitude of $R$ and $\eta_{\max}$:
    \begin{itemize}
\item Suppose $\frac{1}{2\mu R} \leq \eta_{\max} \leq \frac{\log(\max(1,\mu^2 R d_0/ c))}{\mu R}$. Then we can choose $\eta = \eta_{\max}$,
\[
    \Psi_R \leq 3\mu d_0 \exp\left[-\mu \eta_{\max} R \right] + c \eta_{\max} \leq 3\mu d_0 \exp\left[-\mu \eta_{\max} R \right] + \tilde\cO\rbr[\bigg]{\frac{c}{\mu R}}\,.
\]
\item Instead if $\eta_{\max} > \frac{\log(\max(1,\mu^2 R d_0/ c))}{\mu R}$, we pick $\eta = \frac{\log(\max(1,\mu^2 R d_0/ c))}{\mu R}$ to claim that
\[
    \Psi_R \leq 3 \mu d_0 \exp\left[-\log(\max(1,\mu^2 R d_0/ c))\right] + \tilde\cO\rbr[\bigg]{\frac{c}{\mu R}} \leq \tilde\cO\rbr[\bigg]{\frac{c}{\mu R}} \,.
\]
    \end{itemize}
    \end{proof}

    The next lemma is an extension of \citep[Lemma 13]{stich2019error}, \citep[Lemma~13]{kulunchakov2019estimate} and is useful to derive convergence rates for general convex functions ($\mu=0$) and non-convex functions.
\begin{lemma}[sub-linear convergence rate]\label{lemma:general}
For every non-negative sequence $\{d_{r-1}\}_{r \geq 1}$ and any
parameters $\eta_{\max}  \geq 0$, $c \geq 0$,
$R \geq 0$, there exists a constant step-size
$\eta \leq \eta_{\max}$ and weights $w_r  = 1$ such that,
\begin{align*}
    \Psi_R := \frac{1}{R+1}\sum_{r=1}^{R+1} \left( \frac{d_{r-1}}{\eta}  - \frac{d_r}{\eta} + c_1 \eta +c_2 \eta^2\right) \leq  \frac{d_0}{\eta_{\max}(R+1)} + \frac{2\sqrt{c_1 d_0}}{\sqrt{R +1}} + 2\rbr[\bigg]{\frac{d_0}{R+1}}^{\frac{2}{3}} c_2^{\frac{1}{3}}\,.
\end{align*}
\end{lemma}
\begin{proof} Unrolling the sum, we can simplify
\[
    \Psi_R \leq \frac{d_0}{\eta (R+1)} + c_1 \eta + c_2 \eta^2\,.
\]
Similar to the strongly convex case (Lemma~\ref{lem:constant}), we distinguish the following cases:
\begin{itemize}
    \item When $R+1 \leq \frac{d_0}{c_1 \eta_{\max}^2}$, and $R+1 \leq \frac{d_0}{c_2 \eta_{\max}^3}$ we pick $\eta = \eta_{\max}$ to claim
    \[
        \Psi_R \leq \frac{d_0}{\eta_{\max} (R+1)} + c_1 \eta_{\max} + c_2 \eta_{\max}^2 \leq \frac{d_0}{\eta_{\max} (R+1)} + \frac{\sqrt{c_1 d_0}}{\sqrt{R +1}}+ \rbr[\bigg]{\frac{d_0}{R+1}}^{\frac{2}{3}} c_2^{\frac{1}{3}} \,.
    \]
    \item In the other case, we have $\eta_{\max}^2 \geq \frac{d_0}{c_1(R+1)}$ or $\eta_{\max}^3 \geq \frac{d_0}{c_2(R+1)}$. We choose $\eta = \min\cbr[\bigg]{\sqrt{\frac{d_0}{c_1(R+1)}}, \sqrt[3]{\frac{d_0}{c_2(R+1)}}}$ to prove
    \[
        \Psi_R \leq \frac{d_0}{\eta (R+1)} + c \eta = \frac{2\sqrt{c_1 d_0}}{\sqrt{R +1}} + 2\sqrt[3]{\frac{d_0^2 c_2}{(R+1)^2}} \,.\vspace{-5mm}
    \]
\end{itemize}
\end{proof}

    Next, we state a relaxed triangle inequality true for the squared
    $\ell_2$ norm.
\begin{lemma}[relaxed triangle inequality]\label{lem:norm-sum}
    Let $\{\vv_1,\dots,\vv_\tau\}$ be $\tau$ vectors in $\R^d$. Then the following are true:
    \begin{enumerate}
\item $\norm{\vv_i + \vv_j}^2 \leq (1 + a)\norm{\vv_i}^2 + (1 + \tfrac{1}{a})\norm{\vv_j}^2$ for any $a >0$, and
\item $\norm{\sum_{i=1}^\tau \vv_i}^2 \leq \tau \sum_{i=1}^\tau\norm{\vv_i}^2$.
    \end{enumerate}
\end{lemma}
\begin{proof}
The proof of the first statement for any $a > 0$ follows from the identity:
\[
    \norm{\vv_i + \vv_j}^2 = (1 + a)\norm{\vv_i}^2 + (1 + \tfrac{1}{a})\norm{\vv_j}^2 - \norm{\sqrt{a} \vv_i + \tfrac{1}{\sqrt{a}}\vv_j}^2\,.
\]
For the second inequality, we use the convexity of
$\xx \rightarrow \norm{\xx}^2$ and Jensen's inequality
\[
     \norm[\bigg]{\frac{1}{\tau}\sum_{i=1}^\tau \vv_i }^2 \leq \frac{1}{\tau}\sum_{i=1}^\tau\norm[\big]{ \vv_i }^2\,.
\]
\end{proof}

Next we state an elementary lemma about expectations of norms of
random vectors.
\begin{lemma}[separating mean and variance]\label{lem:independent}
Let $\{\Xi_1,\dots,\Xi_{\tau}\}$ be $\tau$ random variables in $\R^d$ which are not necessarily independent. First suppose that their mean is $\E[\Xi_i] = \xi_i$ and variance is bounded as $\E[\norm{\Xi_i - \xi_i}^2]\leq \sigma^2$. Then, the following holds
\[
    \E[\norm{\sum_{i=1}^\tau \Xi_i}^2] \leq \norm{\sum_{i=1}^\tau \xi_i}^2+ \tau^2 \sigma^2\,.
\]
Now instead suppose that their \emph{conditional mean} is $\E[\Xi_i | \Xi_{i-1}, \dots \Xi_{1}] = \xi_i$ i.e. the variables $\{\Xi_i - \xi_i\}$ form a martingale difference sequence, and the variance is bounded by $\E[\norm{\Xi_i - \xi_i}^2]\leq \sigma^2$ as before. Then we can show the tighter bound
\[
    \E[\norm{\sum_{i=1}^\tau \Xi_i}^2] \leq 2\norm{\sum_{i=1}^\tau \xi_i}^2+ 2\tau \sigma^2\,.
\]
\end{lemma}
\begin{proof}
For any random variable $X$, $\E[X^2] = (\E[X - \E[X]])^2 + (\E[X])^2$ implying
\[
     \E[\norm{\sum_{i=1}^\tau \Xi_i}^2] = \norm{\sum_{i=1}^\tau \xi_i}^2 +  \E[\norm{\sum_{i=1}^\tau \Xi_i - \xi_i}^2]  \,.
\]
Expanding the above expression using relaxed triangle inequality (Lemma~\ref{lem:norm-sum}) proves the first claim:
\[
   \E[\norm{\sum_{i=1}^\tau \Xi_i - \xi_i}^2]  \leq \tau\sum_{i=1}^\tau \E[\norm{ \Xi_i - \xi_i}^2] \leq \tau^2\sigma^2\,.
\]
For the second statement, $\xi_i$ is not deterministic and depends on $\Xi_{i-1}, \dots, \Xi_1$. Hence we have to resort to the cruder relaxed triangle inequality to claim
\[
    \E[\norm{\sum_{i=1}^\tau \Xi_i}^2] \leq 2\norm{\sum_{i=1}^\tau \xi_i}^2 +  2\E[\norm{\sum_{i=1}^\tau \Xi_i - \xi_i}^2] 
\]
and then use the tighter expansion of the second term:
\[
   \E[\norm{\sum_{i=1}^\tau \Xi_i - \xi_i}^2]  = \sum_{i, j} \E\sbr*{(\Xi_i - \xi_i)^\top(\Xi_j - \xi_j)} = \sum_{i} \E\sbr*{\norm{\Xi_i - \xi_i}^2} \leq \tau \sigma^2\,.
\]
The cross terms in the above expression have zero mean since $\{\Xi_i - \xi_i\}$ form a martingale difference sequence.
\end{proof}


\newpage
\section{Properties of convex functions}

We now study two lemmas which hold for any smooth and strongly-convex functions. The first is a generalization of the standard strong convexity inequality \eqref{eqn:strong-convexity}, but can handle gradients computed at slightly perturbed points.
 \begin{lemma}[perturbed strong convexity]\label{lem:magic}
    The following holds for any $\beta$-smooth and $\mu$-strongly convex function $h$, and any $\xx, \yy, \zz$ in the domain of $h$:
    \[
\inp{\nabla h(\xx)}{\zz - \yy} \geq h(\zz) - h(\yy) +\frac{\mu}{4}\norm{\yy - \zz}^2  - \beta \norm{\zz - \xx}^2\,.
    \]
\end{lemma}
\begin{proof}
    Given any $\xx$, $\yy$, and $\zz$, we get the following two inequalities using smoothness and strong convexity of $h$:
    \begin{align*}
\inp{\nabla h(\xx)}{\zz - \xx} &\geq h(\zz) - h(\xx) - \frac{\beta}{2}\norm{\zz - \xx}^2\\
\inp{\nabla h(\xx)}{\xx - \yy} &\geq h(\xx) - h(\yy) + \frac{\mu}{2}\norm{\yy - \xx}^2 \,.
    \end{align*}
    Further, applying the relaxed triangle inequality gives
    \[
\frac{\mu}{2}\norm{\yy - \xx}^2 \geq \frac{\mu}{4}\norm{\yy - \zz}^2 - \frac{\mu}{2}\norm{\xx - \zz}^2\,.
    \]
    Combining all the inequalities together we have
    \[
\inp{\nabla h(\xx)}{\zz - \yy} \geq h(\zz) - h(\yy) +\frac{\mu}{4}\norm{\yy - \zz}^2  - \frac{\beta + \mu}{2}\norm{\zz - \xx}^2\,.
    \]
    The lemma follows since $\beta \geq \mu$.
\end{proof}

Here, we see that a gradient step is a contractive operator.
\begin{lemma}[contractive mapping]\label{lem:contractive}
For any $\beta$-smooth and $\mu$-strongly convex function $h$, points $\xx, \yy$ in the domain of $h$, and step-size $\eta \leq \frac{1}{\beta}$, the following is true
\[
    \norm{\xx - \eta \nabla h(\xx) - \yy + \eta \nabla h(\yy)}^2 \leq (1 - \mu\eta)\norm{\xx - \yy}^2\,.
\]
\end{lemma}
\begin{proof}
    \begin{align*}
\norm{\xx - \eta \nabla h(\xx) - \yy + \eta \nabla h(\yy)}^2 &= \norm{\xx - \yy}^2 + \eta^2 \norm{\nabla h(\xx) - \nabla h(\yy)}^2 - 2\eta\inp*{\nabla h(\xx) - \nabla h(\yy)}{\xx - \yy}\\
&\refLE{asm:smoothness} \norm{\xx - \yy}^2 + (\eta^2\beta - 2\eta) \inp*{\nabla h(\xx) - \nabla h(\yy)}{\xx - \yy}\,.
    \end{align*}
    Recall our bound on the step-size $\eta \leq \frac{1}{\beta}$ which implies that $(\eta^2\beta - 2\eta) \leq -\eta$. Finally, apply the $\mu$-strong convexity of $h$ to get
    \[
-\eta\inp*{\nabla h(\xx) - \nabla h(\yy)}{\xx - \yy} \leq -\eta \mu\norm{\xx - \yy}^2\,.
    \]
\end{proof}


\newpage
\section{Convergence of \fedavg}

\begin{algorithm}[h]
    \caption{\fedavg: Federated Averaging}\label{alg:fedavg}
    \begin{algorithmic}[1]
      \STATE \textbf{server input:} initial $\xx$, and global step-size $\eta_g$
          \STATE \textbf{\color{client} client $i$'s input:} local step-size $\eta_l$
            \FOR{each round $r=1,\dots, R$}
            \STATE sample clients $\cS \subseteq \{1,\dots,N\}$
            \STATE \textbf{communicate} $\xx$ to all clients $i \in \cS$
            \ONCLIENT{\tikzmk{A} $i \in \cS$}
              \STATE initialize local model $\yy_{i} \leftarrow \xx$
              \FOR{$k = 1,\dots,K $}
        \STATE compute mini-batch gradient $g_i(\yy_{i})$
        \STATE $\yy_{i} \leftarrow \yy_{i} - \eta_l g_i(\yy_{i})$ 
        \ENDFOR
        \STATE \textbf{communicate} $\Delta\yy_i \leftarrow \yy_{i} - \xx$
            \ENDON
            \tikzmk{B}\blockfed{client}
            \STATE $\Delta\xx \leftarrow \tfrac{1}{\abs{\cS}}\sum_{i \in \cS}\Delta\yy_i$ 
            \STATE $\xx \leftarrow \xx + \eta_g\Delta\xx$
            \ENDFOR
      \end{algorithmic}
    \end{algorithm}

We outline the \fedavg\ method in Algorithm \ref{alg:fedavg}. In round $r$ we sample $\cS^r \subseteq [N]$ clients with $\abs{\cS^r} = S$ and then perform the following updates:
\begin{itemize}
    \item Starting from the shared global parameters $\yy^r_{i,0} =
      \xx^{r-1}$, we update the local parameters for $k \in [K]$
    \begin{equation}\label{eqn:fed-local-updates-sample}
\yy^{r}_{i,k} = \yy^{r}_{i,k-1} - \eta_l g_i(\yy_{i,k-1}^{r})\,.
    \end{equation}
    \item Compute the new global parameters
      using only updates from the clients $i \in \cS^r$ and a global step-size $\eta_g$:
    \begin{equation}\label{eqn:fed-global-updates-sample}
\xx^r = \xx^{r-1} + \frac{\eta_g }{S}\sum_{i\in\cS^r} (\yy^{r}_{i,K} -
\xx^{r-1})\,.
    \end{equation}
\end{itemize}
Finally, for some weights $\{w_r\}$, we output
\begin{equation}\label{eqn:fedavg-output}
    \bar \xx^R = \xx^{r-1} \text{ with probability } \frac{w_r}{\sum_{\tau}w_{\tau}} \text{ for } r\in\{1,\dots,R+1\} \,.
\end{equation}

\subsection{Bounding heterogeneity}
Recall our bound on the gradient dissimilarity:
\begin{equation}\label{eqn:heterogeneity}
          \frac{1}{N}\sum_{i=1}^N \norm{\nabla f_i(\xx)}^2 \leq G^2 + B^2 \norm{\nabla f(\xx)}^2 \,.
      \end{equation}
      If $\{f_i\}$ are convex, we can relax the assumption to
      \begin{equation}\label{eqn:heterogeneity-convex}
          \frac{1}{N}\sum_{i=1}^N \norm{\nabla f_i(\xx)}^2 \leq G^2 + 2\beta  B^2(f(\xx) - f^\star) \,.
      \end{equation}  
We defined two variants of the bounds on the heterogeneity depending of whether the functions are convex or not. Suppose that the functions $f$ is indeed {convex} as in \eqref{asm:strong-convexity} and {$\beta$-smooth} as in \eqref{asm:smoothness}, then it is straightforward to see that \eqref{eqn:heterogeneity} implies \eqref{eqn:heterogeneity-convex}. Thus for convex functions, \eqref{asm:heterogeneity} is mildly weaker.
Suppose that the functions $\{f_1,\dots,f_N\}$ are convex and $\beta$-smooth. Then \eqref{eqn:heterogeneity-convex} is satisfied with $B^2 = 2$ since
\begin{align*}
    \frac{1}{N}\sum_{i=1}^N \norm{\nabla f_i(\xx)}^2 &\leq \frac{2}{N}\sum_{i=1}^N \norm{\nabla f_i(\xx^\star)}^2 + \frac{2}{N}\sum_{i=1}^N \norm{\nabla f_i(\xx) - \nabla f_i(\xx^\star)}^2\\
    & \refLE{eqn:smoothness} \underbrace{\frac{2}{N}\sum_{i=1}^N \norm{\nabla f_i(\xx^\star)}^2}_{=:\ \sigma_f^2} + 4\beta (f(\xx) - f^\star)\,.
\end{align*}
Thus, $(G,B)$-BGD \eqref{eqn:heterogeneity-convex} is equivalent to the heterogeneity assumption of \citep{mishchenko2019distributed} with $G^2 = \sigma_f^2$. Instead, if we have the stronger assumption \eqref{asm:heterogeneity} but the functions are possibly non-convex, then $G = \e$ corresponds to the \textbf{local dissimilarity} defined in \citep{li2018fedprox}. Note that assuming $G$ is negligible is quite strong and corresponds to the strong-growth condition \citep{vaswani2019fast}.

\subsection{Rates of convergence (Theorem~\ref{thm:fedavg})}\label{appsec:fedavg}
We first restate Theorem~\ref{thm:fedavg} with some additional details and then see its proof.
\begin{theorem}\label{thm:fedavg-full}
Suppose that the functions $\{f_i\}$ satisfies assumptions
\ref{asm:variance}, \ref{asm:smoothness}, and \ref{asm:heterogeneity}. Then, in each of the following cases, there exist weights
$\{w_r\}$ and local step-sizes $\eta_l$
such that for any $\eta_g \geq 1$ the output of \fedavg\ \eqref{eqn:fedavg-output} $\bar\xx^R$ satisfies
    \begin{itemize}
        \item \textbf{Strongly convex:} $f_i$ satisfies \eqref{asm:strong-convexity} for $\mu >0$, $\eta_l \leq \frac{1}{8(1 + B^2)\beta K \eta_g}$,  $R \geq\frac{8(1 + B^2)\beta}{\mu}$ then
        \[
\expect[f(\bar\xx^R)] - f(\xx^\star) \leq
\tilde\cO\rbr*{\frac{M^2}{\mu RKS}  + \frac{\beta G^2}{\mu^2 R^2} + \mu D^2 \exp(-\tfrac{\mu}{16(1+ B^2)\beta} R)}\,,
    \]
    \item \textbf{General convex:} $f_i$ satisfies \eqref{asm:strong-convexity} for $\mu = 0$, $\eta_l \leq \frac{1}{(1 + B^2)8\beta K \eta_g}$, $R \geq 1$ then
    \[
\expect[f(\bar\xx^R)] - f(\xx^\star) \leq
\cO\rbr*{\frac{M D}{\sqrt{RKS}} + \frac{D^{4/3}(\beta G^2)^{1/3}}{(R+1)^{2/3}}+ \frac{B^2\beta D^2)}{R}}\,,
    \]
    \item \textbf{Non-convex:} $f_i$ satisfies \eqref{asm:heterogeneity} and $\eta_l \leq \frac{1}{(1 + B^2)8\beta K \eta_g}$, then
    \[
\expect[\norm{\nabla f(\bar\xx^R)}^2] \leq
\cO\rbr*{\frac{\beta M \sqrt{F}}{\sqrt{RKS}} + \frac{F^{2/3}(\beta G^2)^{1/3}}{(R+1)^{2/3}}+ \frac{B^2\beta F}{R}}\,,
    \]
    \end{itemize}
    where $M^2 := \sigma^2\rbr{1 + \tfrac{S}{\eta_g^2}} + K(1 - \tfrac{S}{N}) G^2$, $D := \norm{\xx^0 - \xx^\star}$, and $F := f(\xx^0) - f^\star$.
\end{theorem}

\subsection{Proof of convergence}
We will only prove the rate of convergence for convex functions here. The corresponding rates for non-convex functions are easy to derive following the techniques in the rest of the paper.

\begin{lemma}\textbf{\em (one round progress)}\label{lem:fedavg-progress}
Suppose our functions satisfies assumptions \eqref{asm:heterogeneity} and
\eqref{asm:strong-convexity}--\eqref{asm:smoothness}. For any step-size satisfying
$\eta_l \leq \frac{1}{(1 + B^2)8\beta K\eta_g}$ and effective step-size
$\tilde\eta := K \eta_g \eta_l$, the updates of \fedavg\ satisfy
\[
    \E \norm{\xx^{r} - \xx^\star}^2 \leq (1 - \tfrac{\mu
      \tilde\eta}{2}) \E \norm{\xx^{r-1} - \xx^\star}^2 +
    (\tfrac{1}{KS})\tilde\eta^2 \sigma^2 + (1 - \tfrac{S}{N})\tfrac{4\tilde\eta^2}{S} G^2 -\tilde\eta(\E[f(\xx^{r-1})] -
    f(\xx^\star)) + 3\beta\tilde\eta\error_{r}\,,
\]
where $\error_{r}$ is the drift caused by the local updates on the
clients defined to be
\[
    \error_{r} := \frac{1}{KN}\sum_{k=1}^{K}\sum_{i=1}^N
    \expect_r[\norm*{\yy_{i,k}^r - \xx^{r-1}}^2]\,.
\]
\end{lemma}
\begin{proof}
We start with the observation that the updates \eqref{eqn:fed-local-updates-sample} and \eqref{eqn:fed-global-updates-sample} imply that the server update in round $r$ can be written as below (dropping the superscripts everywhere)
\[
    \Delta \xx = - \frac{\tilde\eta}{K S} \sum_{k, i \in \cS} g_i(\yy_{i,k-1}) \text{ and } \E[\Delta \xx] = - \frac{\tilde\eta}{KN} \sum_{k, i} \E[\nabla f_i(\yy_{i,k-1})]\,.
\]
We adopt the convention that summations are always over $k \in [K]$ or $i \in [N]$ unless otherwise stated. Expanding using above observing, we proceed as\footnote{We use the notation $\E_{r-1}[\,\cdot\,]$ to mean conditioned on filtration $r$ i.e. on all the randomness generated prior to round $r$.}
\begin{align*}
    \E_{r-1}\norm{\xx + \Delta \xx - \xx^\star}^2 &= \norm{\xx - \xx^\star}^2 -\frac{2\tilde\eta}{KN}\sum_{k,i}\inp{\nabla f_i(\yy_{i,k-1})}{\xx - \xx^\star} + \tilde\eta^2\E_{r-1}\norm[\bigg]{\frac{1}{K S} \sum_{k, i \in \cS} g_i(\yy_{i,k-1})}^2\\
    &\stackrel{\text{Lem.~}\ref{lem:independent}}{\leq}\norm{\xx - \xx^\star}^2 \underbrace{-\frac{2\tilde\eta}{KN}\sum_{k,i}\inp{\nabla f_i(\yy_{i,k-1})}{\xx - \xx^\star}}_{\cA_1} \\&\hspace{2cm}+ \underbrace{\tilde\eta^2\E_{r-1}\norm[\bigg]{\frac{1}{K S} \sum_{k, i \in \cS} \nabla f_i(\yy_{i,k-1})}^2}_{\cA_2} + \frac{\tilde\eta^2 \sigma^2}{KS}\,.
\end{align*}
We can directly apply Lemma~\ref{lem:magic} with $h = f_i$, $\xx = \yy_{i,k-1}$, $\yy = \xx^\star$, and $\zz = \xx$ to the first term $\cA_1$
    \begin{align*}
\cA_1 &= \frac{2\tilde\eta}{KN}\sum_{k,i} \inp*{ \nabla f_i(\yy_{i,k-1})}{\xx^\star - \xx} \\
&\leq \frac{2\tilde\eta}{KN} \sum_{k,i} \rbr*{f_i(\xx^\star) - f_i(\xx) + \beta \norm{\yy_{i,k-1} - \xx}^2 - \frac{\mu}{4}\norm{\xx - \xx^\star}^2}\\
&= -2\tilde\eta\rbr*{f(\xx) - f(\xx^\star) + \frac{\mu}{4}\norm{\xx - \xx^\star}^2} + 2\beta \tilde\eta\error\,.
    \end{align*}
    For the second term $\cA_2$, we repeatedly apply the relaxed triangle inequality (Lemma \ref{lem:independent})
    \begin{align*}
        \cA_2 &= \tilde\eta^2\E_{r-1}\norm[\bigg]{\frac{1}{K S} \sum_{k, i \in \cS} \nabla f_i(\yy_{i,k-1}) - \nabla f_i(\xx) + \nabla f_i(\xx)}^2\\
        &\leq 2\tilde\eta^2\E_{r-1}\norm[\bigg]{\frac{1}{K S} \sum_{k, i \in \cS} \nabla f_i(\yy_{i,k-1}) - \nabla f_i(\xx)}^2 + 2\tilde\eta^2\E_{r-1}\norm[\bigg]{\frac{1}{S} \sum_{i \in \cS} \nabla f_i(\xx)}^2\\
        &\leq \frac{2\tilde\eta^2}{KN} \sum_{i,k} \E_{r-1}\norm{ \nabla f_i(\yy_{i,k-1}) - \nabla f_i(\xx)}^2 + 2\tilde\eta^2\E_{r-1}\norm[\bigg]{\frac{1}{S} \sum_{i \in \cS} \nabla f_i(\xx) - \nabla f(\xx) + \nabla f(\xx)}^2\\
        &\leq \frac{2\tilde\eta^2\beta^2}{KN} \sum_{i,k} \E_{r-1}\norm{ \yy_{i,k-1} - \xx}^2 + 2\tilde\eta^2\norm{\nabla f(\xx)}^2 + (1 - \tfrac{S}{N})4\tilde\eta^2 \frac{1}{S N}\sum_{i} \norm{\nabla f_i(\xx)}^2\\
        &\leq 2\tilde\eta^2\beta^2\error + 8\tilde\eta^2\beta (B^2 +1) (f(\xx) - f(\xx^\star)) + (1 - \tfrac{S}{N})\tfrac{4\tilde\eta^2}{S} G^2
    \end{align*}
    The last step used Assumption $(G,B)$-BGD assumption \eqref{eqn:heterogeneity-convex} that $\frac{1}{N}\sum_{i=1}^N \norm{\nabla f_i(\xx)}^2 \leq G^2 + 2\beta  B^2(f(\xx) - f^\star)$. The extra $(1 - \frac{S}{N}$ improvement we get is due to sampling the functions $\{f_i\}$ \emph{without} replacement. Plugging back the bounds on $\cA_1$ and $\cA_2$,
    \begin{multline*}
        \E_{r-1}\norm{\xx + \Delta \xx - \xx^\star}^2 \leq (1 - \tfrac{\mu\tilde\eta}{2})\norm{\xx - \xx^\star}^2 - (2\tilde\eta - 8\beta\tilde\eta^2(B^2 + 1))(f(\xx) - f(\xx^\star))  \\
       +(1 + \tilde\eta\beta)2\beta \tilde\eta\error + \tfrac{1}{KS}\tilde\eta^2 \sigma^2 + (1 - \tfrac{S}{N})\tfrac{4\tilde\eta^2}{S} G^2\,.
    \end{multline*}
    The lemma now follows by observing that $8\beta\tilde\eta(B^2 + 1) \leq 1$ and that $B \geq 0$.
\end{proof}

\begin{lemma}\textbf{\em (bounded drift)}\label{lem:fedavg-error-bound}
Suppose our functions satisfies assumptions \eqref{asm:heterogeneity} and
\eqref{asm:strong-convexity}--\eqref{asm:smoothness}. Then the updates of \fedavg\
for any step-size satisfying $\eta_l \leq \frac{1}{(1 + B^2)8\beta K\eta_g}$ have bounded drift:
\[
    3\beta\tilde\eta \error_{r} \leq \tfrac{2\tilde\eta}{3}(\E[f(\xx^{r-1})] ) - f(\xx^\star) + \frac{\tilde\eta^2\sigma^2}{2K\eta_g^2} +18 \beta  \tilde\eta^3 G^2\,.
\]
\end{lemma}
\begin{proof}
If $K=1$, the lemma trivially holds since $\yy_{i,0} = \xx$ for all $i \in [N]$ and $\error_r =0$. Assume $K \geq 2$ here on. Recall that the local update made on client $i$ is $\yy_{i,k} = \yy_{i,k-1} - \eta_lg_i(\yy_{i,k-1})$. Then,
\begin{align*}
    \E \norm{\yy_{i,k} - \xx}^2 &= \E \norm{\yy_{i,k-1} - \xx - \eta_lg_i(\yy_{i,k-1})}^2\\
    &\leq \E \norm{\yy_{i,k-1} - \xx - \eta_l\nabla f_i(\yy_{i,k-1})}^2 + \eta_l^2\sigma^2\\
    &\leq (1 + \tfrac{1}{K-1})\E \norm{\yy_{i,k-1} - \xx}^2 + K\eta_l^2\norm{\nabla f_i(\yy_{i,k-1})}^2 + \eta_l^2\sigma^2\\
    &= (1 + \tfrac{1}{K-1})\E \norm{\yy_{i,k-1} - \xx}^2 + \frac{\tilde\eta^2}{\eta_g K}\norm{\nabla f_i(\yy_{i,k-1})}^2 + \frac{\tilde\eta^2\sigma^2}{K^2\eta_g^2}\\
    &\leq (1 + \tfrac{1}{K-1})\E \norm{\yy_{i,k-1} - \xx}^2 + \frac{2\tilde\eta^2}{\eta_g K}\norm{\nabla f_i(\yy_{i,k-1}) - \nabla f_i(\xx)}^2 + \frac{2\tilde\eta^2}{\eta_g K}\norm{\nabla f_i(\xx)}^2 +  \frac{\tilde\eta^2\sigma^2}{K^2\eta_g^2}\\
    &\leq (1 + \tfrac{1}{K-1} + \tfrac{2\tilde\eta^2\beta^2}{\eta_g K})\E \norm{\yy_{i,k-1} - \xx}^2 + \frac{2\tilde\eta^2}{\eta_g K}\norm{\nabla f_i(\xx)}^2 +  \frac{\tilde\eta^2\sigma^2}{K^2\eta_g^2}\\
    &\leq (1 + \tfrac{2}{(K-1)})\E \norm{\yy_{i,k-1} - \xx}^2 + \frac{2\tilde\eta^2}{\eta_g K}\norm{\nabla f_i(\xx)}^2 +  \frac{\tilde\eta^2\sigma^2}{K^2\eta_g^2}\,.
\end{align*}
In the above proof we separated the mean and the variance in the first inequality, then used the relaxed triangle inequality with $a = \frac{1}{K-1}$ in the next inequality. Next equality uses the definition of $\tilde\eta$, and the rest follow from the Lipschitzness of the gradient. Unrolling the recursion above,
\begin{align*}
    \E \norm{\yy_{i,k} - \xx}^2 &\leq  \sum_{\tau = 1}^{k-1} (\frac{2\tilde\eta^2}{\eta_g K}\norm{\nabla f_i(\xx)}^2 +  \frac{\tilde\eta^2\sigma^2}{K^2\eta_g^2})( 1 + \tfrac{2}{(K-1)})^\tau\\
    &\leq (\frac{2\tilde\eta^2}{\eta_g K}\norm{\nabla f_i(\xx)}^2 +  \frac{\tilde\eta^2\sigma^2}{K^2\eta_g^2})3K\,.
\end{align*}
Averaging over $i$ and $k$, multiplying by $3\beta \tilde\eta$ and then using Assumption \ref{asm:heterogeneity},
\begin{align*}
    3\beta \tilde\eta \error_r &\leq \frac{1}{N}\sum_{i} 18 \beta  \tilde\eta^3 \norm{\nabla f_i(\xx)}^2 +  \frac{3\beta \tilde\eta^3\sigma^2}{K\eta_g^2}\\
    &\leq 18 \beta  \tilde\eta^3 G^2+  \frac{3\beta \tilde\eta^3\sigma^2}{K\eta_g^2} + 36 \beta^2  \tilde\eta^3B^2 (f(\xx) - f(\xx^\star))
\end{align*}
The lemma now follows from our assumption that $8 (B^2 + 1) \beta \tilde\eta \leq 1$.
\end{proof}

\paragraph{Proof of Theorems \ref{thm:fedavg},~\ref{thm:fedavg-full}} Adding the statements of Lemmas \ref{lem:fedavg-progress} and \ref{lem:fedavg-error-bound}, we get
\begin{align*}
    \E \norm{\xx + \Delta \xx - \xx^\star}^2 &\leq (1 - \tfrac{\mu
      \tilde\eta}{2}) \E \norm{\xx - \xx^\star}^2 +
    (\tfrac{1}{KS})\tilde\eta^2 \sigma^2 + (1 - \tfrac{S}{N})\tfrac{4\tilde\eta^2}{S} G^2 -\tilde\eta(\E[f(\xx)] -
    f(\xx^\star)) \\&\hspace{2cm} +  \tfrac{2\tilde\eta}{3}(\E[f(\xx)] ) - f(\xx^\star) + \frac{\tilde\eta^2\sigma^2}{2K\eta_g^2} +18 \beta  \tilde\eta^3 G^2 \\
    &= (1 - \tfrac{\mu
      \tilde\eta}{2}) \E \norm{\xx - \xx^\star}^2 - \tfrac{\tilde\eta}{3}(\E[f(\xx)] - f(\xx^\star)) + \tilde\eta^2 \rbr[\bigg]{\tfrac{\sigma^2}{KS}(1 + \tfrac{S}{\eta_g^2}) + \tfrac{4G^2}{S}(1 - \tfrac{S}{N}) + 18 \beta \tilde\eta G^2}\,.
\end{align*}
Moving the $(f(\xx) - f(\xx^\star))$ term and dividing throughout by $\frac{\tilde\eta}{3}$, we get the following bound for any $\tilde\eta \leq \frac{1}{8(1 + B^2)\beta}$
\[
    \E[f(\xx^{r-1})] - f(\xx^\star) \leq  \tfrac{3}{\tilde\eta}(1 - \tfrac{\mu\tilde\eta}{2})\norm{\xx^{r-1} - \xx^\star}^2  - \tfrac{3}{\tilde\eta}\norm{\xx^{r} - \xx^\star}^2 + 3\tilde\eta \rbr[\bigg]{\tfrac{\sigma^2}{KS}(1 + \tfrac{S}{\eta_g^2}) + \tfrac{4G^2}{S}(1 - \tfrac{S}{N}) + 18 \beta \tilde\eta G^2}\,.
\]
If $\mu = 0$ (general convex), we can directly apply Lemma~\ref{lemma:general}. Otherwise, by averaging using weights $w_r = (1  - \tfrac{\mu \tilde\eta}{2})^{1-r}$ and using the same weights to pick output $\bar\xx^R$, we can simplify the above recursive bound (see proof of Lem.~\ref{lem:constant}) to prove that for any $ \tilde\eta$ satisfying $\frac{1}{\mu R} \leq \tilde\eta \leq \frac{1}{8(1 + B^2)\beta}$
\[
    \E[f(\bar\xx^{R})] - f(\xx^\star) \leq \underbrace{3 \norm{\xx^0 - \xx^\star}^2}_{=:d}\mu\exp(-\tfrac{\tilde\eta}{2}\mu R) + \tilde\eta \rbr[\big]{\underbrace{\tfrac{2\sigma^2}{KS}(1 + \tfrac{S}{\eta_g^2}) + \tfrac{8G^2}{S}(1 - \tfrac{S}{N})}_{=: c_1}} +  \tilde\eta^2 (\underbrace{36 \beta  G^2}_{=: c_2})
\]
Now, the choice of $\tilde \eta =\min\cbr*{ \frac{\log(\max(1,\mu^2 R d/c_1))}{\mu R}, \frac{1}{(1 + B^2)8\beta}}$ yields the desired rate.The proof of the non-convex case is very similar and also relies on Lemma~\ref{lemma:general}.


\subsection{Lower bound for \fedavg\ (Theorem~\ref{thm:lowerbound})}\label{appsec:lowerbound}
We first formalize the class of algorithms we look at before proving out lower bound.

\begin{enumerate}
    \myitem{A6} \label{asm:algorithm} We assume that \fedavg\ is run with $\eta_g = 1$, $K > 1$, and arbitrary possibly adaptive positive step-sizes $\{\eta_1,\dots,\eta_R\}$ are used with $\eta_r \leq \frac{1}{\mu}$ and fixed within a round for all clients. Further, the server update is a convex combination of the client updates with non-adaptive weights.
\end{enumerate}
Note that we only prove the lower bound here for $\eta_g = 1$. In fact, by taking $\eta_g$ infinitely large and scaling $\eta_l \propto \frac{1}{K\eta_g}$ such that the effective step size $\tilde\eta = \eta_l \eta_g K$ remains constant, \fedavg\ reduces to the simple large batch SGD method. Hence, proving a lower bound for arbitrary $\eta_g$ is not possible, but also is of questionable relevance. Further, note that when $\sigma^2 = 0$, the upper bound in Theorem~\ref{thm:fedavg-full} uses $\eta_g = 1$ and hence the lower bound serves to show that our analysis is tight.

Below we state a more formal version of Theorem~\ref{thm:lowerbound}.
\begin{theorem}
 For any positive constants $G$, $\mu$, there exist $\mu$-strongly convex functions satisfying \ref{asm:heterogeneity} for which that the output of \fedavg\ satisfying \ref{asm:algorithm} has the error for any $r \geq 1$:
 \[
    f(\xx^r) - f(\xx^\star) \geq \Omega\rbr[\big]{ \min\rbr[\big]{f(\xx^0) - f(\xx^\star) , \frac{G^2}{\mu R^2}}}\,.
 \]
\end{theorem}
\begin{proof}
Consider the following simple one-dimensional functions for any given $\mu$ and $G$:
\[
    f_1(x) := \mu x^2 + Gx,\text{ and } f_2(x) := - Gx\,,
\]
with $f(x) = \tfrac{1}{2}(f_1(x) + f_2(x)) = \frac{\mu}{2}x^2$ and optimum at x = 0.
Clearly $f$ is $\mu$-strongly convex and further $f_1$ and $f_2$ satisfy \ref{asm:heterogeneity} with B = 3. Note that we chose $f_2$ to be a linear function (not strongly convex) to simplify computations. The calculations made here can be extended with slightly more work for $(\tilde f_2 = \frac{\mu}{2}x^2 - Gx)$ (e.g. see Theorem 1 of \citep{safran2019good}).

Let us start \fedavg\ from $x^0 > 0$. A single local update for $f_1$ and $f_2$ in round $r \geq 1$ is respectively
\[
    y_1 = y_1 - \eta_r(2\mu x + G) \text{ and } y_2 = y_2 + \eta_r G\,.
\]
Then, straightforward computations show that the update at the end of round $r$ is of the following form for some averaging weight $\alpha \in [0,1]$
\[
    x^r = x^{r-1}((1 - \alpha)(1 - 2 \mu \eta_r)^K + \alpha) + \eta_r G\sum_{\tau = 0}^{K-1}(\alpha - (1-\alpha)(1 - 2 \mu \eta_r)^\tau)\,.
\]
Since $\alpha$ was picked obliviously, we can assume that $\alpha \leq 0.5$. If indeed $\alpha > 0.5$, we can swap the definitions of $f_1$ and $f_2$ and the sign of $x^0$. With this, we can simplify as
\begin{align*}
    x^r &\geq x^{r-1} \frac{(1 - 2 \mu \eta_r)^K + 1}{2} + \frac{\eta_r G}{2}\sum_{\tau = 0}^{K-1}(1 - (1 - 2 \mu \eta_r)^\tau)\\
    &\geq x^{r-1} (1 - 2 \mu \eta_r)^K + \frac{\eta_r G}{2}\sum_{\tau = 0}^{K-1}(1 - (1 - 2 \mu \eta_r)^\tau)\,.
\end{align*}
Observe that in the above expression, the right hand side is increasing with $\eta_r$---this represents the effect of the client drift and increases the error as the step-size increases. The left hand side decreases with $\eta_r$---this is the usual convergence observed due to taking gradient steps. The rest of the proof is to show that even with a careful balancing of the two terms, the effect of $G$ cannot be removed. Lemma \ref{lem:tedious-computations} performs exactly such a computation to prove that for any $r \geq 1$,
\[
  x^r \geq c \min(x_0, \frac{G}{\mu R})\,.
\]
We finish the proof by noting that $f(x^r) = \frac{\mu}{2}(x^r)^2$.
\end{proof}

\begin{lemma}\label{lem:tedious-computations}
  Suppose that for all $r \geq 1$, $\eta_r \leq \frac{1}{\mu}$ and the following is true:
  \[
    x^r \geq x^{r-1} (1 - 2 \mu \eta_r)^K + \frac{\eta_r G}{2}\sum_{\tau = 0}^{K-1}(1 - (1 - 2 \mu \eta_r)^\tau)\,.
    \]
    Then, there exists a constant $c > 0$ such that for any sequence of step-sizes $\{\eta^r\}$:
  \[
      x^r \geq c \min(x_0, \frac{G}{\mu R})
  \]
\end{lemma}

\begin{proof}
Define $\gamma_r = \mu \eta_r R(K-1)$. Such a $\gamma_r$ exists and is positive since $K \geq 2$. Then, $\gamma_r$ satisfies
\[
  (1 - 2 \mu \eta_r)^{\frac{K -1}{2}} = (1 - \frac{2\gamma_r}{R(K-1)})^{\frac{K-1}{2}} \leq \exp(-\gamma_r/R)\,.
\]
We then have
\begin{align*}
  x^r &\geq x^{r-1} (1 - 2 \mu \eta_r)^K + \frac{\eta_r G}{2}\sum_{\tau = 0}^{K-1}(1 - (1 - 2 \mu \eta_r)^\tau)\\
  &\geq x^{r-1} (1 - 2 \mu \eta_r)^K + \frac{\eta_r G}{2}\sum_{\tau = (K-1)/2}^{K-1}(1 - (1 - 2 \mu \eta_r)^\tau)\\
  &\geq x^{r-1} (1 - 2 \mu \eta_r)^K + \frac{\gamma_r G}{4\mu}(1 - \exp(-\gamma_r/R))\,.
\end{align*}
The second inequality follows because $\eta_r \leq \frac{1}{\mu}$ implies that $(1 - (1 - 2 \mu \eta_r)^\tau)$ is always positive. If $\gamma_r \geq R/8$, then we have a constant $c_1 \in (0,1/32)$ which satisfies
\begin{equation}\label{eqn:lowerb-case1}
x^r \geq \frac{c_1G}{\mu}\,.
\end{equation}
On the other hand, if $\gamma_r < R/8$, we have a tighter inequality
\begin{align*}
  (1 - 2 \mu \eta_r)^{\frac{K -1}{2}} = (1 - \frac{2\gamma_r}{R(K-1)})^{\frac{K-1}{2}} \leq 1 - \frac{\gamma_r}{R}\,,
\end{align*}
implying that
\begin{align}
  x^r &\geq x^{r-1} \rbr[\bigg]{1 - \frac{2 \gamma_r}{R(K-1)}}^K + \frac{\gamma_r^2 G}{4R\mu}\nonumber\\
  &\geq x^{r-1}(1 - \frac{4\gamma_r}{R}) + \frac{\gamma_r^2 G}{4\mu R}\,.\label{eqn:lowerb-case2}
\end{align}
The last step used Bernoulli's inequality and the fact that $K-1 \leq K/2$ for $K \geq 2$. Observe that in the above expression, the right hand side is increasing with $\gamma_r$---this represents the effect of the client drift and increases the error as the step-size increases. The left hand side decreases with $\gamma_r$---this is the usual convergence observed due to taking gradient steps. The rest of the proof is to show that even with a careful balancing of the two terms, the effect of $G$ cannot be removed.

Suppose that all rounds after $r_0 \geq 0$ have a small step-size i.e. $\gamma_{r} \leq R/8$ for all $r > r_0$ and hence satisfies \eqref{eqn:lowerb-case2}. Then we will prove via induction that
\begin{equation}
x^r \geq \min(c_r x^{r_0}, \frac{G}{256 \mu R})  \text{, for constants } c_r := (1 - \tfrac{1}{2R})^{r - r_0}\,. \label{eqn:lowerb-induction}
\end{equation}
For $r = r_0$, \eqref{eqn:lowerb-induction} is trivially satisfied. Now for $r > r_0$,
\begin{align*}
  x^r &\geq x^{r-1}(1 - \frac{4\gamma_r}{R}) + \frac{\gamma_r^2 G}{4\mu R}\\
  &\geq  \min \rbr*{ x^{r-1}(1 - \tfrac{1}{2R}) \,,\, \frac{G}{256\mu R}}\\
  &= \min \rbr*{ c_r x^{r_0} \,,\, \frac{G}{256\mu R}} \,.
\end{align*}
The first step is because of \eqref{eqn:lowerb-case2} and the last step uses the induction hypothesis. The second step considers two cases for $\gamma_r$: either $\gamma_r \leq \frac{1}{8}$
and $(1 - \tfrac{1}{2R}) \geq (1 - \tfrac{1}{2R})$, or $\gamma_r^2 \geq \frac{1}{64}$. Finally note that $c^r \geq \frac{1}{2}$ using Bernoulli's inequality. We have hence proved
\[
  x^R \geq \min \rbr*{\frac{1}{2}x^{r_0} \,,\, \frac{G}{256\mu R}}
\]
Now suppose $\gamma_{r_0} > R/8$. Then \eqref{eqn:lowerb-case1} implies that $x^R \geq \frac{c G}{\mu R}$ for some constant $c > 0$. If instead no such $r_0 \geq 1$ exists, then we can set $r_0 = 0$. Now finally observe that the previous proof did not make any assumption on $R$, and in fact the inequality stated above holds for all $r \geq 1$.
\end{proof}


\newpage
\section{Convergence of \covfefe}\label{appsec:covfefe}
We first restate the convergence theorem more formally, then prove the result for the convex case, and then for non-convex case. Throughout the proof, we will focus on the harder option II. The proofs for \covfefe\ with option I are nearly identical and so we skip them.

\begin{theorem}\label{thm:convergence-full}
    Suppose that the functions $\{f_i\}$ satisfies assumptions
\ref{asm:variance} and \ref{asm:smoothness}. Then, in each of the following cases, there exist weights
    $\{w_r\}$ and local step-sizes $\eta_l$
    such that for any $\eta_g \geq 1$ the output \eqref{eqn:output} of \covfefe\ satisfies:
        \begin{itemize}
            \item \textbf{Strongly convex:} $f_i$ satisfies \eqref{asm:strong-convexity} for $\mu >0$, $\eta_l \leq \min\rbr*{\tfrac{1}{81
        \beta K\eta_g}, \tfrac{S}{15 \mu N K\eta_g}}$,  $R \geq \max(\frac{162\beta}{\mu}, \frac{30N}{S})$ then
            \[
    \expect[f(\bar\xx^R)] - f(\xx^\star) \leq
    \tilde\cO\rbr*{\frac{\sigma^2}{\mu RKS} \rbr{1 +
        \tfrac{S}{\eta_g^2}} + \frac{N\mu}{S} \tilde D^2 \exp\rbr*{-\min\cbr*{\frac{S}{30N},
          \frac{\mu}{162\beta}}R}}\,.
        \]
        \item \textbf{General convex:} $f_i$ satisfies \eqref{asm:strong-convexity} for $\mu = 0$, $\eta_l \leq \tfrac{1}{81
        \beta K\eta_g}$, $R \geq 1$ then
        \[
    \expect[f(\bar\xx^R)] - f(\xx^\star) \leq
    \cO\rbr*{\frac{\sigma \tilde D}{\sqrt{RKS}}\rbr*{\sqrt{1 + \tfrac{S}{\eta_g^2}}} + \sqrt{\frac{N}{S}}\frac{\beta \tilde D^2}{R}}\,,
        \]
        \item \textbf{Non-convex:} $\eta_l \leq \frac{1}{24K\eta_g\beta}\rbr*{\frac{S}{N}}^{\frac{2}{3}}$, and $R \geq 1$, then
        \[
    \expect[\norm{\nabla f(\bar\xx^R)}^2] \leq
    \cO\rbr*{\frac{\sigma \sqrt{ F}}{\sqrt{RKS}}\rbr*{\sqrt{1 + \tfrac{N}{\eta_g^2}}} + \frac{\beta F}{R}\rbr*{\frac{N}{S}}^{\frac{2}{3}}}\,.
        \]
        \end{itemize}
        Here $\tilde D^2 := (\norm{\xx^0 - \xx^\star}^2+ \tfrac{1}{2 N \beta^2}\sum_{i=1}^N \norm{\cc_i^0 - \nabla f_i(\xx^\star)}^2)$ and $F := (f(\xx_0) - f(\xx^\star))$.
    \end{theorem}
    
    \begin{remark}
    Note that the $\tilde D^2$ defined above involves an additional term $\tfrac{1}{2 N \beta^2}\sum_{i=1}^N \norm{\cc_i^0 - \nabla f_i(\xx^\star)}^2$. This is standard in variance reduction methods \citep{johnson2013accelerating,defazio2014saga,hanzely2019one}. Theoretically, we will use a warm-start strategy to set $\cc_i^0$ and in the first $N/S$ rounds, we compute $\cc_i^0 = \gg_i(\xx^0)$ over a batch size of size $K$. Then, using smoothness of $f_i$, we can bound this additional term as 
    \[
        \tfrac{1}{2 N \beta^2}\sum_{i=1}^N \norm{\cc_i^0 - \nabla f_i(\xx^\star)}^2 \leq \frac{1}{\beta} (f(\xx^0) - f^\star) + \frac{\sigma^2}{K\beta^2} 
        \leq D^2 + \frac{\sigma^2}{K\beta^2} \,.
    \]
    Thus, the asymptotic rates of \covfefe\ for general convex functions only incurs an additive term of the order of $O(\sqrt{\frac{N}{S}} \frac{1}{R})$.
    For strongly convex functions, we only see the affects in the logarithmic terms.
\end{remark}

\begin{remark}
When $\sigma =0$ i.e. when clients compute full gradients, the communication complexity of SCAFFOLD is: i) for strongly convex case it is $\tilde O\rbr*{\frac{N}{S} + \frac{\beta}{\mu}}$, ii) for general convex functions it is $O\rbr*{\sqrt{\frac{N}{S}}\frac{\beta}{R}}$,
\footnote{
A previous version of the paper showed a worse dependence of $O\rbr*{{\frac{N}{S}}\frac{\beta}{R}}$ due to sub-optimal choice of step-size $\eta$.}
and iii) for non-convex functions it is $O\rbr*{{\frac{N}{S}}^{2/3} \frac{\beta}{R}}$. In comparison, the follow up work of FedDyn~\citep{acar2021federated} proves communication complexity matching ours in the convex and strongly convex settings, but a \emph{worse} $O\rbr*{{\frac{N}{S}}\frac{\beta}{R}}$ complexity in the non-convex settings (all when $\sigma=0$).
\end{remark}
    We will rewrite \covfefe\ using notation which is convenient for the proofs:
$\{\yy_i\}$ represent the client models, $\xx$ is the aggregate server
model, and $\cc_i$ and $\cc$ are the client and server control
variates.  For an equivalent description which is easier to implement,
we refer to Algorithm \ref{alg:sampling}.  The server maintains a
global control variate $\cc$ as before and each client maintains its
own control variate $\cc_i$. In round $r$, a subset of clients $\cS^r$
of size $S$ are sampled uniformly from $\{1,\dots, N\}$. Suppose that
\emph{every} client performs the following updates
\begin{itemize}
    \item Starting from the shared global parameters $\yy^0_{i,r} =
      \xx^{r-1}$, we update the local parameters for $k \in [K]$
    \begin{equation}\label{eqn:local-updates-sample}
\yy^{r}_{i,k} = \yy^{r}_{i,k-1} - \eta_l \vv^{r}_{i,k}\,, \quad \text{where}\quad \vv^{r}_{i,k} := g_i(\yy_{i,k-1}^{r}) - \cc_i^{r-1} + \cc^{r-1}
    \end{equation}
    \item Update the control iterates using (option II):
    \begin{equation}\label{eqn:control-updates-sample}
\tilde\cc^{r}_i = \cc^{r-1} - \cc_i^{r-1} + \tfrac{1}{K \eta_l}(\xx^{r-1} - \xx^r_{i,K}) = \frac{1}{K}\sum_{k=1}^K g_i(\yy_{i,k-1}^{r})\,.
    \end{equation} We update the local control variates only for clients $i \in \cS^r$
    \begin{equation}
\cc^{r}_i = \begin{cases}
\tilde\cc^{r}_i &\text{ if } i \in \cS^r\\
\cc^{r-1}_i &\text{ otherwise.}
\end{cases}
    \end{equation}
    \item Compute the new global parameters and global control variate
      using only updates from the clients $i \in \cS^r$:
    \begin{equation}\label{eqn:global-updates-sample}
\xx^r = \xx^{r-1} + \frac{\eta_g }{S}\sum_{i\in\cS^r} (\yy^{r}_{i,K} -
\xx^{r-1}) \ \text{ and } \ \cc^r = \frac{1}{N}\sum_{i=1}^N \cc^r_i =
\frac{1}{N}\rbr[\big]{\sum_{i\in\cS^r} \cc^{r}_{i} + \sum_{j
    \notin\cS^r} \cc^{r-1}_{j}}\,.
    \end{equation}
\end{itemize}
Finally, for some weights $\{w_r\}$, we output
\begin{equation}\label{eqn:output}
    \bar \xx^R = \xx^{r-1} \text{ with probability } \frac{w_r}{\sum_{\tau}w_{\tau}} \text{ for } r\in\{1,\dots,R+1\} \,.
\end{equation}
Note that the clients are agnostic to the sampling and their updates
are identical to when all clients are participating. Also note that
the control variate choice \eqref{eqn:control-updates-sample}
corresponds to (option II) of
Algorithm \ref{alg:sampling}. Further, the updates of the clients $i
\notin \cS^r$ is forgotten and is defined only to make the proofs
easier. While actually implementing the method, only clients $i \in
\cS^r$ participate and the rest remain inactive (see Algorithm
\ref{alg:sampling}).


\subsection{Convergence of \covfefe\ for convex functions (Theorem \ref{thm:convergence})}
We will first bound the variance of \covfefe\ update in Lemma~\ref{lem:server-variance-sample}, then see how sampling of clients effects our control variates in Lemma~\ref{lem:control-error-sample}, and finally bound the amount of client-drift in Lemma~\ref{lem:drift-bound-sampling}. We will then use these three lemmas to prove the progress in a single round in Lemma~\ref{lem:progress-sample}. Combining this progress with Lemmas~\ref{lem:constant} and \ref{lemma:general} gives us the desired rates.

\paragraph{Additional definitions.} Before proceeding with the proof of our lemmas, we need some additional definitions of the various errors we track. As before, we define the effective step-size to be 
\[
    \tilde\eta := K\eta_l \eta_g\,.
\]
We define client-drift to be how much the clients move from their starting point:
 \begin{equation}\label{eqn:client-drift-def}
    \error_r := \frac{1}{KN}\sum_{k=1}^K \sum_{i=1}^N \E[\norm{\yy_{i,k}^r - \xx^{r-1}}^2]\,.     
 \end{equation}
Because we are sampling the clients, not all the client control-variates get updated every round. This leads to some `lag' which we call control-lag:
\begin{equation}\label{eqn:control-lag-def}
    \ce_r := \frac{1}{N}\sum_{j=1}^N\E\norm{\E[\cc_i^r] - \nabla f_i(\xx^\star)}^2\,.
 \end{equation}

\paragraph{Variance of server update.} We study how the variance of the server update can be bounded.
\begin{lemma}\label{lem:server-variance-sample}
  For updates \eqref{eqn:local-updates-sample}---\eqref{eqn:global-updates-sample}, we can bound the variance of the server update in any round $r$ and any $\tilde\eta := \eta_l\eta_g K \geq 0$ as follows
  \[
    \E[\norm{\xx^r - \xx^{r-1}}^2] \leq 8\beta \tilde\eta^2 (\E[f(\xx^{r-1})] -f(\xx^\star)) + 8\tilde\eta^2 \ce_{r-1} + 4\tilde\eta^2\beta^2\error_r + \frac{12\tilde\eta^2\sigma^2}{KS}\,.
  \]
\end{lemma}
\begin{proof}
      The server update in round $r$ can be written as follows (dropping the superscript $r$ everywhere)
      \[
      \E\norm{\Delta \xx}^2 = \E\norm[\big]{-\frac{\tilde\eta}{KS}\sum_{k,i\in\cS}\vv_{i,k}}^2 = \E\norm[\big]{\frac{\tilde\eta}{KS}\sum_{k,i\in\cS} (g_i(\yy_{i,k-1}) + \cc - \cc_i)}^2\,,
      \]
      which can then be expanded as
      \begin{align*}
\E\norm{\Delta \xx}^2 &\leq \E\norm[\big]{\frac{\tilde\eta}{KS}\sum_{k,i\in\cS} (g_i(\yy_{i,k-1}) + \cc - \cc_i)}^2\\
&\leq 4\E\norm[\big]{\frac{\tilde\eta}{KS}\sum_{k,i\in\cS} g_i(\yy_{i,k-1}) - \nabla f_i(\xx)}^2 + 4\tilde\eta^2\E\norm{\cc}^2 + 4\E\norm[\big]{\frac{\tilde\eta}{KS}\sum_{k,i\in\cS}\nabla f_i(\xx^\star)- \cc_i}^2\\&\hspace{3cm} + 4\E\norm[\big]{\frac{\tilde\eta}{KS}\sum_{k,i\in\cS} \nabla f_i(\xx) - \nabla f_i(\xx^\star)}^2\\
&\refLE{eqn:smoothness} 4\E\norm[\big]{\frac{\tilde\eta}{KS}\sum_{k,i\in\cS} g_i(\yy_{i,k-1}) - \nabla f_i(\xx)}^2 + 4\tilde\eta^2\E\norm{\cc}^2 + 4\E\norm[\big]{\frac{\tilde\eta}{S}\sum_{i\in\cS}\nabla f_i(\xx^\star)- \cc_i}^2\\&\hspace{3cm} + 8\beta\tilde\eta^2(\E[f(\xx)] - f(\xx^\star))\\
&\leq  4\E\norm[\big]{\frac{\tilde\eta}{KS}\sum_{k,i\in\cS} \nabla f_i(\yy_{i,k-1}) - \nabla f_i(\xx)}^2 + 4\tilde\eta^2\norm{\E[\cc]}^2 + 4\norm[\big]{\frac{\tilde\eta}{S}\sum_{i\in\cS}\nabla f_i(\xx^\star)- \E[\cc_i]}^2\\&\hspace{3cm} + 8\beta\tilde\eta^2(\E[f(\xx)] - f(\xx^\star)) + \frac{12\tilde\eta^2\sigma^2}{KS}\,.
      \end{align*}
      The inequality before the last used the smoothness of $\{f_i\}$. The last inequality which separates the mean and the variance is an application of Lemma~\ref{lem:independent}: the variance of $(\frac{1}{KS}\sum_{k,i\in\cS} g_i(\yy_{i,k-1}))$ is bounded by $\sigma^2/KS$. Similarly, $\cc_j$ as defined in \eqref{eqn:control-updates-sample} for any $j \in [N]$ has variance smaller than $\sigma^2/K$ and hence the variance of $(\frac{1}{S}\sum_{i\in\cS}\cc_i)$ is smaller than $\sigma^2/KS$.

      Using Lemma~\ref{lem:norm-sum}.2  twice to simplify:
      \begin{align*}
   \E\norm{\Delta \xx}^2 &\leq \frac{4\tilde\eta^2}{KN}\sum_{k,i}  \E\norm[\big]{\nabla f_i(\yy_{i,k-1}) - \nabla f_i(\xx)}^2 + 4\tilde\eta^2\norm{\E[\cc]}^2 + \frac{4\tilde\eta^2}{N}\sum_{i}  \norm[\big]{\nabla f_i(\xx^\star)- \E[\cc_i]}^2\\&\hspace{3cm} + 8\beta\tilde\eta^2(\E[f(\xx)] - f(\xx^\star)) + \frac{12\tilde\eta^2\sigma^2}{KS}\\
   &\leq \underbrace{\frac{4\tilde\eta^2}{KN}\sum_{k,i}\E\norm[\big]{ \nabla f_i(\yy_{i,k-1}) - \nabla f_i(\xx)}^2}_{\cT_1} + \frac{8\tilde\eta^2}{N}\sum_{i}\norm[\big]{ \nabla f_i(\xx^\star)- \E[\cc_i]}^2\\&\hspace{3cm} + 8\beta\tilde\eta^2(\E[f(\xx)] - f(\xx^\star)) + \frac{12\tilde\eta^2\sigma^2}{KS}\,.
      \end{align*}
      The second step follows because $\cc = \frac{1}{N} \sum_{i}\cc_i$.
      Since the gradient of $f_i$ is $\beta$-Lipschitz, $\cT_1 \leq \frac{\beta^2 4\tilde\eta^2}{KN}\sum_{k,i}\E\norm[\big]{ \yy_{i,k-1} - \xx}^2 = 4\tilde\eta^2\beta^2\error$. The definition of the error in the control variate $\ce_{r-1} := \frac{1}{N}\sum_{j=1}^N\E\norm{\E[\cc_i] - \nabla f_i(\xx^\star)}^2$ completes the proof.
\end{proof}

\paragraph{Change in control lag.}
We have previously related the variance of the server update to the control lag. We now examine how the control-lag grows each round.
\begin{lemma}\label{lem:control-error-sample}
 For updates \eqref{eqn:local-updates-sample}---\eqref{eqn:global-updates-sample} with the control update \eqref{eqn:control-updates-sample} and assumptions
\ref{asm:strong-convexity}--\ref{asm:smoothness}, the following holds true for any $\tilde\eta := \eta_l\eta_g K \in [0 ,1/\beta]$:
     \[
\ce_{r} \leq \rbr{1 - \tfrac{S}{N}}\ce_{r-1} + \tfrac{S}{N}\rbr*{4\beta (\E[f(\xx^{r-1})] - f(\xx^\star)) + 2\beta^2 \error_{r}} \,.
     \]
\end{lemma}
\begin{proof}
 Recall that after round $r$, the control update rule \eqref{eqn:control-updates-sample} implies that $\cc^r_i$ is set as per
\[
    \cc^r_i = \begin{cases}
\cc^{r-1}_i &\text{ if } i \notin\ \cS^r \text{ i.e. with probability }(1 - \frac{S}{N}).\,,\\
\frac{1}{K}\sum_{k=1}^K g_i(\yy^{r}_{i,k-1}) &\text{ with probability }\frac{S}{N}.
    \end{cases}
\]
Taking expectations on both sides yields
\[
    \E[\cc^r_i] = (1 - \tfrac{S}{N})\E[\cc^{r-1}_i] + \tfrac{S}{KN}\tsum_{k=1}^K \E[\nabla f_i(\yy^{r}_{i,k-1})]\,, \ \ \forall\ i \in [N]\,.
\]
Plugging the above expression in the definition of $\ce_r$ we get
\begin{align*}
    \ce_r &= \frac{1}{N}\sum_{i=1}^N \norm{\E[\cc_i^r] - \nabla f_i(\xx^\star)}^2\\
    &= \frac{1}{N}\sum_{i=1}^N \norm{(1 - \tfrac{S}{N})(\E[\cc^{r-1}_i] - \nabla f_i(\xx^\star)) + \tfrac{S}{N} (\tfrac{1}{K}\tsum_{k=1}^K \E[\nabla f_i(\yy^{r}_{i,k-1})] - \nabla f_i(\xx^\star))}^2\\
    &\leq \rbr{1 - \tfrac{S}{N}} \ce_{r-1} + \tfrac{S}{N^2 K} \tsum_{k=1}^K \E\norm{\nabla f_i(\yy^{r}_{i,k-1}) -\nabla f_i(\xx^\star) }^2\,.
\end{align*}
The final step applied Jensen's inequality twice. We can then further simplify using the relaxed triangle inequality as
     \begin{align*}
 \E_{r-1}[\ce_{r}] &\leq \rbr*{1 - \frac{S}{N}}\ce_{r-1} + \frac{S}{N^2K}\sum_{i, k} \E \norm{\nabla f_i(\yy^{r}_{i,k-1}) - \nabla f_i(\xx^\star)}^2\\
 &\leq \rbr*{1 - \frac{S}{N}}\ce_{r-1} + \frac{2S}{N^2}\sum_{i} \E \norm{\nabla f_i(\xx^{r-1}) - \nabla f_i(\xx^\star)}^2 + \frac{2S}{N^2K}\sum_{i,k}\E \norm{\nabla f_i(\yy^{r}_{i,k-1}) - \nabla f_i(\xx^{r-1})}^2\\
 &\refLE{eqn:lip-grad} \rbr*{1 - \frac{S}{N}}\ce_{r-1} + \frac{2S}{N^2}\sum_{i} \E \norm{\nabla f_i(\xx^{r-1}) - \nabla f_i(\xx^\star)}^2 + \frac{2S}{N^2K}\beta^2 \sum_{i,k}\E \norm{\yy^{r}_{i,k-1} - \xx^{r-1}}^2\\
 &\refLE{eqn:smoothness}\rbr*{1 - \frac{S}{N}}\ce_{r-1} + \frac{S}{N}(4\beta(\E[f(\xx^{r-1})] - f(\xx^\star)) + \beta^2\error_r)\,.
     \end{align*}
     The last two inequalities follow from smoothness of $\{f_i\}$ and the definition $\error_r = \frac{1}{NK}\beta^2 \sum_{i,k}\E \norm{\yy^{r}_{i,k-1} - \xx^{r-1}}^2$.
     \end{proof}

\paragraph{Bounding client-drift.} We will now bound the final source of error which is the client-drift.
\begin{lemma}\label{lem:drift-bound-sampling}
    Suppose our step-sizes satisfy $\eta_l \leq \frac{1}{81 \beta K\eta_g}$ and $f_i$ satisfies assumptions \ref{asm:strong-convexity}--\ref{asm:smoothness}. Then, for any global $\eta_g \geq 1$  we can bound the drift as
    \[
       3\beta\tilde\eta\error_{r}  \leq \tfrac{2\tilde\eta^2}{3}\ce_{r-1}+ \tfrac{\tilde\eta}{25\eta_g^2}(\E[f(\xx^{r-1})] - f(\xx^\star))+ \tfrac{\tilde\eta^2 }{K\eta_g^2}\sigma^2\,.
    \]
\end{lemma}
\begin{proof}
    First, observe that if $K=1$, $\error_r =0$ since $\yy_{i,0} = \xx$ for all $i \in [N]$ and that $\ce_{r-1}$ and the right hand side are both positive. Thus the lemma is trivially true if $K=1$. For $K > 1$, we build a recursive bound of the drift.Starting from the definition of the update \eqref{eqn:local-updates-sample} and then applying the relaxed triangle inequality, we can expand
    \begin{align*}
        \frac{1}{S}\E_{r-1}\sbr[\bigg]{\sum_{i \in \cS} \norm[\big]{(\yy_i - \eta_l\vv_i) - \xx}^2} &= \frac{1}{S}\E_{r-1}\sbr[\bigg]{\sum_{i \in \cS} \norm[\big]{\yy_i - \eta_l g_i(\yy_i) + \eta_l\cc - \eta_l\cc_i - \xx}^2}\\
        &\leq \frac{1}{S}\E_{r-1}\sbr[\bigg]{\sum_{i \in \cS} \norm[\big]{\yy_i - \eta_l \nabla f_i(\yy_i) + \eta_l\cc - \eta_l\cc_i - \xx}^2} + \eta_l^2\sigma^2\\
        &\leq \frac{(1+ a)}{S}\E_{r-1}\sbr[\bigg]{\sum_{i \in \cS} \underbrace{\norm[\big]{\yy_i - \eta_l \nabla f_i(\yy_i) + \eta_l\nabla f_i(\xx) - \xx}^2}_{\cT_2}} \\ &\hspace{3cm}+(1 + \tfrac{1}{a})\eta_l^2 \underbrace{\E_{r-1}\sbr[\bigg]{\frac{1}{S}\sum_{i \in \cS}\norm{\cc - \cc_i + \nabla f_i(\xx)}^2}}_{\cT_3} + \eta_l^2\sigma^2\,.
    \end{align*}
    The final step follows from the relaxed triangle inequality (Lemma~\ref{lem:norm-sum}). Applying the contractive mapping Lemma~\ref{lem:contractive} for $\eta_l \leq 1/\beta$ shows
    \[
        \cT_2 = \frac{1}{S}\sum_{i \in \cS}\norm[\big]{\yy_i - \eta_l \nabla f_i(\yy_i) + \eta_l\nabla f_i(\xx) - \xx}^2 \leq \norm[\big]{\yy_i - \xx}^2\,.
    \]
    Once again using our relaxed triangle inequality to expand the other term $\cT_3$, we get
    \begin{align*}
        \cT_3 &= \E_{r-1}\sbr[\bigg]{\frac{1}{S}\sum_{i \in \cS}\norm{\cc - \cc_i + \nabla f_i(\xx)}^2}\\
        &= \frac{1}{N}\sum_{j =1}^N\norm{\cc - \cc_i + \nabla f_i(\xx)}^2\\
        &= \frac{1}{N}\sum_{j =1}^N\norm{\cc - \cc_i + \nabla f_i(\xx^\star) + \nabla f_i(\xx) - \nabla f_i(\xx^\star)}^2\\
        &\leq 3\norm{\cc}^2 + \frac{3}{N}\sum_{j =1}^N \norm{\cc_i - \nabla f_i(\xx^\star)}^2  + \frac{3}{N}\sum_{j =1}^N \norm{\nabla f_i(\xx) - \nabla f_i(\xx^\star)}^2\\
        &\leq \frac{6}{N}\sum_{j =1}^N \norm{\cc_i - \nabla f_i(\xx^\star)}^2  + \frac{3}{N}\sum_{j =1}^N \norm{\nabla f_i(\xx) - \nabla f_i(\xx^\star)}^2\\
        &\leq \frac{6}{N}\sum_{j =1}^N \norm{\cc_i - \nabla f_i(\xx^\star)}^2  + 6\beta (f(\xx) - f(\xx^\star))\,.
    \end{align*}
    The last step used the smoothness of $f_i$. Combining the bounds on $\cT_2$ and $\cT_3$ in the original inequality and using $a = \frac{1}{K-1}$ gives    
    \begin{multline*}
    \frac{1}{N}\sum_{i}\E \norm[\big]{\yy_{i,k} - \xx}^2  \leq \frac{(1 + \tfrac{1}{K-1})}{N}\sum_{i}\E \norm[\big]{\yy_{i,k-1} - \xx}^2 + \eta_l^2 \sigma^2 \\ +  6\eta_l^2 K\beta (f(\xx) - f(\xx^\star)) + \frac{6K \eta_l^2}{N}\sum_{i}\E\norm{\cc_i - \nabla f_i(\xx^\star)}^2\,.
    \end{multline*}
    Recall that with the choice of $\cc_i$ in \eqref{eqn:control-updates-sample}, the variance of $c_i$ is less than $\frac{\sigma^2}{K}$. Separating its mean and variance gives
    \begin{multline}\label{eqn:error-recursion-sample}
    \frac{1}{N}\sum_{i}\E \norm[\big]{\yy_{i,k} - \xx}^2  \leq \rbr*{1 + \frac{1}{K-1}}\frac{1}{N}\sum_{i}\E \norm[\big]{\yy_{i,k-1} - \xx}^2 + 7\eta_l^2 \sigma^2 +\\  6\eta_l^2 K\beta (f(\xx) - f(\xx^\star)) + \frac{6K \eta_l^2}{N}\sum_{i}\norm{\E[\cc_i] - \nabla f_i(\xx^\star)}^2
    \end{multline}
    Unrolling the recursion \eqref{eqn:error-recursion-sample}, we get the following for any $k \in \{1,\dots,K\}$
    \begin{align*}
     \frac{1}{N}\sum_{i}\E \norm[\big]{\yy_{i,k} - \xx}^2  &\leq \rbr*{6K \beta\eta_l^2 (f(\xx) - f(\xx^\star)) + 6K \eta_l^2 \ce_{r-1} + 7 \beta \eta_l^2\sigma^2} \rbr*{\sum_{\tau=0}^{k-1} (1 + \tfrac{1}{K-1})^\tau} \\
     & \leq \rbr*{6K \beta\eta_l^2 (f(\xx) - f(\xx^\star)) + 6K \eta_l^2 \ce_{r-1} + 7 \beta \eta_l^2\sigma^2} (K-1) ((1 + \tfrac{1}{K-1})^K - 1)\\
      & \leq \rbr*{6K \beta\eta_l^2 (f(\xx) - f(\xx^\star)) + 6K \eta_l^2 \ce_{r-1} + 7 \beta \eta_l^2\sigma^2} 3K\\
     &\leq 18 K^2 \beta\eta_l^2 (f(\xx) - f(\xx^\star)) + 18 K^2 \eta_l^2 \ce_{r-1} + 21 K \beta \eta_l^2\sigma^2\,.
    \end{align*}
    The inequality $(K-1) ((1 + \tfrac{1}{K-1})^K - 1) \leq 3K$ can be verified for $K =2,3$ manually. For $K \geq 4$,
    \[
        (K-1) ((1 + \tfrac{1}{K-1})^K - 1) < K(\exp(\tfrac{K}{K-1}) - 1) \leq K (\exp(\tfrac{4}{3})-1) < 3K\,.
    \]
    Again averaging over $k$ and multiplying by $3\beta$ yields
    \begin{align*}
          3\beta \error_{r} &\leq   54 K^2 \beta^2\eta_l^2 (f(\xx) - f(\xx^\star)) + 54 K^2 \beta \eta_l^2 \ce_{r-1} + 63 \beta K \eta_l^2\sigma^2\\
          &= \tfrac{1}{\eta_g^2} \rbr*{54 \beta^2\tilde\eta^2 (f(\xx) - f(\xx^\star)) + 54\beta\tilde\eta^2 \ce_{r-1} + 63\beta \tilde\eta^2 \tfrac{\sigma^2}{K}}\\
          &\leq \tfrac{1}{\eta_g^2} \rbr*{\tfrac{1}{25} (f(\xx) - f(\xx^\star)) + \tfrac{2}{3}\tilde\eta \ce_{r-1} + \tilde\eta \tfrac{\sigma^2}{K}}\,.
    \end{align*}
    The equality follows from the definition $\tilde\eta = K\eta_l\eta_g$, and the final inequality uses the bound that $\tilde\eta \leq \tfrac{1}{81\beta}$. \end{proof}
    
\paragraph{Progress in one round.} Now that we have a bound on all errors, we can describe our progress.
\begin{lemma}\label{lem:progress-sample}
    Suppose assumptions \ref{asm:strong-convexity}--\ref{asm:smoothness} are true. Then the following holds for
    any step-sizes satisfying $\eta_g\geq 1$, $\eta_l \leq \min\rbr*{\tfrac{1}{81 \beta
        K\eta_g}, \tfrac{S}{15\mu N K\eta_g}}$, and effective step-size $\tilde\eta := K \eta_g \eta_l$
    \begin{multline*}
        \expect\sbr*{\norm{\xx^{r} - \xx^\star}^2 +
          \tfrac{9N\tilde\eta^2}{S} \ce_{r} } \leq (1 - \tfrac{\mu
          \tilde\eta}{2}) \rbr*{\E\norm{\xx^{r-1} - \xx^\star}^2+
          \tfrac{9N\tilde\eta^2}{S} \ce_{r-1}} - \tilde\eta(\E[f(\xx^{r-1})] -
        f(\xx^\star))+ \tfrac{12\tilde\eta^2 }{KS}(1 + \tfrac{S}{\eta_g^2})\sigma^2\,.
    \end{multline*}
\end{lemma}
\begin{proof}
    Starting from our server update equation,
    \[
\Delta\xx = -\frac{\tilde\eta}{KS}\sum_{k, i\in \cS} (g_i(\yy_{i,k-1}) + \cc - \cc_i) \text{, and } \E[\Delta\xx] = -\frac{\tilde\eta}{KN} \sum_{k, i} g_i(\yy_{i,k-1})\,.
    \]
    We can then apply Lemma~\ref{lem:server-variance-sample} to bound the second moment of the server update as
      \begin{align*}
\E_{r-1}\norm{\xx + \Delta\xx - \xx^\star}^2 &= \E_{r-1}\norm{\xx - \xx^\star}^2 - \frac{2\tilde\eta}{KS}\E_{r-1}\sum_{k, i \in \cS} \inp{\nabla f_i(\yy_{i,k-1})}{\xx - \xx^\star} + \E_{r-1}\norm[\big]{\Delta \xx}^2\nonumber\\
&\leq  \underbrace{\frac{2\tilde\eta}{KS}\E_{r-1}\sum_{k, i \in \cS} \inp{\nabla f_i(\yy_{i,k-1})}{\xx^\star - \xx}}_{\cT_4} +  \E_{r-1}\norm{\xx - \xx^\star}^2  \nonumber\\&\hspace{1cm} + 8\beta \tilde\eta^2 (\E[f(\xx^{r-1})] -f(\xx^\star)) + 8\tilde\eta^2 \ce_{r-1} + 4\tilde\eta^2\beta^2\error + \frac{12\tilde\eta^2\sigma^2}{KS} \,.
    \end{align*}
      The term $\cT_4$ can be bounded by using perturbed strong-convexity (Lemma~\ref{lem:magic}) with $h = f_i$, $\xx = \yy_{i,k-1}$, $\yy = \xx^\star$, and $\zz = \xx$ to get
    \begin{align*}
\E[\cT_4] &= \frac{2\tilde\eta}{KS}\E\sum_{k, i \in \cS} \inp*{ \nabla f_i(\yy_{i,k-1})}{\xx^\star - \xx} \\
&\leq \frac{2\tilde\eta}{KS} \E\sum_{k, i \in \cS} \rbr*{f_i(\xx^\star) - f_i(\xx) + \beta \norm{\yy_{i,k-1} - \xx}^2 - \frac{\mu}{4}\norm{\xx - \xx^\star}^2}\\
&= -2\tilde\eta\E\rbr*{f(\xx) - f(\xx^\star) + \frac{\mu}{4}\norm{\xx - \xx^\star}^2} + 2\beta \tilde\eta\error\,.
    \end{align*}
    Plugging $\cT_4$ back, we can further simplify the expression to get
     \begin{align*}
\E \norm{\xx + \Delta\xx - \xx^\star}^2 &\leq \E \norm{\xx - \xx^\star}^2  -2\tilde\eta\rbr*{f(\xx) - f(\xx^\star) + \frac{\mu}{4}\norm{\xx - \xx^\star}^2} + 2\beta\tilde\eta\error
\\ &\hspace{2cm}+ \frac{12 \tilde\eta^2 \sigma^2}{KS} + 8\beta \tilde\eta^2 (\E[f(\xx^{r-1})] -f(\xx^\star)) + 8\tilde\eta^2 \ce_{r-1} + 4\tilde\eta^2\beta^2\error  \\
&= (1 - \tfrac{\mu \tilde\eta}{2})\norm{\xx - \xx^\star}^2  + (8\beta\tilde\eta^2 -2\tilde\eta)(f(\xx) - f(\xx^\star))     \\&\hspace{2cm}+ \frac{12 \tilde\eta^2 \sigma^2}{KS}+ (2\beta\tilde\eta + 4\beta^2\tilde\eta^2)\error + 8\tilde\eta^2 \ce_{r-1}\,.
    \end{align*}
    We can use Lemma \ref{lem:control-error-sample} (scaled by $9\tilde\eta^2 \frac{N}{S}$) to bound the control-lag
    \[
 9\tilde\eta^2\tfrac{N}{S}\ce_r \leq  (1 - \tfrac{\mu \tilde\eta}{2})9\tilde\eta^2\tfrac{N}{S}\ce_{r-1} +  9(\tfrac{\mu\tilde\eta N}{2S} - 1) \tilde\eta^2\ce_{r-1} + 9\tilde\eta^2 \rbr*{4\beta (\E[f(\xx^{r-1})] - f(\xx^\star)) + 2\beta^2 \error}
    \]
    Now recall that Lemma~\ref{lem:drift-bound-sampling} bounds the client-drift:
    \[
        3\beta\tilde\eta\error_{r}  \leq \tfrac{2\tilde\eta^2}{3}\ce_{r-1}+ \tfrac{\tilde\eta}{25\eta_g^2}(\E[f(\xx^{r-1})] - f(\xx^\star))+ \tfrac{\tilde\eta^2 }{K\eta_g^2}\sigma^2\,.
    \]
    Adding all three inequalities together,
    \begin{align*}
\E \norm{\xx + \Delta\xx - \xx^\star}^2 +  \frac{9\tilde\eta^2 N \ce_r}{S} &\leq (1  - \frac{\mu\tilde\eta}{2})\rbr*{\E \norm{\xx - \xx^\star}^2 +  \frac{9\tilde\eta^2 N \ce_{r-1}}{S}}  + (44\beta\tilde\eta^2 - \frac{49}{25}\tilde\eta)(f(\xx) - f(\xx^\star)) \\ &\hspace{1cm} + \frac{12 \tilde\eta^2 \sigma^2}{KS}(1 + \tfrac{S}{\eta_g^2})+ (22\beta^2\tilde\eta^2 - \beta\tilde\eta)\error + (\tfrac{9\mu\tilde\eta N}{2S}- \tfrac{1}{3})\tilde\eta^2 \ce_{r-1}
    \end{align*}
    Finally, the lemma follows from noting that $\tilde\eta \leq \frac{1}{81\beta}$ implies $44\beta^2\tilde\eta^2 \leq \frac{24}{25}\tilde\beta$ and $\tilde\eta \leq \frac{S}{15 \mu N}$ implies $\tfrac{9\mu\tilde\eta N}{2S} \leq \frac{1}{3}$.
\end{proof}

\textbf{The final rate for strongly convex} follows simply by unrolling the recursive bound in Lemma~\ref{lem:progress-sample} using Lemma~\ref{lem:constant}. Also note that if $c^0_i = g_i(\xx^0)$, then $\frac{\tilde\eta N}{S}\ce_0$ can be bounded in terms of function sub-optimality $F$.
For the \textbf{general convex} setting, averaging over $r$ in Lemma~\ref{lem:progress-sample} with $\mu=0$ gives
\begin{align*}
        \frac{1}{R}\sum_{r=1}^{R} \E[f(\xx^{r-1})] -
        f(\xx^\star) &\leq \frac{1}{\tilde\eta R}\norm{\xx^{0} - \xx^\star}^2 +
          \frac{9N\tilde\eta}{S R} \ce_{0}  + \frac{12\tilde\eta }{KS}(1 + \tfrac{S}{\eta_g^2})\sigma^2\\
          &\leq 4\norm{\xx^{0} - \xx^\star} \sigma \sqrt{\frac{3(1 + S/\eta_g^2)}{RKS}}\\
          &\hspace{2cm} + \sqrt{\frac{N}{S}}\frac{\norm{\xx^{0} - \xx^\star}^2 + 9\ce_0}{R} + \frac{81 \beta \norm{\xx^{0} - \xx^\star}^2 }{R}\,.
          \,.
\end{align*}
The last step follows from using a step size of $\tilde\eta = \min\rbr*{\frac{1}{81 \beta}, \sqrt{\frac{S}{N}}, \frac{\norm{\xx^{0} - \xx^\star}}{\sigma}\sqrt{\frac{KS}{12R(1 + \tfrac{S}{\eta^2_g}}} }$.


\subsection{Convergence of \covfefe\ for non-convex functions  (Theorem \ref{thm:convergence})}
We now analyze the most general case of \covfefe\ with option II on functions which are potentially non-convex. Just as in the non-convex proof, we will first bound the variance of the server update in Lemma~\ref{lem:non-convex-variance-sample}, the change in control lag in Lemma~\ref{lem:non-convex-control-error-sample} and finally we bound the client-drift in Lemma~\ref{lem:nonconvex-drift-bound-sampling}. Combining these three together gives us the progress made in one round in Lemma~\ref{lem:non-convex-progress-sample}. The final rate is derived from the progress made using Lemma~\ref{lemma:general}.

\paragraph{Additional notation.}
Recall that in round $r$, we update the control variate as \eqref{eqn:control-updates-sample}
\[
    \cc_i^r = \begin{cases}
    \frac{1}{K}\sum_{k=1}^K g_i(\yy_{i,k-1}^{r}) &\text{ if } i \in \cS^r\,,\\
    \cc_i^{r-1} & \text{ otherwise}\,.
    \end{cases}
\]
We introduce the following notation to keep track of the `lag' in the update of the control variate: define a sequence of parameters $\{\alphav_{i,k-1}^{r-1}\}$ such that for any $i \in [N]$ and $k \in [K]$ we have $\alphav_{i,k-1}^{0} := \xx^0$ and for $r \geq 1$,
\begin{equation}\label{def:alpha}
    \alphav_{i,k-1}^{r} := \begin{cases}
    \yy_{i,k-1}^{r} &\text{ if } i \in \cS^r\,,\\
    \alphav_{i,k-1}^{r-1} & \text{ otherwise}\,.
    \end{cases}
\end{equation}
By the update rule for control variates \eqref{eqn:control-updates-sample} and the definition of $\{\alphav_{i,k-1}^{r-1}\}$ above, the following property always holds:
\[
    \cc_i^r = \frac{1}{K}\sum_{k=1}^K g_i(\alphav_{i,k-1}^{r})\,.
\]
We can then define the following $\tce_r$ to be the error in control variate for round $r$:
     \begin{equation}\label{def:non-convex-control-lag}
        \tce_{r} := \frac{1}{KN}\sum_{k=1}^K \sum_{i=1}^N\E \norm{\alphav_{i,k-1}^{r} - \xx^{r}}^2\,.
     \end{equation}
Also recall the closely related definition of client drift caused by local updates:
 \[
    \error_r := \frac{1}{KN}\sum_{k=1}^K \sum_{i=1}^N \E[\norm{\yy_{i,k}^r - \xx^{r-1}}^2]\,.
 \]

 \paragraph{Variance of server update.} Let us analyze how the control variates effect the variance of the aggregate server update.
\begin{lemma}\label{lem:non-convex-variance-sample}
 For updates \eqref{eqn:local-updates-sample}---\eqref{eqn:global-updates-sample}and assumptions
\ref{asm:variance} and \ref{asm:smoothness}, the following holds true for any $\tilde\eta := \eta_l\eta_g K \in [0 ,1/\beta]$:
\[
    \E\norm{\E_{r-1} [\xx^r] - \xx^{r-1}}^2 \leq 2\tilde\eta^2 \beta^2\error_r + 2\tilde\eta^2\E\norm{\nabla f(\xx^{r-1})}^2 \,, \text{ and}
\]
     \[
\E \norm{\xx^r - \xx^{r-1}}^2 \leq 4\tilde\eta^2\beta^2\error_r + 8\tilde\eta^2 \beta^2 \tce_{r-1} + 4\tilde\eta^2  \E\norm{\nabla f(\xx^{r-1})}^2 + \frac{9 \tilde\eta^2 \sigma^2}{KS}\,.
     \]
\end{lemma}
\begin{proof}
Recall that that the server update satisfies
\[
    \E[\Delta\xx] = -\frac{\tilde\eta}{KN} \sum_{k, i} \E[g_i(\yy_{i,k-1})]\,.
\]
From the definition of $\alphav_{i,k-1}^{r-1}$ and dropping the superscript everywhere we have
    \[
\Delta\xx = -\frac{\tilde\eta}{KS}\sum_{k, i\in \cS} (g_i(\yy_{i,k-1}) + \cc - \cc_i) \text{ where } \cc_i = \frac{1}{K}\sum_{k}g_i(\alphav_{i,k-1}) \,.
    \]
    Taking norm on both sides and separating mean and variance, we proceed as
    \begin{align*}
        \E\norm{\Delta \xx}^2 &= \E\norm{-\frac{\tilde\eta}{KS}\sum_{k, i\in \cS} (g_i(\yy_{i,k-1}) - g_i(\alphav_{i,k-1}) + \cc - \cc_i)}^2\\
        &\leq\E\norm[\bigg]{-\frac{\tilde\eta}{KS}\sum_{k, i\in \cS} (\nabla f_i(\yy_{i,k-1}) + \E[\cc] - \E[\cc_i])}^2 + \frac{9 \tilde\eta^2 \sigma^2}{KS}\\
        &\leq \E\sbr[\bigg]{\frac{\tilde\eta^2  }{KS}\sum_{k, i\in \cS} \norm[\bigg]{\nabla f_i(\yy_{i,k-1}) + \E[\cc] - \E[\cc_i]}^2} + \frac{9 \tilde\eta^2 \sigma^2}{KS}\\
        &= \frac{\tilde\eta^2  }{KN}\sum_{k, i}\E \norm[\bigg]{(\nabla f_i(\yy_{i,k-1})- \nabla f_i(\xx)) + (\E[\cc] - \nabla f(\xx)) + \nabla f(\xx) - (\E[\cc_i]- \nabla f_i(\xx))}^2 + \frac{9 \tilde\eta^2 \sigma^2}{KS}\\
        &\leq \frac{4\tilde\eta^2  }{KN}\sum_{k, i}\E \norm{\nabla f_i(\yy_{i,k-1})- \nabla f_i(\xx)}^2 + \frac{8\tilde\eta^2  }{KN}\sum_{k, i}\E \norm{\nabla f_i(\alphav_{i,k-1}) - \nabla f_i(\xx)}^2 \\&\hspace{3cm}+ 4\tilde\eta^2  \E\norm{\nabla f(\xx)}^2 + \frac{9 \tilde\eta^2 \sigma^2}{KS}\\
        &\leq 4\tilde\eta^2\beta^2 \error_r + 8\beta^2 \tilde\eta^2 \tce_{r-1} + 4\tilde\eta^2  \E\norm{\nabla f(\xx)}^2 + \frac{9 \tilde\eta^2 \sigma^2}{KS}\,.
    \end{align*}
    In the first inequality, note that the three random variables---$\frac{1}{KS}\sum_{k,i\in\cS}g_i(\yy_{i,k})$,  $\frac{1}{S}\sum_{i\in\cS}\cc_i$, and $\cc$---may not be independent but each have variance smaller than $\frac{\sigma^2}{KS}$ and so we can apply Lemma~\ref{lem:independent}. The rest of the inequalities follow from repeated applications of the relaxed triangle inequality, $\beta$-Lipschitzness of $f_i$, and the definition of $\tce_{r-1}$ \eqref{def:non-convex-control-lag}. This proves the second statement. The first statement follows from our expression of $\E_{r-1}[\Delta \xx]$ and similar computations.
\end{proof}

\paragraph{Lag in the control variates.} We now analyze the `lag' in the control variates due to us sampling only a small subset of clients each round. Because we cannot rely on convexity anymore but only on the Lipschitzness of the gradients, the control-lag increases faster in the non-convex case.
\begin{lemma}\label{lem:non-convex-control-error-sample}
 For updates \eqref{eqn:local-updates-sample}---\eqref{eqn:global-updates-sample} and assumptions
\ref{asm:variance}, \ref{asm:smoothness}, the following holds true for any $\tilde\eta \leq \frac{1}{24 \beta} (\tfrac{S}{N})^\alpha$ for $\alpha \in [\frac{1}{2},1]$ where $\tilde \eta:= \eta_l\eta_g K$:
     \[
\tce_{r} \leq \rbr{1 - \tfrac{17S}{36N}}\tce_{r-1} + \tfrac{1}{48\beta^2}(\tfrac{S}{N})^{2\alpha -1} \norm{\nabla f(\xx^{r-1})}^2 + \tfrac{97}{48}(\tfrac{S}{N})^{2\alpha -1} \error_{r} + (\tfrac{S}{N \beta^2})\frac{\sigma^2}{32KS} \,.
     \]
\end{lemma}
\begin{proof}
The proof proceeds similar to that of Lemma~\ref{lem:control-error-sample} except that we cannot rely on convexity. Recall that after round $r$, the definition of $\alphav_{i,k-1}^r$ \eqref{def:alpha} implies that
\[
    \E_{\cS^r}[\alphav^r_{i,k-1}] = (1 - \tfrac{S}{N})\alphav^{r-1}_{i,k-1} + \tfrac{S}{N}\yy_{i,k-1}^r\,.
\]
Plugging the above expression in the definition of $\tce_r$ we get
\begin{align*}
    \tce_{r} &= \frac{1}{KN} \sum_{i,k}\E \norm{\alphav_{i,k-1}^{r} - \xx^{r}}^2\\
    &= \rbr*{1 - \frac{S}{N}}\cdot \underbrace{\frac{1}{KN}\sum_i\E \norm{\alphav_{i,k-1}^{r-1} - \xx^r }^2}_{\cT_5} + \frac{S}{N}\cdot\underbrace{\frac{1}{KN} \sum_{k, i} \E\norm{\yy^{r}_{i,k-1} -\xx^{r}}^2}_{\cT_6}\,.
\end{align*}
We can expand the second term $\cT_6$ with the relaxed triangle inequality to claim
\[
    \cT_6 \leq 2(\error_r + \E\norm{\Delta \xx^r}^2)\,.
\]
We will expand the first term $\cT_5$ to claim for a constant $b \geq 0$ to be chosen later
\begin{align*}
    \cT_5 &= \frac{1}{KN}\sum_i\E( \norm{\alphav_{i,k-1}^{r-1} - \xx^{r-1} }^2 + \norm{\Delta \xx^r}^2 + \E_{r-1}\inp*{\Delta \xx^r}{\alphav_{i,k-1}^{r-1} - \xx^{r-1} })\\
    &\leq  \frac{1}{KN}\sum_i\E (\norm{\alphav_{i,k-1}^{r-1} - \xx^{r-1} }^2 + \norm{\Delta \xx^r}^2 + \tfrac{1}{b}\norm{\E_{r-1}[\Delta \xx^r]}^2 + b\norm{\alphav_{i,k-1}^{r-1} - \xx^{r-1} }^2)
\end{align*}
where we used Young's inequality which holds for any $b \geq 0$. Combining the bounds for $\cT_5$ and $\cT_6$,
\begin{align*}
    \tce_{r} &\leq \rbr*{1 - \tfrac{S}{N}}(1 + b)\tce_{r-1} + 2\tfrac{S}{N}\error_r + 2 \E\norm{\Delta \xx^r}^2 + \tfrac{1}{b}\E\norm{ \E_{r-1}[\Delta \xx^r]}^2\\
    &\leq (\rbr*{1 - \tfrac{S}{N}}(1 + b) + 16\tilde\eta^2\beta^2)\tce_{r-1} + (\tfrac{2S}{N} + 8\tilde\eta^2\beta^2 + 2\tfrac{1}{b}\tilde\eta^2\beta^2)\error_r + (8 + 2\tfrac{1}{b})\tilde\eta^2\E\norm{\nabla f(\xx)}^2)+ \frac{18\tilde\eta^2\sigma^2}{KS}
\end{align*}
The last inequality applied Lemma~\ref{lem:non-convex-variance-sample}. Verify that with choice of $b = \frac{S}{2(N-S)}$, we have $\rbr*{1 - \tfrac{S}{N}}(1 + b) \leq (1 - \frac{S}{2N}) $ and $\frac{1}{b} \leq \frac{2N}{S}$\,. Plugging these values along with the bound on the step-size $16\beta^2\tilde\eta^2\leq \frac{1}{36}(\tfrac{S}{N})^{2\alpha} \leq \frac{S}{36N}$ completes the lemma.
\end{proof}

\paragraph{Bounding the drift.}
We will next bound the client drift $\error_r$. For this, convexity is not crucial and we will recover a very similar result to Lemma~\ref{lem:drift-bound-sampling} only use the Lipschitzness of the gradient.
 \begin{lemma}\label{lem:nonconvex-drift-bound-sampling}
Suppose our step-sizes satisfy $\eta_l \leq \frac{1}{24 \beta K\eta_g}$ and $f_i$ satisfies assumptions \ref{asm:variance}--\ref{asm:smoothness}. Then, for any global $\eta_g \geq 1$  we can bound the drift as
\[
    \tfrac{5}{3}\beta^2 \tilde\eta \error_r  \leq \tfrac{5}{3}\beta^3 \tilde\eta^2\tce_{r-1}+ \tfrac{\tilde\eta}{24\eta_g^2}\E\norm{\nabla f(\xx^{r-1})}^2 + \tfrac{\tilde\eta^2\beta }{4 K\eta_g^2}\sigma^2\,.
\]
\end{lemma}
\begin{proof}
First, observe that if $K=1$, $\error_r =0$ since $\yy_{i,0} = \xx$ for all $i \in [N]$ and that $\tce_{r-1}$ and the right hand side are both positive. Thus the Lemma is trivially true if $K=1$ and we will henceforth assume $K\geq 2$. Starting from the update rule \eqref{eqn:local-updates-sample} for $i \in [N]$ and $k \in [K]$
\begin{align*}
    \E\norm{\yy_{i,k} - \xx}^2 &= \E\norm{\yy_{i,k-1}-\eta_l(g_i(\yy_{i,k-1}) + \cc - \cc_i) - \xx}^2\\
    &\leq \E\norm{\yy_{i,k-1}-\eta_l(\nabla f_i(\yy_{i,k-1}) + \cc - \cc_i) - \xx}^2 + \eta_l^2 \sigma^2\\
    &\leq (1 + \tfrac{1}{K - 1})\E\norm{\yy_{i,k-1} - \xx}^2 +K \eta_l^2 \E\norm{\nabla f_i(\yy_{i,k-1}) + \cc - \cc_i}^2 + \eta_l^2 \sigma^2\\
    &= (1 + \tfrac{1}{K - 1})\E\norm{\yy_{i,k-1} - \xx}^2 + \eta_l^2 \sigma^2\\&\hspace{1cm}+K \eta_l^2 \E\norm{\nabla f_i(\yy_{i,k-1})- \nabla f_i(\xx)  + (\cc - \nabla f(\xx)) + \nabla f(\xx) - (\cc_i - \nabla f_i(\xx)}^2 \\
    &\leq (1 + \tfrac{1}{K - 1})\E\norm{\yy_{i,k-1} - \xx}^2 + 4K\eta_l^2 \E\norm{\nabla f_i(\yy_{i,k-1})- \nabla f_i(\xx)}^2  + \eta_l^2 \sigma^2\\&\hspace{1cm} + 4K\eta_l^2 \E\norm{\cc - \nabla f(\xx)}^2 + 4K\eta_l^2 \E\norm{\nabla f(\xx)}^2 + 4K\eta_l^2 \E\norm{\cc_i - \nabla f_i(\xx)}^2 \\
    &\leq (1 + \tfrac{1}{K - 1} + 4K\beta^2\eta_l^2)\E\norm{\yy_{i,k-1} - \xx}^2 + \eta_l^2 \sigma^2 + 4K\eta_l^2 \E\norm{\nabla f(\xx)}^2 \\&\hspace{1cm} + 4K\eta_l^2 \E\norm{\cc - \nabla f(\xx)}^2 + 4K\eta_l^2 \E\norm{\cc_i - \nabla f_i(\xx)}^2
\end{align*}
The inequalities above follow from repeated application of the relaxed triangle inequalities and the $\beta$-Lipschitzness of $f_i$. Averaging the above over $i$, the definition of $\cc = \frac{1}{N}\sum_i \cc_i$ and $\tce_{r-1}$ \eqref{def:non-convex-control-lag} gives
\begin{align*}
    \frac{1}{N}\sum_i\E\norm{\yy_{i,k} - \xx}^2 &\leq (1 + \tfrac{1}{K - 1} + 4K\beta^2\eta_l^2)\frac{1}{N}\sum_i\E\norm{\yy_{i,k-1} - \xx}^2 \\&\hspace{1cm}+ \eta_l^2 \sigma^2 + 4K\eta_l^2 \E\norm{\nabla f(\xx)}^2 + 8K\eta_l^2\beta^2 \tce_{r-1}\\
    &\leq \rbr*{\eta_l^2 \sigma^2 + 4K\eta_l^2 \E\norm{\nabla f(\xx)}^2 + 8K\eta_l^2\beta^2 \tce_{r-1}} \rbr*{\sum_{\tau =0 }^{k-1}(1 + \tfrac{1}{K - 1} + 4K\beta^2\eta_l^2)^{\tau}}\\
    &= \rbr*{\frac{\tilde\eta^2 \sigma^2}{K^2\eta_g^2} + \frac{4\tilde\eta^2}{K\eta_g^2} \E\norm{\nabla f(\xx)}^2 + \frac{8\tilde\eta^2\beta^2}{K\eta_g^2} \tce_{r-1}} \rbr*{\sum_{\tau =0 }^{k-1}(1 + \tfrac{1}{K - 1} + \frac{4\beta^2\tilde\eta^2}{K\eta_g^2})^{\tau}}\\
    &\leq \rbr*{\frac{\tilde\eta \sigma^2}{24 \beta K^2\eta_g^2} + \frac{1}{144 \beta^2 K\eta_g^2} \E\norm{\nabla f(\xx)}^2 + \frac{\tilde\eta\beta}{3K\eta_g^2} \tce_{r-1}}3K\,.
\end{align*}
The last inequality used the bound on the step-size $\beta\tilde\eta \leq \frac{1}{24}$. Averaging over $k$ and multiplying both sides by $\frac{5}{3}\beta^2 \tilde\eta$ yields the lemma statement.
\end{proof}

\paragraph{Progress made in each round.} Given that we can bound all sources of error, we can finally prove the progress made in each round.
\begin{lemma}\label{lem:non-convex-progress-sample}
 Suppose the updates \eqref{eqn:local-updates-sample}---\eqref{eqn:global-updates-sample}  satisfy assumptions \ref{asm:variance}--\ref{asm:smoothness}. For any effective step-size $\tilde\eta := K \eta_g \eta_l$ satisfying $\tilde\eta \leq \frac{1}{24\beta}\rbr*{\frac{S}{N}}^{\frac{2}{3}}$,
 \[
    \rbr[\bigg]{\E[f(\xx^{r})] + 12 \beta^3 \tilde\eta^2\tfrac{N}{S}\tce_{r} } \leq \rbr[\bigg]{\E[f(\xx^{r-1})]+ 12 \beta^3 \tilde\eta^2\tfrac{N}{S}\tce_{r-1}} + \frac{5\beta \tilde\eta^2 \sigma^2}{ K S}(1 + \tfrac{S}{\eta_g^2}) - \frac{\tilde\eta}{14}\E \norm{\nabla f(\xx^{r-1})}^2\,.
 \]
\end{lemma}
\begin{proof}
Starting from the smoothness of $f$ and taking conditional expectation gives
\[
    \E_{r-1}[f(\xx + \Delta \xx)] \leq f(\xx) + \inp{\nabla f(\xx)}{\E_{r-1}[\Delta \xx]} + \frac{\beta}{2}\E_{r-1}\norm{\Delta\xx}^2\,.
\]
We as usual dropped the superscript everywhere. Recall that the server update can be written as
    \[
\Delta\xx = -\frac{\tilde\eta}{KS}\sum_{k, i\in \cS} (g_i(\yy_{i,k-1}) + \cc - \cc_i) \text{, and } \E_\cS[\Delta\xx] = -\frac{\tilde\eta}{KN} \sum_{k, i} g_i(\yy_{i,k-1})\,.
    \]
    Substituting this in the previous inequality and applying Lemma~\ref{lem:non-convex-variance-sample} to bound $\E[\norm{\Delta \xx}^2]$ gives
    \begin{align*}
        \E[f(\xx + \Delta \xx)] - f(\xx) &\leq -\frac{\tilde\eta}{KN}\sum_{k, i} \inp{\nabla f(\xx)}{ \E[\nabla f_i(\yy_{i,k-1})]} + \frac{\beta}{2}\E\norm{\Delta \xx}^2\\
        &\leq -\frac{\tilde\eta}{KN}\sum_{k, i} \inp{\nabla f(\xx)}{ \E[\nabla f_i(\yy_{i,k-1})]} +  \\&\hspace{2cm} 2\tilde\eta^2\beta^3\error_r + 4\tilde\eta^2 \beta^3 \tce_{r-1} + 2\beta \tilde\eta^2  \E\norm{\nabla f(\xx)}^2 + \frac{9 \beta \tilde\eta^2 \sigma^2}{2KS}\\
        &\leq - \frac{\tilde\eta}{2}\norm{\nabla f(\xx)}^2 + \frac{\tilde\eta }{2 }\sum_{i,k}\E\norm[\bigg]{\frac{1}{KN }\sum_{i,k}\nabla f_i(\yy_{i,k-1}) - \nabla f(\xx)}^2 +  \\&\hspace{2cm} 2\tilde\eta^2\beta^3\error_r + 4\tilde\eta^2 \beta^3 \tce_{r-1} + 2\beta \tilde\eta^2  \E\norm{\nabla f(\xx)}^2 + \frac{9 \beta \tilde\eta^2 \sigma^2}{2KS} \\
        &\leq - \frac{\tilde\eta}{2}\norm{\nabla f(\xx)}^2 + \frac{\tilde\eta }{2KN }\sum_{i,k}\E\norm[\bigg]{\nabla f_i(\yy_{i,k-1}) - \nabla f_i(\xx)}^2 +  \\&\hspace{2cm} 2\tilde\eta^2\beta^3\error_r + 4\tilde\eta^2 \beta^3 \tce_{r-1} + 2\beta \tilde\eta^2  \E\norm{\nabla f(\xx)}^2 + \frac{9 \beta \tilde\eta^2 \sigma^2}{2KS} \\
        &\leq - (\tfrac{\tilde\eta}{2} - 2\beta \tilde\eta^2)\norm{\nabla f(\xx)}^2 + (\tfrac{\tilde\eta }{2} + 2\beta\tilde\eta^2)\beta^2\error_r+ 4\beta^3\tilde\eta^2 \tce_{r-1}  + \frac{9 \beta \tilde\eta^2 \sigma^2}{2KS}\,.
    \end{align*}
    The third inequality follows from the observation that $-ab = \frac{1}{2}((b - a)^2 - a^2) - \frac{1}{2}b^2 \leq \frac{1}{2}((b - a)^2 - a^2)$ for any $a,b \in \real$, and the last from the $\beta$-Lipschitzness of $f_i$. Now we use Lemma~\ref{lem:non-convex-control-error-sample} to bound $\tce_r$ as
    \begin{align*}
        12 \beta^3 \tilde\eta^2\tfrac{N}{S}\tce_{r} &\leq 12 \beta^3 \tilde\eta^2\tfrac{N}{S}\rbr*{\rbr{1 - \tfrac{17S}{36N}}\tce_{r-1} + \tfrac{1}{48\beta^2}(\tfrac{S}{N})^{2\alpha -1} \norm{\nabla f(\xx^{r-1})}^2 + \tfrac{97}{48}(\tfrac{S}{N})^{2\alpha -1} \error_{r} + (\tfrac{S}{N \beta^2})\frac{\sigma^2}{32KS}}\\
        &= 12 \beta^3 \tilde\eta^2\tfrac{N}{S}\tce_{r-1} - \tfrac{17}{3}\beta^3\tilde\eta^2\tce_{r-1} + \tfrac{1}{4}\beta \tilde\eta^2(\tfrac{N}{S})^{2 - 2\alpha}\norm{\nabla f(\xx)}^2 + \tfrac{97}{4}\beta^3\tilde\eta^2(\tfrac{N}{S})^{2 - 2\alpha} \error_r + \frac{3\beta \tilde\eta^2 \sigma^2}{8KS}\,.
    \end{align*}
    Also recall that Lemma~\ref{lem:nonconvex-drift-bound-sampling} states that
    \[
        \tfrac{5}{3}\beta^2 \tilde\eta \error_r  \leq \tfrac{5}{3}\beta^3 \tilde\eta^2\tce_{r-1}+ \tfrac{\tilde\eta}{24\eta_g^2}\E\norm{\nabla f(\xx^{r-1})}^2 + \tfrac{\tilde\eta^2\beta }{4 K\eta_g^2}\sigma^2\,.
    \]
    Adding these bounds on $\tce_r$ and $\error_r$ to that of $\E[f(\xx + \Delta \xx)]$ gives
    \begin{multline*}
        (\E[f(\xx + \Delta \xx)] + 12 \beta^3 \tilde\eta^2\tfrac{N}{S}\tce_{r}) \leq (\E[f(\xx)] + 12 \beta^3 \tilde\eta^2\tfrac{N}{S}\tce_{r-1}) + (\tfrac{5}{3}- \tfrac{17}{3})\beta^3\tilde\eta^2\tce_{r-1}\\ - (\tfrac{\tilde\eta}{2} - 2\beta \tilde\eta^2- \tfrac{1}{4}\beta \tilde\eta^2(\tfrac{N}{S})^{2 - 2\alpha})\norm{\nabla f(\xx)}^2 + (\tfrac{\tilde\eta }{2} - \tfrac{5\tilde\eta }{3} + 2\beta\tilde\eta^2 + \tfrac{97}{4}\beta\tilde\eta^2(\tfrac{N}{S})^{2 - 2\alpha})\beta^2\error_r  + \tfrac{39 \beta \tilde\eta^2 \sigma^2}{8KS}(1 + \tfrac{S}{\eta_g^2})\,.
    \end{multline*}
    By our choice of $\alpha = \frac{2}{3}$ and plugging in the bound on step-size $\beta \tilde\eta (\tfrac{N}{S})^{2 - 2\alpha} \leq \frac{1}{24}$ proves the lemma.
    \end{proof}
The \textbf{non-convex rate} of convergence now follows by unrolling the recursion in Lemma~\ref{lem:non-convex-progress-sample} and selecting an appropriate step-size $\tilde\eta$ as in Lemma~\ref{lemma:general}. Finally note that if we initialize $\cc_i^0 = g_i(\xx^0)$ then we have $\tce_0 = 0$.


\newpage
\section{Usefulness of local steps  (Theorem \ref{thm:interpolation})}\label{appsec:interpolation}
Let us state our rates of convergence for \covfefe\ which interpolates between identical and completely heterogeneous clients. In this section, we always set $\eta_g = 1$ and assume all clients participate $(S = N)$.
\begin{theorem}\label{thm:interpolation-full}
    Suppose that the functions $\{f_i\}$ are quadratic and satisfy assumptions
\ref{asm:variance}, \ref{asm:smoothness} and additionally \ref{asm:hessian-similarity}. Then, for global step-size $\eta_g = 1$ in each of the following cases, there exist probabilities
    $\{p_k^r\}$ and local step-size $\eta_l$
    such that the output \eqref{eqn:covfefe-agg-simple} of \covfefe\ when run with no client sampling ($S = N$) using update \eqref{eqn:covfefe-simple} satisfies:
        \begin{itemize}
            \item \textbf{Strongly convex:} $f_i$ satisfies \eqref{asm:strong-convexity} for $\mu >0$, $\eta_l \leq \min(\frac{1}{10\beta}, \frac{1}{22\delta K}, \frac{1}{10\mu K})$, $R \geq \max(\frac{20\beta}{\mu}, \frac{44\delta K + 20\mu K}{\mu}, 20K)$ then
            \[
    \expect[\norm{\nabla f(\bar\xx^R)}^2] \leq
    \tilde\cO\rbr*{\frac{\beta \sigma^2}{\mu RKN} + \mu D^2 \exp\rbr*{-\frac{\mu}{20\beta + 44\delta K + 20\mu K}RK}}\,.
        \]
        \item \textbf{General convex:} $f$ satisfies $\nabla^2 f \mgeq -\delta I$, $\eta_l\leq \min(\frac{1}{10\beta}, \frac{1}{22\delta K})$, and $R \geq 1$, then
        \[
    \expect[\norm{\nabla f(\bar\xx^R)}^2] \leq
    \cO\rbr*{\frac{ \sigma \sqrt{\beta (f(\xx_0) - f^\star)}}{\sqrt{RKN}} + \frac{(\beta + \delta K) (f(\xx_0) - f^\star)}{R K}}\,.
        \]
        \end{itemize}
    \end{theorem}
Note that if $\delta = 0$, we match (up to acceleration) the lower bound in \citep{woodworth2018graph}. While certainly $\delta = 0$ when the functions are identical as studied in \citep{woodworth2018graph}, our upper-bound is significantly stronger since it is possible that $\delta = 0$ even for highly heterogeneous functions. For example, objective perturbation \citep{chaudhuri2011differentially,kifer2012private} is an optimal mechanism to achieve differential privacy for smooth convex objectives \citep{bassily2014private}. Intuitively, objective perturbation relies on masking each client's gradients by adding a large random linear term to the objective function. In such a case, we would have high gradient dissimilarity but no Hessian dissimilarity.

Our non-convex convergence rates are the first of their kind as far as we are aware---no previous work shows how one can take advantage of similarity for non-convex functions. However, we should note that non-convex quadratics do not have a global lower-bound on the function value $f^\star$. We will instead assume that $f^\star$ almost surely lower-bounds the value of $f(\xx^R)$, implicitly assuming that the iterates remain bounded.

\paragraph{Outline.} In the rest of this section, we will focus on proving Theorem~\ref{thm:interpolation-full}. We will show how to bound variance in Lemma~\ref{lem:interp-variance}, bound the amount of drift in Lemma~\ref{lem:interp-drift}, and show progress made in one step in Lemma~\ref{lem:interp-onestep}. In all of these we do not use convexity, but strongly rely on the functions being quadratics. Then we combine these to derive the progress made by the server in one round---for this we need \emph{weak}-convexity to argue that averaging the parameters does not hurt convergence too much. As before, it is straight-forward to derive rates of convergence from the one-round progress using Lemmas~\ref{lem:constant} and \ref{lemma:general}.

\subsection{Additional notation and assumptions}
For any matrix $M$ and vector $\vv$, let $\norm{\vv}_M^2 := \vv^\top M \vv$. Since all functions in this section are quadratics, we can assume w.l.o.g they are of the following form:
\begin{align*}
    f_i(\xx) - f_i(\xx_i^\star) = \frac{1}{2}\norm{\xx - \xx_i^\star}_{A_i}^2 \text{ for } i \in [N]\,, \text{ and }
    f(\xx) = \frac{1}{2}\norm{\xx - \xx_i^\star}_{A}^2 \,, \text{ for all } \xx\,,
\end{align*}
for some $\{\xx_i^\star\}$ and $\xx^\star$, $A := \frac{1}{N}\sum_{i=1}^N A_i$. We also assume that $A$ is a symmetric matrix though this requirement is easily relaxed. Note that this implies $f(\xx^\star) = 0$ and that $\nabla f_i(\xx) = A(\xx - \xx^\star_i)$. If $\{f_i\}$ are additionally convex, we have that $\xx_i^\star$ is the optimum of $f_i$ and $\xx^\star$ the optimum of $f$. However, this is not necessarily true in general.

We will also focus on a simplified version of \covfefe\ where in each round $r$, client $i$ performs the following update starting from $\yy_{i,0}^r \leftarrow \xx^{r-1}$:
\begin{align}
    \begin{split}\label{eqn:covfefe-simple}
    \yy_{i,k}^r &= \yy_{i,k-1}^r - \eta (g_i(\yy_{i,k-1}^r) + \nabla f(\xx^{r-1}) - \nabla f_i(\xx^{r-1}))\,, \text { i.e.}\\
        \E_{r-1,k-1}[\yy_{i,k}^r] &= \yy_{i,k-1}^r - \eta A(\yy_{i,k-1}^r - \xx^\star) - \eta(A_i - A)(\yy_{i,k-1}^r - \xx^{r-1}))\,,
    \end{split}
\end{align}
where the second part is specialized to quadratics and the expectation is conditioned over everything before current step $k$ of round $r$. At the end of each round, as before, $\xx^r = \frac{1}{N}\sum_{i=1}^N\yy^r_{i,K}$. The final output of the algorithm is chosen using probabilities $\{p^r_k\}$ as
\begin{equation}\label{eqn:covfefe-agg-simple}
    \bar\xx^R = \xx^r_k \text { with probability }p_k^r \,, \text{ where } \xx^r_k := \frac{1}{N}\sum_{i=1}^N \yy_{i,k}^r \,.
\end{equation}
Note that we are now possibly outputting iterates computed within a single round and that $\xx^r = \xx^r_K$. Beyond this, the update above differs from our usual \covfefe\ in two key aspects: a) it uses gradients computed at $\xx^{r-1}$ as control variates instead of those at either $\xx^{r-2}$ (as in option I) or $\yy_{i,k}^{r-1}$ (as in option II), and b) it uses full batch gradients to compute its control variates instead of stochastic gradients. The first issue is easy to fix and our proof extends to using both option I or option II using techniques in Section~\ref{appsec:covfefe}. The second issue is more technical---using stochastic gradients for control variates couples the randomness across the clients in making the local-updates \emph{biased}. While it may be possible to get around this (cf. \citep{lei2017less,nguyen2017sarah,tran2019hybrid}), we will not attempt to do so in this work. Note that if $K$ local update steps typically represents running multiple epochs on each client. Hence one additional epoch to compute the control variate $\nabla f_i(\xx)$ does not significantly add to the cost.

Finally, we define the following sequence of positive numbers for notation convenience:
\begin{align*}
   \xi_{i,k}^r &:= \rbr*{\E_{r-1}[f(\yy_{i,k}^r)] - f(\xx^\star) + \delta(1 + \tfrac{1}{K})^{K-k}\E_{r-1}\norm{\yy_{i,k}^r- \xx^{r-1}}^2}\,, \text{ and}\\
   \tilde\xi_{i,k}^r &:= \rbr*{[f(\E_{r-1}[\yy_{i,k}^r])] - f(\xx^\star) + \delta(1 + \tfrac{1}{K})^{K-k}\E_{r-1,k-1}\norm{\E_{r-1}[\yy_{i,k}^r] - \xx^{r-1}}^2}\,.
\end{align*}
Observe that for $k=0$, $\xi_{i,0}^r = \tilde \xi_{i,0}^r = f(\xx^{r-1}) - f(\xx^\star)$.

\subsection{Lemmas tracking errors}
\paragraph{Effect of averaging.} We see how averaging can reduce variance. A similar argument was used in the special case of one-shot averaging in \citep{zhang2013communication}.
\begin{lemma}\label{lem:avg-variance}
    Suppose $\{f_i\}$ are quadratic functions and assumption \ref{asm:variance} is satisfied. Then let $\xx_{k}^r$ and $\yy_{i,k}^r$ be vectors in step $k$ and round $r$ generated using \eqref{eqn:covfefe-simple}---\eqref{eqn:covfefe-agg-simple}. Then,
    \[
        \E_{r-1}\norm{\nabla f(\xx^r_k)}^2 \leq \frac{1}{N}\sum_{i=1}^N\norm{\nabla f(\E_{r-1}[\yy_{i,k}^r])}^2 + \frac{1}{N^2}\sum_{i=1}^N\E_{r-1}[\norm{\nabla f(\yy^r_{i,k})}^2]\,.
    \]
\end{lemma}
\begin{proof} Observe that the variables $\{\yy_{i,k} - \xx\}$ are independent of each other (the only source of randomness is the local gradient computations). The rest of the proof is exactly that of Lemma~\ref{lem:independent}. Dropping superscripts everywhere, 
    \begin{align*}
        \E_{r-1}\norm{A(\xx_k^r - \xx^\star)}^2 &= \E_{r-1}\norm{\tfrac{1}{N}\sum_i A(\yy_{i,k} - \xx^\star)}^2\\
        &= \E_{r-1}\norm{\tfrac{1}{N}\sum_i A(\E_{r-1}[\yy_{i,k}] - \xx^\star)}^2 + \E_{r-1}\norm{\tfrac{1}{N}\sum_i A(\E_{r-1}[\yy_{i,k}] - \yy_{i,k})}^2\\
        &= \E_{r-1}\norm{\tfrac{1}{N}\sum_i A(\E_{r-1}[\yy_{i,k}] - \xx^\star)}^2 + \tfrac{1}{N^2}\sum_i \E_{r-1}\norm{A(\E_{r-1}[\yy_{i,k}] - \yy_{i,k})}^2\\
        &= \E_{r-1}\norm{\tfrac{1}{N}\sum_i A(\E_{r-1}[\yy_{i,k}] - \xx^\star)}^2 + \tfrac{1}{N^2}\sum_i \E_{r-1}\norm{A(\yy_{i,k} - \xx^\star - \E_{r-1}[\yy_{i,k} - \xx^\star])}^2\\
        &\leq \E_{r-1}\norm{\tfrac{1}{N}\sum_i A(\E_{r-1}[\yy_{i,k}] - \xx^\star)}^2 + \tfrac{1}{N^2}\sum_i \E_{r-1}\norm{A(\yy_{i,k} - \xx^\star)}^2\,.
    \end{align*}
    The third equality was because $\{\yy_{i,k}\}$ are independent of each other conditioned on everything before round $r$.
\end{proof}

We next see the effect of averaging on function values.
\begin{lemma}\label{lem:averaging-value}
    Suppose that $f$ is $\delta$ general-convex, then we have:
    \[
        \frac{1}{N}\sum_{i=1}^n \xi_{i,k}^r \geq \E_{r-1}[f(\xx^r_{k})] - f(\xx^\star)\,, \text{ and } \frac{1}{N}\sum_{i=1}^n \tilde\xi_{i,k}^r \geq f(\E_{r-1}[\xx^r_{k}]) - f(\xx^\star)\,.
    \]
\end{lemma}
\begin{proof}
    Since $f$ is $\delta$-general convex, it follows that the function $f(\zz) + \delta(1 + \tfrac{1}{K})^{K-k}\norm{\zz - \xx}^2_2$ is convex in $\zz$ for any $k \in [K]$. The lemma now follows directly from using convexity and the definition of $\xx^r_k = \frac{1}{N}\yy_{i,k}^r$.
\end{proof}

\paragraph{Bounding drift of one client.} We see how the client drift of \covfefe\ depends on $\delta$.
\begin{lemma}\label{lem:interp-drift}
    For the update \eqref{eqn:covfefe-simple}, assuming \eqref{asm:hessian-similarity} and that $\{f_i\}$ are quadratics, the following holds for any $\eta \leq \frac{1}{21\delta K}$
    \[
        \E_{r-1,k-1}\norm{\yy_{i,k}^r - \xx^{r-1}}^2 \leq (1 + \tfrac{1}{2K})\norm{\yy_{i,k-1}^r - \xx^{r-1}}^2 + 7K\eta^2\norm{\nabla f(\yy_{i,k-1}^r)}^2 + \eta^2\sigma^2\,.
    \]
\end{lemma}
\begin{proof}
    Starting from the update step \eqref{eqn:covfefe-simple}
    \begin{align*}
        \E_{r-1,k-1}\norm{\yy_i^+ - \xx}^2 &\leq \norm{\yy_i - \xx -\eta A(\yy_i - \xx^\star) -\eta(A_i - A)(\yy_i - \xx)}^2 + \eta^2\sigma^2\\
        &\leq (1 + \tfrac{1}{7(K-1)})\norm{(I-\eta(A_i - A)(\yy_i - \xx)}^2 + 7K\eta^2\norm{A(\yy_i - \xx^\star)}^2 + \eta^2\sigma^2\,.
    \end{align*}
    Note that if $K=1$, then the first inequality directly proves the lemma. For the second inequality, we assumed $K \geq 2$ and then applied our relaxed triangle inequality. By assumption \ref{asm:hessian-similarity}, we have the following for $\eta\delta \leq 1$
    \[
        \norm{(I - \eta(A_i - A))^2} = \norm{I - \eta(A_i - A)}^2 \leq (1 + \eta\delta)^2 \leq 1 + 3\eta\delta\,.
    \]
    Using the bound on the step-size $\eta \leq \frac{1}{21 \delta K}$ gives
    \[
        \E_{r-1,k-1}\norm{\yy_i^+ - \xx}^2 \leq (1 + \tfrac{1}{7K})(1 + \tfrac{1}{7(K-1)})\norm{\yy_i - \xx}^2 + 7K\eta^2\norm{A(\yy_i - \xx^\star)}^2 + \eta^2\sigma^2
    \]
    Simple computations now give the Lemma statement for all $K\geq 1$.
\end{proof}

\paragraph{Tracking the variance.} We will see how to bound the variance of the output.
\begin{lemma}\label{lem:interp-variance}
Consider the update \eqref{eqn:covfefe-simple} for quadratic $\{f_i\}$ with $\eta \leq \max(\frac{1}{2\delta K}, \frac{1}{\beta})$. Then, if further \eqref{asm:hessian-similarity},~\eqref{asm:smoothness} and \eqref{asm:variance} are satisfied, we have
\[
    \E_{r-1}f(\xx^r) \leq f(\E_{r-1}[\xx^r]) + 3K \beta \tfrac{\sigma^2}{N}\,.
\]
Further if $\{f_i\}$ are strongly convex satisfying \eqref{asm:strong-convexity}, we have
\[
    \E_{r-1}f(\xx^r) \leq f(\E_{r-1}[\xx^r]) + \beta\tfrac{\sigma^2}{N}\sum_{k=1}^K (1 - \mu \eta)^{k-1}\,.
\]
\end{lemma}
\begin{proof}
    We can rewrite the update step \eqref{eqn:covfefe-simple} as below:
    \[
        \yy_{i,k} = \yy_{i,k-1} - \eta (A_i(\yy_{i,k-1} - \xx^\star) + (A - A_i)(\xx - \xx^\star)) - \eta\zeta_{i,k}\,,
    \]
    where by the bounded variance assumption \ref{asm:variance}, $\zeta_{i,k}$ is a random variable satisfying $\E_{k-1,r-1}[\zeta_{i,k}] = 0$ and $\E_{k-1,r-1}\norm{\zeta_{i,k}}^2 \leq \sigma^2$. Subtracting $\xx^\star$ from both sides and unrolling the recursion gives
    \begin{align*}
        \yy_{i,K} - \xx^\star &= (I - \eta A_i)(\yy_{i,K-1} - \xx^\star) - \eta ((A - A_i)(\xx - \xx^\star) + \zeta_{i,K})\\
        &= (I - \eta A_i)^K(\xx - \xx^\star) -  \sum_{k=1}^{K} \eta (I - \eta A_i)^{k-1}(\zeta_{i,k} + (A - A_i)(\xx - \xx^\star))\,.
    \end{align*}
    Similarly, the expected iterate satisfies the same equation without the $\zeta_{i,k}$
    \[
        \E_{r-1}[\yy_{i,K}] - \xx^\star = (I - \eta A_i)^K(\xx - \xx^\star) -  \sum_{k=1}^{K} \eta (I - \eta A_i)^{k-1} (A - A_i)(\xx - \xx^\star)\,.
    \]
    This implies that the difference satisfies
    \[  
        \E_{r-1}[\yy_{i,K}] - \yy_{i,K} = \eta \sum_{k=1}^K (I - \eta A_i)^{k-1}\zeta_{i,k}\,.
    \]
    We can relate this to the function value as follows:
    \begin{align*}
        \E_{r-1}\norm{\xx^r_K -\xx^\star}_A^2 &= \norm{\E_{r-1}[\xx^r_k] - \xx^\star}_A^2 + \E_{r-1}\norm{\E_{r-1}[\xx^r_k] - \xx^r_K}_A^2 \\
        &= \norm{\E_{r-1}[\xx^r_k] - \xx^\star}_A^2 + \E_{r-1}\norm{\tfrac{1}{N}\tsum_{i}(\E_{r-1}[\yy_{i,K}] - \yy_{i,K})}_A^2\\
        &= \norm{\E_{r-1}[\xx^r_k] - \xx^\star}_A^2 + \eta^2 \E_{r-1}\norm{\tfrac{1}{N}\tsum_{i,k}(I - \eta A_i)^{k-1}\zeta_{i,k}}_A^2\\
        &= \norm{\E_{r-1}[\xx^r_k] - \xx^\star}_A^2 + \tfrac{\eta^2}{N^2} \E_{r-1}\tsum_{i,k}\norm{(I - \eta A_i)^{k-1}\zeta_{i,k}}_A^2\\
        &\leq \norm{\E_{r-1}[\xx^r_k] - \xx^\star}_A^2 + \tfrac{\beta \eta^2}{N^2} \E_{r-1}\tsum_{i,k}\norm{(I - \eta A_i)^{k-1}\zeta_{i,k}}_2^2\,.
    \end{align*}
    The last inequality used smoothness of $f$ and the one before that relied on the independence of $\zeta_{i,k}$. Now, if $f_i$ is general convex we have for $\eta \leq \frac{1}{2\delta K}$ that $I - \eta A_i \mleq (1 + \frac{1}{2K})I$ and hence
    \[
        \norm{(I - \eta A_i)^{k-1}\zeta_{i,k}}_2^2 \leq \sigma^2 (1 + \tfrac{1}{2K})^{2(k-1)} \leq 3\sigma^2\,.
    \]
    This proves our second statement of the lemma. For strongly convex functions, we have for $\eta \leq \frac{1}{\beta}$,
    \[
        \norm{(I - \eta A_i)^{k-1}\zeta_{i,k}}_2^2 \leq \sigma^2 (1 - \eta\mu)^{2(k-1)}\leq \sigma^2 (1 - \eta\mu)^{k-1}\,.
    \]
\end{proof}

\subsection{Lemmas showing progress}
\paragraph{Progress of one client in one step.} Now we focus only on a single client and monitor their progress.
\begin{lemma}\label{lem:interp-onestep}
    Suppose \eqref{asm:hessian-similarity}, \eqref{asm:smoothness} and \eqref{asm:variance} hold, and $\{f_i\}$ are quadratics. Then, the following holds for the update \eqref{eqn:covfefe-simple} with $\eta \leq \min(\frac{1}{10\beta}, \frac{1}{22\delta K}, \frac{1}{\mu K})$ with $\mu = 0$ is $f$ is non-convex or general-convex
    \begin{align*}
        \xi_{i,k}^r &\leq (1 - \tfrac{\mu\eta}{6})\xi_{i,k-1}^r- \tfrac{\eta}{6} \E_{r-1}\norm{\nabla f(\yy_{i,k-1}^r)}^2 + 7\beta \eta^2 \sigma^2  \text{, and }\\      
        \tilde\xi_{i,k}^r &\leq (1 - \tfrac{\mu\eta}{6})\tilde\xi_{i,k-1}^r- \tfrac{\eta}{6} \norm{\nabla f(\E_{r-1}[\yy_{i,k-1}^r])}^2\,.
    \end{align*}
    
\end{lemma}
\begin{proof}
    Recall that $\xi_{i,k}^r \geq 0 $ is defined to be
    \[
        \xi_{i,k}^r := \rbr*{\E_{r-1}[f(\yy_{i,k}^r)] - f(\xx^\star) + \delta(1 + \tfrac{1}{K})^{K-k}\E_{r-1}\norm{\yy_{i,k}^r- \xx^{r-1}}^2}\,.
     \]
    Let us start from the local update step \eqref{eqn:covfefe-simple} (dropping unnecessary subscripts and superscripts)
    \begin{align*}
        \E_{r-1,k-1}\norm{\yy_i^+ - \xx^\star}^2_A &\leq \norm{\yy_i - \xx^\star}^2_A - 2\eta\inp*{A(\yy_i - \xx^\star)}{A(\yy_i - \xx^\star)} + 2\eta \inp*{(A - A_i)(\yy_i - \xx)}{A(\yy_i - \xx^\star)} \\&\hspace{2cm}+ \eta^2\norm{A(\yy_i - \xx^\star) + (A_i - A)(\yy_i - \xx))}^2_A + \beta\eta^2\sigma^2\\
        &\leq \norm{\yy_i - \xx^\star}^2_A - \tfrac{3\eta}{2} \norm{A(\yy_i - \xx^\star)}^2_2 + 2\eta \norm{(A - A_i)(\yy_i - \xx)}_2^2\\&\hspace{2cm}+ 2\eta^2\norm{A(\yy_i - \xx^\star)}^2_A + 2\eta^2\norm{(A_i - A)(\yy_i - \xx))}^2_A + \beta\eta^2\sigma^2\\
        &\leq \norm{\yy_i - \xx^\star}^2_A - (\tfrac{3\eta}{2} - 2\eta^2\beta) \norm{A(\yy_i - \xx^\star)}^2_2 + \beta\eta^2\sigma^2 + \delta^2(2\eta^2\beta + 2\eta)\norm{\yy_i - \xx}^2_2\\
        &\leq \norm{\yy_i - \xx^\star}^2_A - (\tfrac{3\eta}{2} - 2\eta^2\beta) \norm{A(\yy_i - \xx^\star)}^2_2 + \beta\eta^2\sigma^2 + \tfrac{\delta}{10K}\norm{\yy_i - \xx}^2_2\,.
    \end{align*}
    The second to last inequality used that $\norm{\,\cdot\,}^2_A \leq \beta \norm{\,\cdot\,}^2_2$ by \eqref{asm:smoothness} and that $\norm{(A - A_i)(\,\cdot\,)}^2_2 \leq \delta^2\norm{\,\cdot\,}^2_2$ by \eqref{asm:hessian-similarity}. The final inequality used that $\eta \leq \max(\frac{1}{10\beta}, \frac{1}{22\delta K})$. Now, multiplying Lemma~\ref{lem:interp-drift} by $\delta(1 + \frac{1}{K})^{K-k}\leq \tfrac{20\delta}{7}$ we have
    \begin{align*}
        \delta(1 + \tfrac{1}{K})^{K-k}\E_{r-1,k-1}\norm{\yy_i^+ - \xx}^2 &\leq \delta(1 + \tfrac{1}{K})^{K-k}(1 + \tfrac{1}{2K})\norm{\yy_i - \xx}^2 + 20\delta K \eta^2\norm{A(\yy_i - \xx^\star)}^2 + 3\delta\eta^2\sigma^2\\
        &\leq \delta(1 + \tfrac{1}{K})^{K-k}(1 + \tfrac{1}{2K} + \tfrac{1}{10 K})\norm{\yy_i - \xx}^2 - \tfrac{\delta}{10K}\norm{\yy_i - \xx}^2 \\&\hspace{2cm}+ 20\delta K \eta^2\norm{A(\yy_i - \xx^\star)}^2 + 3\delta\eta^2\sigma^2\\
        &\leq (1 - \tfrac{1}{5K})\delta(1 + \tfrac{1}{K})^{K-k+1}(1 + \tfrac{1}{K})\norm{\yy_i - \xx}^2 - \tfrac{\delta}{10K}\norm{\yy_i - \xx}^2\\&\hspace{2cm}+ 20\delta K \eta^2\norm{A(\yy_i - \xx^\star)}^2 + 3\delta\eta^2\sigma^2\,.
    \end{align*}
    Adding this to our previous equation gives the following recursive bound:
    \begin{multline*}
        \rbr*{\E_{r-1,k-1}\norm{\yy_i^+ - \xx^\star}^2_A + \delta(1 + \tfrac{1}{K})^{K-k}\E_{r-1,k-1}\norm{\yy_i^+ - \xx}^2} \leq \\\rbr*{\norm{\yy_i - \xx^\star}^2_A + (1 - \tfrac{1}{5K})\delta(1 + \tfrac{1}{K})^{K-k+1}\norm{\yy_i - \xx}^2}
        - (\tfrac{3\eta}{2} - 2\eta^2\beta- 20\delta K \eta^2) \norm{A(\yy_i - \xx^\star)}^2_2 + (3\delta+\beta) \eta^2 \sigma^2
    \end{multline*}
    The bound on our step-size $\eta \leq \min(\frac{1}{10\beta}, \frac{1}{22\delta K})$ implies that $\tfrac{3\eta}{2} - 2\eta^2\beta- 20\delta K \eta^2 \geq \tfrac{\eta}{3}$ and recall that $\delta \leq 2\beta$. This proves first statement of the lemma for non-strongly convex functions ($\mu =0 $). If additionally $f$ is strongly-convex with $\mu > 0$, we have 
    \[
        \eta \norm{A(\yy_i - \xx^\star)}^2_2 \geq \tfrac{\mu\eta}{2} \norm{\yy_i - \xx^\star}^2_A + \tfrac{\eta}{2} \norm{A(\yy_i - \xx^\star)}^2_2\,.
    \]
    This can be used to tighten the inequality as follows\begin{multline*}
        \rbr*{\E_{r-1,k-1}\norm{\yy_i^+ - \xx^\star}^2_A + \delta(1 + \tfrac{1}{K})^{K-(k -1)}\E_{r-1,k-1}\norm{\yy_i^+ - \xx}^2} \leq \\\rbr*{(1 - \tfrac{\mu\eta}{6})\norm{\yy_i - \xx^\star}^2_A + (1 - \tfrac{1}{5K})\delta(1 + \tfrac{1}{K})^{K-k+1}\norm{\yy_i - \xx}^2}
        - \tfrac{\eta}{2} \norm{A(\yy_i - \xx^\star)}^2_2 + 7\beta \eta^2 \sigma^2
    \end{multline*}
    If $\eta \leq \frac{1}{\mu K}$, then $(1 - \tfrac{1}{5K}) \leq (1 - \tfrac{\mu \eta}{6})$ and we have the strongly-convex version of the first statement.

    Now for the second statement, recall that $\tilde\xi_{i,k}^r \geq 0 $ was defined to be
    \[
        \tilde\xi_{i,k}^r := \rbr*{[f(\E_{r-1}[\yy_{i,k}^r])] - f(\xx^\star) + \delta(1 + \tfrac{1}{K})^{K-k}\E_{r-1}\norm{\E_{r-1}[\yy_{i,k}^r] - \xx^{r-1}}^2}\,.   
    \]
    Observe that for quadratics, $\E_{r-1}[\nabla f(\xx)] = \nabla f(\E_{r-1}[\xx])$. This implies that the analysis of $\tilde\xi_{i,k}^r$ is essentially of a deterministic process with $\sigma = 0$, proving the second statement. It is also straightforward to repeat exactly the above argument to formally verify the second statement.
\end{proof}

\paragraph{Server progress in one round.} Now we combine the progress made by each client in one step to calculate the server progress.
\begin{lemma}\label{lem:interp-one-round}
    Suppose \eqref{asm:hessian-similarity}, \eqref{asm:smoothness} and \eqref{asm:variance} hold, and $\{f_i\}$ are quadratics. Then, the following holds for the update \eqref{eqn:covfefe-simple} with $\eta \leq \min(\frac{1}{10\beta}, \frac{1}{21\delta K}, \frac{1}{10\mu K})$ and weights $w_k := (1-\tfrac{\mu \eta}{6})^{1-k}$:
    \begin{align*}
        \frac{\eta}{6}\sum_{k=1}^K w_k \E_{r-1}\norm{\nabla f(\xx_{k}^r)}^2 \leq (f(\E_{r-2}[\xx^{r-1}]) - f^\star) - w_K (f(\E_{r-1}[\xx^{r}]) - f^\star) + \sum_{k=1}^K w_k 8 \eta \tfrac{\sigma^2}{N}\,.
    \end{align*}
    Set $\mu =0$ if $\{f_i\}$s are not strongly-convex (is only general-convex).
\end{lemma}
\begin{proof}
    Let us do the non-convex (and general convex) case first. By summing over Lemma \ref{lem:interp-onestep} we have
    \[
        \frac{\eta}{6}\sum_{k=1}^K \E_{r-1}\norm{\nabla f(\yy_{i,k})}^2 \leq \xi_{i,0}^r - \xi_{i,K}^r + 7K\beta \eta^2 \sigma^2\,.
    \]
    A similar result holds with $\sigma = 0$ for $\E_{r-1}[\yy_{i,k}]$. Now, using Lemma~\ref{lem:avg-variance} we have that
    \[
        \frac{\eta}{6}\sum_{k=1}^K \E_{r-1}\norm{\nabla f(\xx_{k}^r)}^2 \leq \underbrace{\frac{1}{N}\sum_{i=1}^N(\tilde\xi_{i,0}^r + \tfrac{1}{N}\xi_{i,0})}_{=: \theta^r_{+}} - \underbrace{\frac{1}{N}\sum_{i=1}^N(\tilde\xi_{i,K}^r + \tfrac{1}{N}\xi_{i,K})}_{=: \theta^r_{-}} + 7K \beta \eta^2 \tfrac{\sigma^2}{N} \,.
    \]
    Using Lemma~\ref{lem:interp-variance}, we have that 
    \[
        \theta^{r}_{+} = (1+ \tfrac{1}{N})(f(\xx^{r-1}) -f(\xx^\star)) \leq  f(\E_{r-1}[\xx^r]) + \tfrac{1}{N}\E f(\xx^r) - (1 + \tfrac{1}{N})f(\xx^\star) + 3K\beta \tfrac{\sigma^2}{N}\,.
    \] 
    Further, by Lemma~\ref{lem:averaging-value}, we have that
    \[
        \theta^r_{-} \geq f(\E_{r-1}[\xx^r]) + \tfrac{1}{N}f(\xx^r) - (1 + \tfrac{1}{N})f(\xx^\star)\,.
    \]
    Combining the above gives:
    \[
        \frac{\eta}{6}\sum_{k=1}^K \E_{r-1}\norm{\nabla f(\xx_{k}^r)}^2 \leq  f(\E_{r-2}[\xx^{r-1}]) - f(\E_{r-1}[\xx^r]) + 10\beta K \tfrac{\sigma^2}{N}\,.
    \]
    proving the second part of the Lemma for weights $w_k = 1$. The proof of strongly convex follows a very similar argument. Unrolling Lemma~\ref{lem:interp-onestep} using weights $w_k := (1-\tfrac{\mu \eta}{6})^{1-k}$ gives 
    \[
        \frac{\eta}{6}\sum_{k=1}^K w_k \E_{r-1}\norm{\nabla f(\xx_{k}^r)}^2 \leq \theta^r_{+} - w_K \theta^r_{-} + \sum_{k=1}^K w_k 7 \eta \tfrac{\sigma^2}{N}\,.
    \]
    As in the general-convex case, we can use Lemmas~\ref{lem:averaging-value}, \ref{lem:avg-variance} and \ref{lem:interp-variance} to prove that
    \[
        \frac{\eta}{6}\sum_{k=1}^K w_k \E_{r-1}\norm{\nabla f(\xx_{k}^r)}^2 \leq (f(\E_{r-2}[\xx^{r-1}]) - f^\star) - w_K (f(\E_{r-1}[\xx^{r}]) - f^\star) + \sum_{k=1}^K w_k 8 \eta \tfrac{\sigma^2}{N}\,.
    \]
\end{proof}

\paragraph{Deriving final rates.} The proof of Theorem~\ref{thm:interpolation-full} follows by appropriately unrolling Lemma~\ref{lem:interp-one-round}. For general-convex functions, we can simply use Lemma~\ref{lemma:general} with the probabilities set as $p_k^r = \frac{1}{KR}$. For strongly-convex functions, we use $p_k^r \propto (1 - \frac{\mu \eta}{6})^{1 - rk}$ and follow the computations in Lemma~\ref{lem:constant}.